\documentclass[twoside,11pt]{article}
\newif\ifarxiv
\arxivtrue
\usepackage{amsmath} %
\usepackage{amsthm} %
\let\amsproof\proof
\let\endamsproof\endproof
\ifarxiv
\usepackage[preprint]{jmlr2e} %
\else
\usepackage{jmlr2e} %
\fi
\let\proof\amsproof
\let\endproof\endamsproof
\renewcommand{\qedsymbol}{\BlackBox}
\usepackage{graphicx} %
\usepackage{tikz}
\usetikzlibrary{positioning,arrows.meta,calc,fit,decorations.pathreplacing,backgrounds}
\usepackage{standalone} %
\usepackage{amssymb} %
\usepackage{bm}
\usepackage{todonotes}
\usepackage{enumerate}
\usepackage{titletoc}
\usepackage{tabularx}
\usepackage{float}

\hypersetup{
	colorlinks   = true,
	urlcolor     = blue,
	linkcolor    = blue,
	citecolor   = blue
}
\usepackage[capitalize]{cleveref}

\usepackage{bm}
\usepackage{dsfont}

\newcommand{\ind}{\mathds{1}}
\newcommand{\R}{\mathds{R}}
\newcommand{\N}{\mathds{N}}

\newcommand{\Z}{\mathds{Z}}
\newcommand{\E}{\mathds{E}}
\newcommand{\K}{\mathds{K}}

\providecommand{\s}{}

\newcommand{\ba}{\bm{a}}
\newcommand{\bb}{\bm{b}}
\newcommand{\bc}{\bm{c}}
\newcommand{\bd}{\bm{d}}
\newcommand{\be}{\bm{e}}

\newcommand{\bu}{\bm{u}}
\newcommand{\bv}{\bm{v}}

\newcommand{\bx}{\bm{x}}
\newcommand{\by}{\bm{y}}
\newcommand{\bz}{\bm{z}}

\newcommand{\Bcal}{\mathcal{B}}
\newcommand{\Ccal}{\mathcal{C}}

\newcommand{\Fcal}{\mathcal{F}}
\newcommand{\Gcal}{\mathcal{G}}
\newcommand{\Hcal}{\mathcal{H}}

\newcommand{\Ncal}{\mathcal{N}}

\newcommand{\eps}{\varepsilon}

\newcommand{\balpha}{\bm{\alpha}}
\newcommand{\bbeta}{\bm{\beta}}

\newcommand{\col}{\operatorname{col}}

\newcommand{\ra}{\rangle}

\renewcommand{\bar}{\overline}

\newcommand{\wh}[1]{\widehat{#1}}
\newcommand{\wt}[1]{\widetilde{#1}}

\makeatletter
\def\blfootnote{\gdef\@thefnmark{}\@footnotetext}
\makeatother

\newcommand{\envbreak}{ \, \\[-12pt]}

\usepackage{setspace}

\newtheorem{definition}{Definition}
\newtheorem{lemma}{Lemma}
\newtheorem{proposition}{Proposition}
\newtheorem{remark}{Remark}
\newtheorem{corollary}[theorem]{Corollary}

\numberwithin{theorem}{section}
\numberwithin{lemma}{section}
\numberwithin{proposition}{section}
\numberwithin{remark}{section}
\numberwithin{example}{section}
\numberwithin{definition}{section}
\numberwithin{equation}{section}

\Crefname{theorem}{Theorem}{Theorems}
\Crefname{corollary}{Corollary}{Corollaries}
\Crefname{lemma}{Lemma}{Lemmas}
\Crefname{proposition}{Proposition}{Propositions}
\Crefname{assumption}{Assumption}{Assumptions}
\Crefname{definition}{Definition}{Definitions}
\Crefname{remark}{Remark}{Remarks}
\Crefname{example}{Example}{Examples}

\DeclareMathOperator{\clip}{clip}

\newcommand{\floor}[1]{\mathopen{}\mathclose{\left\lfloor #1 \right\rfloor}}
\newcommand{\ceil}[1]{\mathopen{}\mathclose{\left\lceil #1 \right\rceil}}

\ifarxiv\else
\ShortHeadings{Optimal ReLU Approximation of Smooth Functions}{Nagler}
\fi
\firstpageno{1}

\title{Optimal Approximation of Smooth Functions by ReLU Neural Networks}
\author{\name Thomas Nagler \email t.nagler@lmu.de \\
       \addr Department of Statistics, LMU Munich \\
       Munich Center for Machine Learning (MCML)
       }
\editor{}

\begin{document}
\maketitle
\begin{abstract}%
We study uniform approximation of the H\"older ball $\mathcal H_1^s([0,1]^D)$ by fully connected ReLU networks. Except in the one-dimensional Lipschitz case, existing constructions remain separated from the information-theoretic efficiency limit. This article closes the gap for every $s>0$ and $D\ge1$: networks with width $O(N)$, depth $O(L)$, and parameter magnitude $O(1)$ attain uniform approximation error of order $(N^2L^2\log N)^{-s/D}$. No relation between width and depth is required. The construction uses new encoding and decoding mechanisms that exploit higher-order finite-difference bounds for local Taylor coefficients of the target function. As a statistical consequence, empirical risk minimizers over the corresponding clipped network classes attain the minimax regression rate $n^{-2s/(2s+D)}$ without additional logarithmic factors. The paper also gives a detailed exposition of the key approximation techniques, covering both established and new mechanisms.
\end{abstract}


\section{Introduction}
\label{sec:introduction}

Approximation theory for neural networks has a long history. The universal approximation theorems of
\citet{cybenko1989approximation} and \citet{hornik1989multilayer} show that shallow neural
networks can approximate every continuous function on a compact set, but they are qualitative.
They do not determine how the network size must grow to achieve a target
accuracy or, conversely, what accuracy can be achieved with a given architecture.
Such quantitative bounds have become one of the central problems in neural-network
approximation \citep{devore2021neural,dahmen2022special}.

This article concerns the most common instance of the problem.
The goal is to construct approximations of smooth functions $f$ belonging to the H\"older ball
$\Hcal_1^s([0,1]^D)$ by ReLU networks with width $N$, depth $L$, and bounded parameters. The worst-case error $\epsilon_{N, L}$ attained by a given architecture quantifies how efficiently the network class uses its capacity.
There is an information-theoretic limit on this efficiency.
Let $\Ncal(\epsilon, \Fcal, \|\cdot\|_\infty)$ be the covering number of a function class $\Fcal$ at accuracy $\epsilon$ in the sup-norm. The classical entropy estimate \citep[e.g.,][Theorem~27.14]{polyanskiyWu2025}
\begin{align*}
    \log \mathcal N\bigl(\epsilon,
        \Hcal_1^s([0,1]^D),\|\cdot\|_\infty\bigr)
    &\asymp \epsilon^{-D/s}
\end{align*}
shows how much information is needed to represent the H\"older ball at
accuracy $\epsilon$. Any function class that uniformly approximates $\Hcal_1^s([0,1]^D)$ with accuracy $\epsilon$ must have entropy of the same order or higher.
The entropy of ReLU
networks with parameters bounded by $B$ is of order $N^2L^2\log (NB)$ at the relevant
scale \citep{ou2024covering}, and hence of order $N^2L^2\log N$ whenever $B$ grows at most polynomially in $N$. Equating the entropy estimates gives the optimal error
\begin{align*}
    \epsilon_{N, L} = (N^2L^2\log N)^{-s/D}.
\end{align*}

Over the last decade, several works have moved increasingly close to this benchmark.
\citet{yarotsky2017error,yarotsky2018optimal} established approximation rates for smooth
functions and the rate $L^{-2s/D}$ for fixed-width, very deep networks when
$0<s\le1$. \citet{yarotsky2020phase} extended this rate for very deep networks to arbitrary
smoothness, and \citet{siegel2023optimal} to Sobolev and Besov
spaces. \citet{schmidt2020nonparametric} constructed sparse
networks, while \citet{kohler2021rate} obtained error bounds of $N^{-2s/D}$ for wide fully connected networks and $L^{-2s/D}$ for fixed-width deep networks. For arbitrary integer smoothness,
\citet{Lu2021} derived simultaneous width--depth bounds with logarithmic overhead in
both width and depth. For $0<s\le1$, \citet{shen2022optimal} removed these architectural
overheads, but did not track parameter magnitudes. \citet{ou2024covering} attained the
entropy-optimal bounded-parameter rate in the special case $s=D=1$. For arbitrary smoothness, \citet{yang2025optimal} obtained error $(N^2L^2)^{-s/D}$ with width $O(N)$ and depth $O(L)$, and untracked parameter magnitudes (but at least exponential in $L$). Measured against the target error $\epsilon_{N, L}$, this leaves root-logarithmic overhead in width. \citet{nagler2026optimal} obtained the same error with parameter magnitudes bounded by $O(N)$, but with the same width overhead and additional depth overhead when $\log L \gg \log N$. The bounds
are summarized in \Cref{tab:approximation-literature}.

It therefore remains unclear whether ReLU networks are an optimal class of approximators
of general smooth functions. We close this gap for every $s>0$ and $D\in\N_+$.

\begin{table}[tp]
    \centering
    \footnotesize
    \setlength{\tabcolsep}{3pt}
    \renewcommand{\arraystretch}{1.15}
    \begin{tabularx}{\textwidth}{@{}>{\raggedright\arraybackslash}p{0.25\textwidth}
        >{\raggedright\arraybackslash}p{0.115\textwidth}
        >{\raggedright\arraybackslash}p{0.075\textwidth}
        >{\raggedright\arraybackslash}p{0.195\textwidth}
        >{\raggedright\arraybackslash}p{0.115\textwidth}
        >{\raggedright\arraybackslash}X@{}}
        \hline
        Reference
        & Smoothness
        & Width
        & Depth
        & Parameter magnitude
        & Uniform error \\
        \hline
        \citet{yarotsky2018optimal}
        & $0<s\le1$
        & $1$
        & $L$
        & Not tracked
        & $L^{-2s/D}$ \\[2pt]
        \citet{yarotsky2020phase}
        & $s>0$
        & $1$
        & $L$
        & Not tracked
        & $L^{-2s/D}$ \\[2pt]
        \citet{schmidt2020nonparametric}
        & $s>0$
        & $N^2$
        & $\log N$
        & $1$ (sparse)
        & $N^{-2s/D}$ \\[2pt]
        \citet{kohler2021rate}
        & $s>0$
        & $N$
        & $\log N$
        & Not tracked
        & $N^{-2s/D}$ \\[2pt]
        \citet{kohler2021rate}
        & $s>0$
        & $1$
        & $L$
        & Not tracked
        & $L^{-2s/D}$ \\[2pt]
        \citet{Lu2021}
        & $s\in\N_+$
        & $N\log N$
        & $L\log L$
        & Not tracked
        & $(N^2L^2)^{-s/D}$ \\[2pt]
        \citet{shen2022optimal}
        & $0<s\le1$
        & $N$
        & $L$
        & Not tracked
        & $(N^2L^2\log N)^{-s/D}$ \\[2pt]
        \citet{yang2025optimal}
        & $s>0$
        & $N$
        & $L$
        & Not tracked
        & $(N^2L^2)^{-s/D}$ \\[2pt]
        \citet{nagler2026optimal}
        & $s>0$
        & $N$
        & $L+L\sqrt{\log L/\log N}$
        & $N$
        & $(N^2L^2)^{-s/D}$ \\[2pt]
        \citet{ou2024covering}
        & $s=D=1$
        & $N$
        & $L$
        & $1$
        & $(N^2L^2\log N)^{-1}$ \\[2pt]
        \hline
        Present work
        & $s>0$
        & $N$
        & $L$
        & $1$
        & $(N^2L^2\log N)^{-s/D}$ \\
        \hline
    \end{tabularx}
    \caption{Overview of quantitative approximation bounds for functions on $[0,1]^D$ with H\"older-smoothness $s > 0$. Width, depth, and parameter magnitude are stated up to constant factors.}
    \label{tab:approximation-literature}
\end{table}

\medskip
\noindent\textit{Main results.}  For any $f \in \Hcal_1^s([0,1]^D)$, we construct an explicit ReLU network $\phi$
with width $O(N)$, depth $O(L)$, and parameter magnitude $O(1)$ such that
\begin{align*}
    \|f-\phi\|_\infty
    &\lesssim \epsilon_{N,L}.
\end{align*}
The result imposes no relation between width and depth. For $s>1$, it is the first bound that attains $\epsilon_{N,L}$ without overhead in width or depth. In particular, it is the first to achieve the rate $(N^2\log N)^{-s/D}$ for fixed depth. The result is optimal in the sense that no ReLU construction of comparable size can achieve a smaller error, and neither can any other function class with similar metric capacity.

Like the current state-of-the-art constructions, our approach
stores and retrieves local Taylor coefficients by bit extraction. At the optimal scale, the network has to store $K^D\asymp N^2L^2\log N$ Taylor coefficients, while its capacity is of order $N^2L^2\log N$ digits. Storing each coefficient on its own with sufficient accuracy costs order $\log K$ digits, which is the source of the logarithmic gaps in previous constructions. The key insight is that Taylor coefficients of smooth functions have additional regularity that makes them more compressible. Our main new ingredient is a \emph{higher-order point-fitting} network: it stores only the quantized higher-order differences of a coefficient sequence, each in a single digit, and reconstructs the values by a segmented discrete Taylor--Newton recursion with bounded weights. To apply it in dimension $D\ge2$, we flatten the coefficient grid into a single sequence that keeps these difference bounds by inserting Hermite interpolation bridges between its rows.
Along the way, we give simplified, self-contained proofs of all the standard
components. In particular, median smoothing needs only $O(D)$ copies of the
network instead of $3^D$, and a reduced-alphabet base-$\zeta$ digit extraction allows storing bounded integers directly in a network with controlled parameter magnitudes.

The approximation-theoretic result also has implications for machine learning and statistics.
We show that any approximate empirical risk minimizer $\wh f_n$ over suitably sized networks
satisfies
\begin{align*}
    \sup_{f_0\in\Hcal_1^s([0,1]^D)}\E_{f_0}\bigl[\|\widehat f_n-f_0\|_{L_2(P)}^2\bigr]
    &\lesssim n^{-2s/(2s+D)}.
\end{align*}
This rate matches the minimax lower bound of
\citet{stone1982optimal} over the H\"older ball and removes the polylogarithmic factors from previous works.

\medskip
\noindent\textit{Outline.}
\Cref{sec:prelim} introduces the notation.
\Cref{sec:approx} states the main approximation and optimality results, and
\cref{sec:regression} gives the regression result. \Cref{sec:proof-sketch-compact} gives a comprehensive overview of the construction and proof ideas. It is deliberately expository and also reviews and simplifies already established techniques.
Detailed proofs are given in the appendices.

\section{Setup and Notation}
\label{sec:prelim}
We use the following notation throughout the paper. For $f \colon S \to \R$, we write
$\|f\|_{\infty,S}=\sup_{\bx\in S}|f(\bx)|$, omitting $S$ when it is clear from
context, and $\| f\|_{L_2(P)}^2 = \int f(\bx)^2 dP(\bx)$ for a measure $P$. The modulus of continuity of $f \in C([0,1]^D)$ is
\begin{align*}
    \omega_f(\delta) = \sup\{|f(\bx)-f(\by)|\colon \bx,\by\in[0,1]^D,
    \|\bx-\by\|_\infty\le\delta\}, \qquad \delta > 0.
\end{align*}
Our approximation results are stated for the following Hölder
classes.

\begin{definition}[Hölder class]
    \label{def:holder}
    Let $S$ be a closed box in $\R^D$, let $s,C>0$, and set $m = \lceil s\rceil-1$.
    For a multi-index $\balpha=(\alpha_1,\ldots,\alpha_D)\in\N_0^D$, write
    $\|\balpha\|_1=\sum_{j=1}^D\alpha_j$ and
    $\partial^{\balpha}=\partial_1^{\alpha_1}\cdots\partial_D^{\alpha_D}$.
    The Hölder class $\Hcal_C^s(S)$ consists of all $m$ times continuously
    differentiable functions $f\colon S\to\R$ that satisfy the following
    bounds:
    \begin{align*}
        \max_{\balpha:\,\|\balpha\|_1\le m}
        \|\partial^{\balpha}f\|_{\infty,S}
        \le C,\qquad
        \max_{\balpha:\,\|\balpha\|_1=m}
        \sup_{\substack{\bx,\by\in S \\ \bx\ne\by}}
        \frac{\big|\partial^{\balpha}f(\bx)
        -\partial^{\balpha}f(\by)\big|}
        {\|\bx-\by\|_\infty^{s - m}}
        \le C.
    \end{align*}
\end{definition}
In particular, every $f\in\Hcal_C^s(S)$ satisfies
$\|f\|_{\infty,S}\le C$ and, when $s$ is an integer, derivatives of order $s-1$ are
required to be Lipschitz continuous.

We consider fully connected feedforward networks with the ReLU activation
function.
\begin{definition}[Class of fully connected neural networks]
    Let $\sigma : \R \to \R$, $\sigma(x)=\max\{x,0\}=(x)_+$ be the ReLU activation function, and for $v \in \R$ denote by
    $\sigma_v(t) = \sigma(t + v)$ the $v$-biased ReLU function.
    For vectors $\bx = (x_1, \dots, x_D)^\top \in \R^D$ and $\bv = (v_1, \dots, v_D)^\top \in \R^D$, we apply $\sigma_{\bv}$ componentwise, i.e.,
    \begin{align*}
        \sigma_{\bv}(\bx) = (\sigma(x_1+v_1), \dots, \sigma(x_D+v_D))^\top = ((x_1+v_1)_+, \dots, (x_D+v_D)_+)^\top.
    \end{align*}
    A function $f : \R^{N_0} \to \R^{N_{L + 1}}$ is called a
    \emph{fully connected neural network} with depth $L \in \N$, and width vector
    $(N_0, N_1, \dots, N_{L + 1}) \in \N^{L+2}$ if it can be written as
    \begin{align}
        \label{eq:FNN}
        f(\bx) = \bv_L +  W_L  \sigma_{\bv_{L-1}} \cdots W_1 \sigma_{\bv_0} W_0 \bx ,
    \end{align}
    where for each $i = 0, \dots, L$,
    $W_i \in \R^{N_{i + 1} \times N_{i}}$ is the weight matrix and
    $\bv_i \in \R^{N_{i + 1}}$ is the bias vector. Let $B_i = \max\{\max_{j,k} |W_{i, j,k}|, \|\bv_i\|_{\infty}\}$ be the maximum parameter magnitude in layer $i$.
    The class of such networks with depth $L$, maximum hidden-layer width
    $N \in \N$, and maximum parameter magnitude $B \in (0,\infty)$ is defined
    as
    \begin{align*}
        \mathcal{F}_{d_0, d_{L + 1}}(N, L, B)
        &= \left\{ f\text{ of the form \eqref{eq:FNN}}:
        \begin{gathered}
            N_0 = d_0,\quad N_{L + 1} = d_{L + 1}, \\
            \max_{1\le i\le L} N_i \le N,\quad
            \max_{0\le i\le L} B_i \le B
        \end{gathered}\right\}.
    \end{align*}
\end{definition}
We write
    $\phi\in\wt\Fcal_{d,d'}(N,L,B)$
if there is a constant $C \in \N_+$ not depending on
$N,L,B$ such that
    $\phi\in\Fcal_{d,d'}(CN,CL,CB).$
The constant may depend on the smoothness and dimension of the target function.

To measure the complexity of function classes, we use the following
standard notation.
\begin{definition}[Covering number and entropy]
    Let $(V,\|\cdot\|)$ be a normed space, $\Fcal\subset V$, and
    $\epsilon>0$. A finite collection
    $\{g_1,\dots,g_M\}\subset V$ is an $\epsilon$-cover of $\Fcal$ if, for
    every $f\in\Fcal$, there is an $i\in\{1,\dots,M\}$ such that
    $\|f-g_i\|\leq\epsilon$. The covering number
    $\mathcal N(\epsilon,\Fcal,\|\cdot\|)$ is the smallest cardinality of such
    a cover. We call $\log\mathcal N(\epsilon,\Fcal,\|\cdot\|)$ the metric entropy.
\end{definition}

We write $f \lesssim g$ if there is $C\in(0,\infty)$ such that
$f\le Cg$ and, similarly, $f\gtrsim g$ if $g\lesssim f$. If $f\lesssim g$ and
$f\gtrsim g$, we write $f\asymp g$. All unqualified logarithms are natural
logarithms. For real numbers $c\le d$, define
$\clip_{[c,d]}(t)=\min\{\max\{t,c\},d\}$. We apply this map coordinate-wise
to vectors and write $\clip_{[c,d]^m}$ in that case.

\section{Main Results}
\label{sec:main-results}

We now formally state our main results: an approximation error bound, its optimality, and the resulting statistical error rates for nonparametric regression.

\subsection{Optimal Approximation}
\label{sec:approx}

The main contribution of this article is a new bound on the approximation error of ReLU neural networks.
\begin{theorem}[Approximation]
    \label{thm:main-approx}
    Let $s>0$ and $D\in\N_+$. Then there are constants $A\in\N_+$ and $C\ge1$, depending only on $s$ and $D$, such
    that, for every $N\ge2$ and
    $L\in\N_+$, the class $\Fcal^{N,L}:=\Fcal_{D, 1}(AN,AL,1)$ satisfies
    \begin{align}
        \label{eq:compact-optimal-approx}
        \sup_{f \in \Hcal_1^s([0, 1]^D)}
        \inf_{\phi \in \Fcal^{N,L}}\|f - \phi\|_\infty
         & \le C\epsilon_{N,L},
         \qquad \epsilon_{N,L}:=(N^2L^2\log N)^{-s/D}.
    \end{align}
    Both $A$ and $C$ are bounded by a polynomial in $D$.
\end{theorem}
The theorem imposes no constraints on the relation between width and depth.
It can also be inverted: for every $\epsilon\in(0,1)$, choosing $N, L$ with $N^2L^2\log N\ge\epsilon^{-D/s}$ gives uniform error at most $C\epsilon$.
The construction and its relation to earlier approximation strategies are
developed in \cref{sec:proof-sketch-compact}.

\Cref{thm:main-approx} matches the bound of \citet{ou2024covering} for $s=D=1$.
For $0<s\le1$, it attains the same rate as \citet{shen2022optimal}, but with bounded parameters.
For $s>1$, it improves on all previous bounds and removes all architectural overheads.
In fact, the rate in \cref{thm:main-approx} cannot be improved, as made precise by the following lower bound.

\begin{theorem}[Lower bound]
    \label{thm:lower-bound}
    Let $s>0$, $D\in\N_+$, and $b\ge0$. There is $c=c(s,D,b)>0$ such that, for all
    $N\ge2$, $L\in\N_+$,
    \begin{align*}
        \sup_{f \in \Hcal_1^s([0, 1]^D)}
        \inf_{\phi \in \Fcal_{D,1}(N,L,N^b)}\|f - \phi\|_\infty
        &\ge c\,\epsilon_{N,L}.
    \end{align*}
\end{theorem}
The proof is given in \cref{sec:entropy-proof}.
Together with \cref{thm:main-approx}, this shows that $\epsilon_{N,L}$ is the exact rate of uniform approximation of the H\"older ball by ReLU networks with bounded parameters.
In fact, even allowing parameters that grow polynomially in the width does not improve the rate.
The lower bound in \cref{thm:lower-bound} relies on a metric entropy argument that also establishes optimality in a broader sense: no other function class with similar metric capacity can achieve a smaller uniform approximation error.

\begin{corollary}[Entropy optimality]
    \label{cor:network-entropy}
    If $C\epsilon_{N,L}<1/4$, the class $\Fcal^{N,L}$ of
    \cref{thm:main-approx} satisfies
    \begin{align}
        \label{eq:compact-optimal-entropy}
        \log \mathcal N\bigl(\epsilon_{N,L}, \Fcal^{N,L}, \|\cdot\|_\infty\bigr)  \asymp \inf_{\Gcal}
        \log \mathcal N\bigl(\epsilon_{N,L}, \Gcal, \|\cdot\|_\infty\bigr),
    \end{align}
    where the infimum is over all function classes attaining \eqref{eq:compact-optimal-approx} with $\Fcal^{N,L}$ replaced by $\Gcal$.
\end{corollary}
Both sides of \eqref{eq:compact-optimal-entropy} are of order $N^2L^2\log N\asymp\epsilon_{N,L}^{-D/s}$. Hence, up to constant factors, no class that attains accuracy $C\epsilon_{N,L}$ has smaller metric entropy than $\Fcal^{N,L}$.
This sharpens the Kolmogorov--Donoho optimality of \citet{bolcskei2019optimal} and \citet{elbrachter2021deep}, where optimal description length was matched only up to logarithmic factors.

Previous constructions do not attain this optimality, except for
\citet{ou2024covering}, who cover only the case $s=D=1$. \citet{shen2022optimal}
cover $0<s\le1$ and $D\ge1$ with the right width, depth, and error
$\epsilon_{N,L}$, but their parameter magnitudes are exponential in the
depth, which raises the entropy (see below).
For general $s>0$ and $D\ge1$, the best previous results are those of \citet{yang2025optimal} and \citet{nagler2026optimal}. Both attain error $(N^2L^2)^{-s/D}$, which exceeds $\epsilon_{N,L}$ by a factor $(\log N)^{s/D}$. In addition, the exponential parameter magnitudes of \citet{yang2025optimal} raise the entropy to order $N^2L^3\log N$, and the depth overhead of \citet{nagler2026optimal} raises it to order $N^2L^2\log(NL)$.

The lower bound and optimality corollary rely on sup-norm entropy, which is the
relevant notion of complexity for uniform approximation. For ReLU networks,
this makes control of the parameter magnitude essential: with unbounded
parameters, the sup-norm entropy is infinite, even after clipping the output.
\citet[Theorem~2.1]{ou2024covering} show that, at the relevant scales
$\epsilon^{-1}\le\mathrm{poly}(N, L)$, it holds that
\begin{align*}
    \log \mathcal N\bigl(\epsilon, \Fcal_{D, 1}(N, L, B), \|\cdot\|_\infty\bigr)
    &\asymp N^2L^2\log (NB),
\end{align*}
which grows logarithmically in $B$.
In particular, bounded parameters and parameters
growing polynomially in $N$ give the same capacity up to constants. Our
construction exploits this: we first build networks with $B = O(N)$ and then
convert them into equivalent networks with $B = O(1)$ and the same orders of
width and depth (\cref{lem:magnitude-width}).

If $\log B\gg\log N$, the network class has larger entropy than needed to attain accuracy $\epsilon_{N,L}$ (\cref{cor:network-entropy}), so the added capacity is wasted. This matters for estimation: the statistical error of empirical risk minimizers grows with the entropy of the class (\cref{sec:regression}). Parameters growing exponentially in the depth, as in the construction of \citet{shen2022optimal}, inflate the entropy and worsen the regression rate.
Without parameter bounds, the sup-norm entropy is infinite, but VC-dimension arguments \citep{bartlett2019nearly,shen2022optimal,siegel2023optimal} show that $\epsilon_{N,L}$ remains a lower bound whenever $L \le \mathrm{poly}(N)$. In this regime, \cref{thm:main-approx} is optimal even among networks with arbitrary parameters. Whether large parameters can improve the rate when $\log L\gg\log N \gg 1$ remains open; see \cref{sec:discussion}.

\subsection{Statistical Consequence}
\label{sec:regression}

The sharp approximation and entropy bounds translate into optimal statistical behavior of empirical risk minimizers over suitable ReLU network classes.
Consider the nonparametric regression model
\begin{align}
    \label{eq:regression-model}
    Y_i=f_0(X_i)+\xi_i,
    \qquad i=1,\ldots,n,
\end{align}
where $X_i\stackrel{iid}\sim P$ for a probability measure $P$ supported on $[0,1]^D$,
$\xi_i\stackrel{iid}\sim\mathcal N(0,\gamma^2)$, and the two sequences are independent.
Set $r_n=n^{-s/(2s+D)}$ and choose
$N_n, L_n$ such that
\begin{align}
    \label{eq:compact-network-balance}
    r_n^{-D/s}
    \le N_n^2L_n^2\log N_n
    \lesssim r_n^{-D/s}.
\end{align}
Define the clipped network class
\begin{align*}
    \Gcal_n
    =\bigl\{\clip_{[-1,1]}\circ f:f\in\Fcal^{N_n,L_n}\bigr\},
\end{align*}
and note that clipping amounts to appending a fixed ReLU module (\cref{lem:clipping}).

\begin{theorem}[Optimal regression]
    \label{thm:regression}
    Let $s>0$ and $D\in\N_+$. Under
    \eqref{eq:regression-model}, suppose that $f_0\in
    \Hcal_1^s([0,1]^D)$ and consider an approximate empirical risk minimizer $\wh f_n \in \Gcal_n$ in the sense that
    \begin{align}
        \label{eq:compact-approx-erm}
        \frac1n\sum_{i=1}^n
        \bigl(Y_i-\widehat f_n(X_i)\bigr)^2
        \le
        \inf_{f\in\Gcal_n}\frac1n\sum_{i=1}^n
        \bigl(Y_i-f(X_i)\bigr)^2+r_n^2,
    \end{align}
    almost surely. Then, uniformly over all probability measures $P$ on
    $[0,1]^D$,
    \begin{align*}
        \sup_{f_0\in\Hcal_1^s([0,1]^D)}
        \E_{f_0}\bigl[\|\widehat f_n-f_0\|_{L_2(P)}^2\bigr]
        \lesssim r_n^2
        =n^{-2s/(2s+D)}.
    \end{align*}
    The implied constant depends only on $s$, $D$, and $\gamma$.
\end{theorem}
The rate in \cref{thm:regression} is minimax
optimal over the Hölder ball for, e.g., uniform designs \citep{stone1982optimal}, and removes the polylogarithmic factors in previous bounds \citep[e.g.,][]{schmidt2020nonparametric,kohler2021rate,Lu2021,shen2022optimal,yang2025optimal, nagler2026optimal}. Rates without logarithmic factors were previously known only in special cases. \citet{ou2024covering} treat Lipschitz functions in dimension $D=1$, and \citet{li2026optimal} treat H\"older smoothness $0<s\le1$, but with a purpose-built ``dyadic--triangular'' activation function in place of the ReLU.
\Cref{thm:regression} follows from direct substitution of the approximation and entropy bounds from the previous section into the oracle inequality of
\citet[Theorem~4.1]{ou2024covering}. The
details are given in \cref{sec:regression-proof}.

\section{Construction and Proof Ideas}
\label{sec:proof-sketch-compact}

This section explains the key ideas of the construction behind \cref{thm:main-approx} and is intentionally expository.
Our proof of \cref{thm:main-approx} is self-contained and combines the canonical Taylor-network architecture with a new coefficient-retrieval argument.
We give a detailed account of the general construction as well as several known techniques (monomial networks, affine interpolation, bit-extraction, median smoothing), often simplifying previous proofs of the key results.
The main innovations responsible for the tighter error bound are in \cref{sec:higher-order-point-fitting,sec:multivariate-coefficients}. \Cref{sec:piecewise-affine-interpolation,sec:storing-retrieving-coefficients} review the tools needed to store the Taylor coefficients, and \cref{sec:standard-components} explains the remaining standard modules.

\subsection{The Canonical Taylor-Network Construction}
\label{sec:taylor-network-review}

Approximation of H\"older classes by ReLU networks relies on efficient
realizations of local polynomials. Let $m=\lceil s\rceil-1$ and partition
$[0,1]^D$ into cubes of side length $K^{-1}$, with lower corners
$\bbeta/K \in \{0, 1/K, \dots, (K-1)/K\}^D$. On the cube $Q_{\bbeta}$, Taylor's theorem gives
\begin{align*}
    \|f-T_{\bbeta}f\|_{L_\infty(Q_{\bbeta})}\lesssim K^{-s}, \quad \text{where} \quad
    T_{\bbeta}f(\bx)
    =\sum_{\|\balpha\|_1\le m}
      \frac{\partial^{\balpha}f(\bbeta/K)}{\balpha!}
      (\bx-\bbeta/K)^{\balpha}.
\end{align*}
The form of $T_{\bbeta}$ explains the canonical architecture used in many ReLU constructions
\citep{schmidt2020nonparametric,kohler2021rate,Lu2021, nagler2026optimal}: localize
the input around a cube cell, retrieve the appropriate Taylor coefficients, evaluate the
monomials, sum the resulting terms, and smooth the construction near cell
boundaries. The complete architecture is shown in \cref{fig:main-network}.

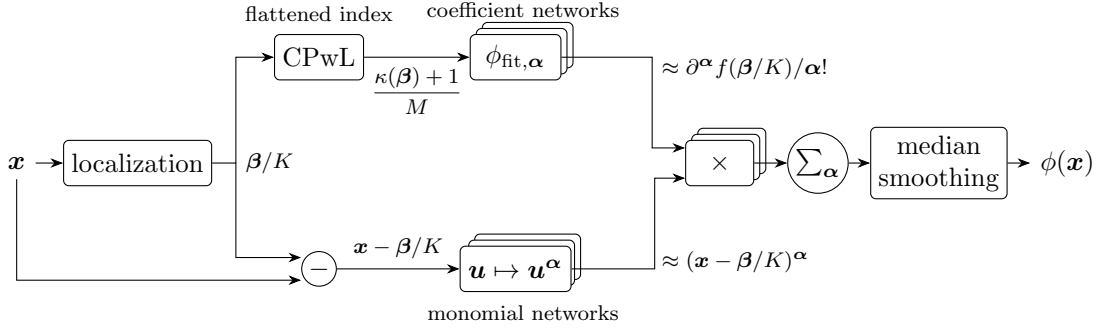
\begin{figure}[tp]
    \centering
    \begin{tikzpicture}[>=Stealth, font=\small,
    box/.style={draw, rounded corners=2pt, minimum height=0.6cm, minimum width=1.0cm,
                inner sep=3pt, align=center, fill=white},
    lab/.style={font=\scriptsize, align=center},
    wire/.style={->}]
    \node (x) at (0, -2.0) {$\bx$};
    \node[box] (loc) at (1.6, -2.0) {localization};
    \node[box] (idx) at (4.0, -0.6) {CPwL};
    \node[circle, draw, inner sep=1pt, fill=white] (minus) at (4.0, -3.4) {$-$};
    \foreach \dx in {0.16, 0.08} {
        \node[box, minimum width=1.2cm] at (6.6+\dx, -0.6+\dx) {};
        \node[box, minimum width=1.4cm] at (6.6+\dx, -3.4+\dx) {};
        \node[box, minimum width=0.9cm] at (9.3+\dx, -2.0+\dx) {};
    }
    \node[box, minimum width=1.2cm] (coef) at (6.6, -0.6) {$\phi_{{\rm fit},\balpha}$};
    \node[box, minimum width=1.4cm] (mono) at (6.6, -3.4) {$\bu \mapsto \bu^{\balpha}$};
    \node[box, minimum width=0.9cm] (prod) at (9.3, -2.0) {$\times$};
    \node[circle, draw, inner sep=1pt, fill=white] (sum) at (10.6, -2.0) {$\sum_{\balpha}$};
    \node[box] (med) at (12.2, -2.0) {median\\ smoothing};
    \node (out) at (13.9, -2.0) {$\phi(\bx)$};
    \draw[wire] (x) -- (loc);
    \draw[wire] (loc.east) -- (2.9, -2.0) |- (idx.west) node[pos=0.0, right, lab] {$\bbeta/K$};
    \draw[wire] (2.9, -2.0) |- (minus.west |- 0, -3.25);
    \draw[wire] (x.south) |- (minus.west |- 0, -3.55);
    \draw[wire] (idx) -- (coef) node[midway, below, lab] {$\dfrac{\kappa(\bbeta) + 1}{M}$};
    \draw[wire] (minus) -- (mono) node[midway, above, lab] {$\bx - \bbeta/K$};
    \draw[wire] (coef.east) -- ++(1.2, 0) |- (prod.west |- 0, -1.8);
    \node[lab, anchor=west] at (8.4, -0.75) {$\approx \partial^{\balpha}f(\bbeta/K)/\balpha!$};
    \draw[wire] (mono.east) -- ++(1.1, 0) |- (prod.west |- 0, -2.2);
    \node[lab, anchor=west] at (8.4, -3.3) {$\approx (\bx - \bbeta/K)^{\balpha}$};
    \draw[wire] (prod) -- (sum);
    \draw[wire] (sum) -- (med);
    \draw[wire] (med) -- (out);
    \node[lab, anchor=south, text width=2.4cm] at (4.0, -0.25) {flattened index};
    \node[lab, anchor=south, text width=2.6cm] at (6.7, -0.2) {coefficient networks};
    \node[lab, anchor=north, text width=2.6cm] at (6.7, -3.75) {monomial networks};
    \node[lab, anchor=north, text width=2.0cm] at (9.4, -2.45) {};
    \node[lab, anchor=north, text width=3.0cm] at (12.2, -2.5) {};
\end{tikzpicture}
    \caption{Structure of the approximating network. The localization
        module (\cref{sec:localization}) maps $\bx$ to the lower corner $\bbeta/K$ of its grid cell and
        then to the corresponding flattened index (\cref{sec:piecewise-affine-interpolation,sec:multivariate-coefficients}). For each multi-index
        $\balpha$, the network retrieves an approximation of the local Taylor
        coefficient $\partial^{\balpha}f(\bbeta/K)/\balpha!$ (\cref{sec:storing-retrieving-coefficients,sec:higher-order-point-fitting}) and multiplies it
        by an approximation of the monomial
        $(\bx-\bbeta/K)^{\balpha}$ (\cref{sec:monomial-networks}). Summing these terms yields the local Taylor
        approximation, and median smoothing (\cref{sec:median-smoothing}) controls the error near cell
        boundaries.}
    \label{fig:main-network}
\end{figure}

For later use, observe that $|(\bx-\bbeta/K)^{\balpha}| \le K^{-\|\balpha\|_1}$ for $\bx \in Q_{\bbeta}$. Hence an error of size $\eps_{\balpha}$ when retrieving the coefficient $\partial^{\balpha} f(\bbeta / K) / \balpha!$
contributes at most
$\eps_{\balpha}K^{-\|\balpha\|_1}$ to the Taylor sum. We therefore set the
target accuracy to
\begin{align*}
    \eps_{\balpha}:=K^{-(s-\|\balpha\|_1)},
\end{align*}
which keeps every summand's error at the scale $K^{-s}$ of the Taylor remainder.

Approximate localization, polynomial evaluation, and smoothing across cell
boundaries can be implemented using standard ReLU constructions
\citep{yarotsky2017error,Lu2021,nagler2026optimal}; see \cref{sec:standard-components}. These modules fit within a smaller network budget than the final approximator.
The decisive task, and the information-theoretic bottleneck of the construction, is to fit the coefficient arrays
\begin{align*}
    \left(\frac{\partial^{\balpha} f(\bbeta/K)}{\balpha !} \right)_{\bbeta \in \{0, \dots, K - 1\}^D},
    \qquad \|\balpha\|_1 \le m,
\end{align*}
using a subnetwork with minimal width and depth. At the optimal scale, each
array contains $K^D \asymp N^2L^2 \log N$ entries. Treating the coefficients
as unrelated numbers leaves a logarithmic gap in the approximation error.
Our construction closes it
with a point-fitting network that exploits higher-order
finite-difference bounds, combined with a Hermite-bridge flattening that preserves those bounds when converting multivariate coefficient arrays into one-dimensional sequences; see \cref{sec:higher-order-point-fitting,sec:multivariate-coefficients}.

\subsection{Piecewise-Linear Interpolation}
\label{sec:piecewise-affine-interpolation}

A key tool for ReLU approximation is the ability to efficiently represent continuous piecewise-linear (CPwL) functions defined on finitely many knots. It is used both to store prescribed values in the network and to localize the input to a grid cell.

\begin{proposition}[Piecewise-linear interpolation]
    \label{prop:pw-linear}
    Let $p > 0$, $N, L \in \N_+$ with $N \ge 2$, and
    $1 \le M \le N^2L$. Suppose
    $0 \le x_1 < \cdots < x_M \le (NL)^p$ and $y_1,\ldots,y_M \in \R$ satisfy
    \begin{align*}
        \min_{1 \le i < M}(x_{i+1}-x_i)
        &\ge (NL)^{-p},
        & \max_{1 \le i \le M}|y_i| &\le (NL)^p.
    \end{align*}
    Then the CPwL interpolant $f$ through
    $(x_i, y_i)$, with constant extension outside $[x_1, x_M]$, satisfies
        $f \in \Fcal_{1,1}(4N, (\ceil{5p}+18)L, N).$
\end{proposition}

Up to constant factors, the result is a special case of
\citet[Proposition~B.3]{ou2024three}, stated for the polynomial scales
used later.
Their construction refines the two-layer mechanism of
\citet[Lemmas~3.3 and~3.4]{daubechies2022nonlinear}.
\citet[Lemma~5.4]{Lu2021} and \citet[Lemma~4.1]{shen2022optimal} provide similar constructions using a technique going back to \citet[Lemma~2.2]{shen2019nonlinear}.
However, they leave parameter magnitudes untracked and the realized functions uncontrolled in small intermediate intervals.

We give a simplified proof based on a slope jump decomposition in \cref{sec:piecewise-affine-proof} and outline it here. The interpolant has
the hinge representation
\begin{align*}
    f(x) &= f(x_1)+\sum_{i=1}^{M}[f']_{x_i}\,\sigma(x-x_i),
\end{align*}
where $[f']_t = \lim_{\varepsilon \searrow 0}\{f'(t+\varepsilon)-f'(t-\varepsilon)\}$
denotes the slope jump of $f$ at $t$. This representation is a single hidden layer with one unit per knot $x_i$.
A more efficient construction is possible by adding a second layer.
Passing a CPwL function $g$ through a ReLU unit creates new knots wherever $g$ changes sign, with slope jump equal to the slope of $g$ there; see the white circles in \cref{fig:piecewise-affine}(a).
\begin{figure}[tp]
    \centering
    \begin{tikzpicture}[font=\small, x=.60cm, y=.55cm, baseline=(current bounding box.north)]
    \definecolor{figbrown}{RGB}{185,75,0}
    \definecolor{figblue}{RGB}{0,90,170}
    \def\bottom{-2.2} \def\top{2.4}
    \node[anchor=west] at (-.6,\top+2.1) {(a) the function $g_2^+$};
    \foreach \xx in {0,3.45,4.95,6.45} {
        \draw[black!25, densely dashed] (\xx,\bottom) -- (\xx,\top);
    }
    \draw[figbrown!45, densely dotted] (1.15,\bottom) -- (1.15,\top);
    \draw[figbrown!45, densely dotted] (7.9,\bottom) -- (7.9,\top);
    \draw[figblue!45, densely dotted] (2.3,\bottom) -- (2.3,\top);
    \draw[->, black!55] (-.6,0) -- (8.7,0) node[anchor=west] {$x$};
    \draw[decorate, decoration={brace, amplitude=3pt}, black!60]
        (0,\top+.15) -- (4.95,\top+.15)
        node[midway, above=3pt, font=\scriptsize] {group $i$};
    \draw[decorate, decoration={brace, amplitude=3pt}, black!60]
        (6.45,\top+.15) -- (8.7,\top+.15)
        node[midway, above=3pt, font=\scriptsize] {group $i+1$};
    \fill[figbrown!15] (1.15,0) -- (3.45,1.725) -- (4.95,0) -- cycle;
    \fill[figbrown!15] (7.9,0) -- (8.6,.665) -- (8.6,0) -- cycle;
    \draw[figbrown!60, thick]
        (-.6,-.345) -- (0,-.8625) -- (3.45,1.725) -- (4.95,0)
        -- (6.45,-1.3775) -- (8.6,.665);
    \draw[figbrown, very thick]
        (-.6,0) -- (1.15,0) -- (3.45,1.725) -- (4.95,0)
        -- (7.9,0) -- (8.6,.665);
    \node[figbrown!80!black, anchor=north west, font=\scriptsize]
        at (.2,-.95) {$g_2^+$};
    \node[figbrown!80!black, anchor=south west, font=\scriptsize]
        at (3.9,1.65) {$\sigma(g_2^+)$};
    \foreach \xx/\lab in {0/{x_{i,1}},3.45/{x_{i,N}},4.95/{x_{i,N+1}},6.45/{x_{i+1,1}}} {
        \fill[black!60] (\xx,0) circle (1.6pt);
        \node[anchor=north, font=\scriptsize] at (\xx,\bottom) {$\lab$};
    }
    \foreach \xx/\lab in {1.15/{x_{i,2}},7.9/{x_{i+1,2}}} {
        \draw[figbrown, fill=white, thick] (\xx,0) circle (1.6pt);
        \node[anchor=north, font=\scriptsize] at (\xx,\bottom) {$\lab$};
    }
    \node[anchor=north, font=\scriptsize] at (2.3,\bottom) {$x_{i,3}$};
\end{tikzpicture} \hfill
    \begin{tikzpicture}[font=\small, x=.34cm, y=.62cm, baseline=(current bounding box.north)]
    \definecolor{figbrown}{RGB}{185,75,0}
    \definecolor{figblue}{RGB}{0,90,170}
    \definecolor{figink}{RGB}{60,60,60}
    \node[anchor=west] at (-5.4,3.05) {(b) the decomposition};
    \draw[black!20, densely dashed] (0.000,-6.800000000000001) -- (0.000,1.7);
    \draw[figbrown!45, densely dotted] (1.000,-6.800000000000001) -- (1.000,1.7);
    \draw[figblue!45, densely dotted] (2.000,-6.800000000000001) -- (2.000,1.7);
    \draw[black!20, densely dashed] (3.000,-6.800000000000001) -- (3.000,1.7);
    \draw[black!20, densely dashed] (4.800,-6.800000000000001) -- (4.800,1.7);
    \draw[figbrown!45, densely dotted] (5.800,-6.800000000000001) -- (5.800,1.7);
    \draw[figblue!45, densely dotted] (6.800,-6.800000000000001) -- (6.800,1.7);
    \draw[black!20, densely dashed] (7.800,-6.800000000000001) -- (7.800,1.7);
    \draw[black!20, densely dashed] (9.600,-6.800000000000001) -- (9.600,1.7);
    \draw[figbrown!45, densely dotted] (10.600,-6.800000000000001) -- (10.600,1.7);
    \draw[figblue!45, densely dotted] (11.600,-6.800000000000001) -- (11.600,1.7);
    \draw[black!20, densely dashed] (12.600,-6.800000000000001) -- (12.600,1.7);
    \draw[black!20, densely dashed] (14.400,-6.800000000000001) -- (14.400,1.7);
    \draw[figbrown!45, densely dotted] (15.400,-6.800000000000001) -- (15.400,1.7);
    \draw[figblue!45, densely dotted] (16.400,-6.800000000000001) -- (16.400,1.7);
    \draw[black!20, densely dashed] (17.400,-6.800000000000001) -- (17.400,1.7);
    \draw[black!20, densely dashed] (3.600,-6.800000000000001) -- (3.600,1.7);
    \draw[black!20, densely dashed] (8.400,-6.800000000000001) -- (8.400,1.7);
    \draw[black!20, densely dashed] (13.200,-6.800000000000001) -- (13.200,1.7);
    \draw[black!50, decorate, decoration={brace, amplitude=3pt}] (0.000,1.8499999999999999) -- (3.600,1.8499999999999999) node[midway, above=3pt, font=\scriptsize] {group 1};
    \draw[black!50, decorate, decoration={brace, amplitude=3pt}] (4.800,1.8499999999999999) -- (8.400,1.8499999999999999) node[midway, above=3pt, font=\scriptsize] {group 2};
    \draw[black!50, decorate, decoration={brace, amplitude=3pt}] (9.600,1.8499999999999999) -- (13.200,1.8499999999999999) node[midway, above=3pt, font=\scriptsize] {group 3};
    \draw[black!50, decorate, decoration={brace, amplitude=3pt}] (14.400,1.8499999999999999) -- (17.400,1.8499999999999999) node[midway, above=3pt, font=\scriptsize] {group 4};
    \draw[black!40] (-0.6,0) -- (18.0,0);
    \draw[figink, very thick] (-0.600,0.300) -- (0.000,0.300) -- (1.000,-0.250) -- (2.000,-0.350) -- (3.000,0.150) -- (3.600,0.300) -- (4.800,0.600) -- (5.800,0.100) -- (6.800,0.300) -- (7.800,0.850) -- (8.400,0.633) -- (9.600,0.200) -- (10.600,-0.250) -- (11.600,-0.300) -- (12.600,0.200) -- (13.200,0.300) -- (14.400,0.500) -- (15.400,0.000) -- (16.400,0.100) -- (17.400,0.650) -- (18.000,0.650);
    \fill[black!60] (0.000,0.300) circle (1.5pt);
    \draw[figbrown, fill=white, thick] (1.000,-0.250) circle (1.6pt);
    \draw[figblue, fill=white, thick] (2.000,-0.350) circle (1.6pt);
    \fill[black!60] (3.000,0.150) circle (1.5pt);
    \fill[black!60] (4.800,0.600) circle (1.5pt);
    \draw[figbrown, fill=white, thick] (5.800,0.100) circle (1.6pt);
    \draw[figblue, fill=white, thick] (6.800,0.300) circle (1.6pt);
    \fill[black!60] (7.800,0.850) circle (1.5pt);
    \fill[black!60] (9.600,0.200) circle (1.5pt);
    \draw[figbrown, fill=white, thick] (10.600,-0.250) circle (1.6pt);
    \draw[figblue, fill=white, thick] (11.600,-0.300) circle (1.6pt);
    \fill[black!60] (12.600,0.200) circle (1.5pt);
    \fill[black!60] (14.400,0.500) circle (1.5pt);
    \draw[figbrown, fill=white, thick] (15.400,0.000) circle (1.6pt);
    \draw[figblue, fill=white, thick] (16.400,0.100) circle (1.6pt);
    \fill[black!60] (17.400,0.650) circle (1.5pt);
    \node[anchor=east] at (-0.8,0.3) {$f$};
    \draw[black!40] (-0.6,-2.0) -- (18.0,-2.0);
    \fill[figbrown!15] (1.000,-2.000) -- (2.000,-1.550) -- (3.000,-1.100) -- (3.600,-2.000) -- (3.600,-2.000) -- cycle;
    \fill[figbrown!15] (5.800,-2.000) -- (6.800,-1.300) -- (7.800,-0.600) -- (8.400,-2.000) -- (8.400,-2.000) -- cycle;
    \fill[figbrown!15] (10.600,-2.000) -- (11.600,-1.600) -- (12.600,-1.200) -- (13.200,-2.000) -- (13.200,-2.000) -- cycle;
    \fill[figbrown!15] (15.400,-2.000) -- (16.400,-1.400) -- (17.400,-0.800) -- (18.000,-0.800) -- (18.000,-2.000) -- cycle;
    \draw[figbrown, very thick] (-0.600,-2.000) -- (0.000,-2.000) -- (1.000,-2.000) -- (2.000,-1.550) -- (3.000,-1.100) -- (3.600,-2.000) -- (4.800,-2.000) -- (5.800,-2.000) -- (6.800,-1.300) -- (7.800,-0.600) -- (8.400,-2.000) -- (9.600,-2.000) -- (10.600,-2.000) -- (11.600,-1.600) -- (12.600,-1.200) -- (13.200,-2.000) -- (14.400,-2.000) -- (15.400,-2.000) -- (16.400,-1.400) -- (17.400,-0.800) -- (18.000,-0.800);
    \draw[figbrown, fill=white, thick] (1.000,-2.000) circle (1.6pt);
    \draw[figbrown, fill=white, thick] (5.800,-2.000) circle (1.6pt);
    \draw[figbrown, fill=white, thick] (10.600,-2.000) circle (1.6pt);
    \draw[figbrown, fill=white, thick] (15.400,-2.000) circle (1.6pt);
    \node[anchor=east] at (-0.8,-1.7) {$=\,\sigma(g_2^+)$};
    \draw[black!40] (-0.6,-3.75) -- (18.0,-3.75);
    \fill[figblue!15] (2.000,-3.750) -- (3.000,-3.150) -- (3.600,-3.750) -- (3.600,-3.750) -- cycle;
    \fill[figblue!15] (6.800,-3.750) -- (7.800,-3.400) -- (8.400,-3.750) -- (8.400,-3.750) -- cycle;
    \fill[figblue!15] (11.600,-3.750) -- (12.600,-3.200) -- (13.200,-3.750) -- (13.200,-3.750) -- cycle;
    \fill[figblue!15] (16.400,-3.750) -- (17.400,-3.300) -- (18.000,-3.300) -- (18.000,-3.750) -- cycle;
    \draw[figblue, very thick] (-0.600,-3.750) -- (0.000,-3.750) -- (1.000,-3.750) -- (2.000,-3.750) -- (3.000,-3.150) -- (3.600,-3.750) -- (4.800,-3.750) -- (5.800,-3.750) -- (6.800,-3.750) -- (7.800,-3.400) -- (8.400,-3.750) -- (9.600,-3.750) -- (10.600,-3.750) -- (11.600,-3.750) -- (12.600,-3.200) -- (13.200,-3.750) -- (14.400,-3.750) -- (15.400,-3.750) -- (16.400,-3.750) -- (17.400,-3.300) -- (18.000,-3.300);
    \draw[figblue, fill=white, thick] (2.000,-3.750) circle (1.6pt);
    \draw[figblue, fill=white, thick] (6.800,-3.750) circle (1.6pt);
    \draw[figblue, fill=white, thick] (11.600,-3.750) circle (1.6pt);
    \draw[figblue, fill=white, thick] (16.400,-3.750) circle (1.6pt);
    \node[anchor=east] at (-0.8,-3.45) {$+\,\sigma(g_3^+)$};
    \draw[black!40] (-0.6,-5.7) -- (18.0,-5.7);
    \draw[figink, very thick] (-0.600,-5.400) -- (0.000,-5.400) -- (1.000,-5.950) -- (2.000,-6.500) -- (3.000,-7.050) -- (3.600,-5.400) -- (4.800,-5.100) -- (5.800,-5.600) -- (6.800,-6.100) -- (7.800,-6.600) -- (8.400,-5.067) -- (9.600,-5.500) -- (10.600,-5.950) -- (11.600,-6.400) -- (12.600,-6.850) -- (13.200,-5.400) -- (14.400,-5.200) -- (15.400,-5.700) -- (16.400,-6.200) -- (17.400,-6.700) -- (18.000,-6.700);
    \fill[black!60] (0.000,-5.400) circle (1.5pt);
    \fill[black!60] (3.000,-7.050) circle (1.5pt);
    \fill[black!60] (3.600,-5.400) circle (1.5pt);
    \fill[black!60] (4.800,-5.100) circle (1.5pt);
    \fill[black!60] (7.800,-6.600) circle (1.5pt);
    \fill[black!60] (8.400,-5.067) circle (1.5pt);
    \fill[black!60] (9.600,-5.500) circle (1.5pt);
    \fill[black!60] (12.600,-6.850) circle (1.5pt);
    \fill[black!60] (13.200,-5.400) circle (1.5pt);
    \fill[black!60] (14.400,-5.200) circle (1.5pt);
    \fill[black!60] (17.400,-6.700) circle (1.5pt);
    \node[anchor=east] at (-0.8,-5.4) {$+\,h$};
\end{tikzpicture}
    \caption{The construction for $N = 4$ with positive
        slope jumps. (a) The function $g_2^+$ crosses zero at $x_{i,2}$ in every group,
        returns to zero at $x_{i,N+1}$, and is nonpositive until the next
        group. The thicker curve shows $\sigma(g_2^+)$, including its flat zero segments.
        (b) Each $\sigma(g_j^+)$ has one knot in every group
        (dotted lines), and $h$ has only boundary knots (dashed lines).}
    \label{fig:piecewise-affine-decomposition}
    \label{fig:piecewise-affine}
\end{figure}
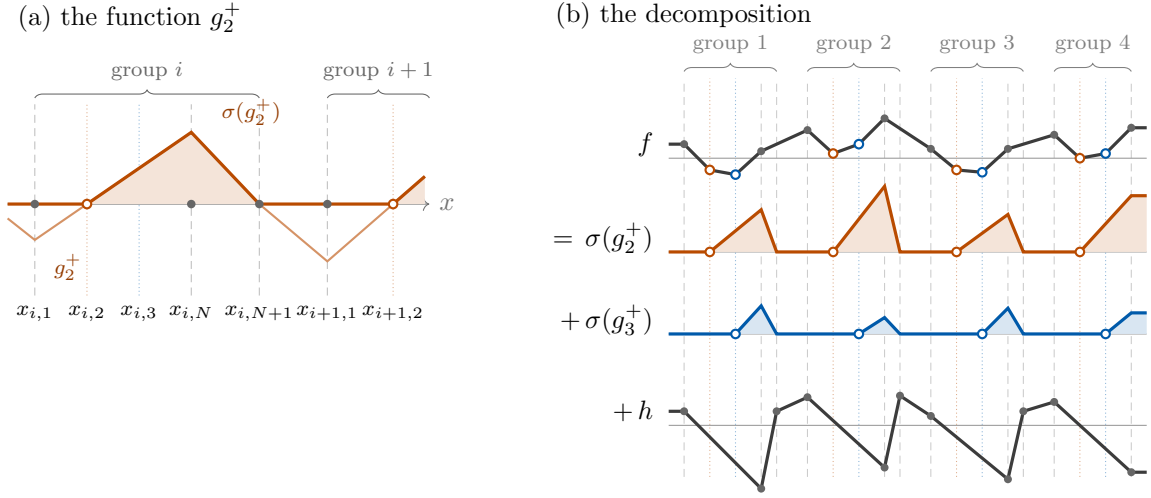

To exploit this, we arrange the knots in $N$ groups of $N$ and call the first and last knot of each group, together with one auxiliary knot after it, boundary knots.
For each interior position $j$, let $g_j^+$ be a CPwL function with hinges only at boundary knots that crosses zero at the $j$-th knot of every group, with slope equal to the positive part of the slope jump of $f$ there; see \cref{fig:piecewise-affine}(a). Define $g_j^-$ in the same way with the negative parts. Then
\begin{align*}
    f &= h+\sum_{j=2}^{N-1}\bigl(\sigma(g_j^+)-\sigma(g_j^-)\bigr),
\end{align*}
where $h$ has hinges only at boundary knots; see
\cref{fig:piecewise-affine-decomposition}(b).

Since $h$ and all $g_j^\pm$ have hinges only at the $O(N)$ boundary knots, one hidden layer computes these hinges and a second one the rectified functions, so $N^2$ knots are fitted with width $O(N)$ and depth $O(1)$. To fit $N^2L$ knots, we add the increments of $f$ over at most $2L$ consecutive blocks of $N^2$ knots in a running sum.
The
resulting network has width $O(N)$ and depth $O(L)$, as claimed.

\subsection{Storing and Retrieving Local Taylor Coefficients}
\label{sec:storing-retrieving-coefficients}

We now turn to the decisive task identified in \cref{sec:taylor-network-review}: storing and retrieving the local Taylor coefficients.
Existing approaches broadly separate into explicit storage and encoded
retrieval. \citet{schmidt2020nonparametric} approximates cell indicators that select from explicitly stored Taylor coefficients. The
wide-network construction of \citet{kohler2021rate} stores the coefficients in the connections between two wide layers.
Encoded retrieval stores a whole batch of coefficients in a single number and uses a subnetwork to decode them.
It is often formulated as a point-fitting problem:
given pairwise distinct inputs $t_1,\ldots,t_J$ and values
$y_1,\ldots,y_J$, the problem is to construct a network satisfying
\begin{align*}
    \phi(t_j)=y_j, \qquad j=1,\ldots,J.
\end{align*}
For local Taylor approximation, $t_j$ encodes the $j$th grid point $\bx_j$, and
   $y_j= \partial^{\balpha}f(\bx_j) /\balpha!$.
Encoded point fitting stores a digit string in the network weights and
uses a subnetwork to decode it. The bit-extraction device for the latter goes
back to \citet{bartlett1998almost}. It stores binary digits $b_1,b_2,\ldots\in\{0,1\}$ in the number
$\theta=\sum_{k\ge1}b_k2^{-k} \in [0, 1]$. The leading digit and the remainder are
\begin{align*}
    b_1
    &=\ind\{\theta\ge1/2\},
    \qquad
    2\theta-b_1
    =\sum_{k\ge1}b_{k+1}2^{-k},
\end{align*}
so iterating the
map $\theta\mapsto(b_1,2\theta-b_1)$ reads the digits one by one.

\begin{figure}[tp]
    \centering
    \begin{tikzpicture}[font=\small]
    \definecolor{figblue}{RGB}{0,90,170}
    \definecolor{figbrown}{RGB}{185,75,0}
    \definecolor{figgold}{RGB}{170,125,0}
    \definecolor{figpurple}{RGB}{110,60,140}
    \tikzset{
        lab/.style={font=\scriptsize},
        digit/.style={figblue, very thick},
        rem/.style={figbrown, thick, densely dashed},
        pt/.style={circle, fill=black!70, inner sep=0pt, minimum size=3.3pt},
    }

    \begin{scope}[x=6cm, y=1.7cm]
        \node[anchor=west] at (-0.08, 1.45) {(a) binary, $m = 4$};
        \draw[black!45] (0, 0) -- (1.04, 0) node[anchor=west, lab, black] {$\theta$};
        \draw[black!45] (0, 0) -- (0, 1.15);
        \draw[black!45] (-0.012, 1) -- (0.012, 1) node[lab, anchor=east, black!60, xshift=-4pt] {$1$};
        \draw[rem] (0, 0) -- (0.4375, 0.875) -- (0.5, 0) -- (0.9375, 0.875) -- (1, 0.875);
        \draw[digit] (0, 0) -- (0.4375, 0) -- (0.5, 1) -- (1, 1);
        \foreach \k in {0,...,15} {
            \node[pt] at ({\k/16}, 0) {};
        }
        \draw[decorate, decoration={brace, mirror, amplitude=2.5pt}, black!60]
            (0.4375, -0.12) -- (0.5, -0.12) node[midway, below=3pt, lab] {$2^{-m}$};
        \node[lab, figblue] at (0.67, 1.12) {$b_1$};
        \node[lab, figbrown] at (0.12, 0.7) {remainder};
        \foreach \a/\lb in {0/0, 1/1} {
            \node[lab, anchor=north] at (\a, -0.04) {$\lb$};
        }
    \end{scope}

    \begin{scope}[shift={(7.7cm,0)}, x=12cm, y=1.7cm]
        \node[anchor=west] at (-0.04, 1.45) {(b) ternary, digits $\{0,1\}$, $m = 4$};
        \fill[black!8] (0.16667, 0) rectangle (0.33333, 1.15);
        \draw[black!45] (0, 0) -- (0.52, 0) node[anchor=west, lab, black] {$\theta$};
        \draw[black!45] (0, 0) -- (0, 1.15);
        \draw[black!45] (-0.006, 1) -- (0.006, 1) node[lab, anchor=east, black!60, xshift=-4pt] {$1$};
        \draw[rem] (0, 0) -- (0.16667, 0.5) -- (0.33333, 0) -- (0.5, 0.5) -- (0.52, 0.5);
        \draw[digit] (0, 0) -- (0.16667, 0) -- (0.33333, 1) -- (0.52, 1);
        \foreach \t in {0.0000, 0.0123, 0.0370, 0.0494, 0.1111, 0.1235, 0.1481, 0.1605,
                        0.3333, 0.3457, 0.3704, 0.3827, 0.4444, 0.4568, 0.4815, 0.4938} {
            \node[pt] at (\t, 0) {};
        }
        \draw[decorate, decoration={brace, mirror, amplitude=2.5pt}, black!60]
            (0.16667, -0.12) -- (0.33333, -0.12) node[midway, below=3pt, lab] {gap $\tfrac16$};
        \node[lab, figblue] at (0.42, 1.12) {$b_1$};
        \node[lab, figbrown] at (0.07, 0.58) {remainder};
        \foreach \a/\lb in {0/0, 0.5/\tfrac12} {
            \node[lab, anchor=north] at (\a, -0.04) {$\lb$};
        }
    \end{scope}
\end{tikzpicture}
    \caption{Extracting the leading digit $b_1$ and the remainder from a
        stored string of length $m = 4$; dots mark the admissible inputs.
        (a) In binary, $b_1$ must jump between neighboring inputs separated by $2^{-m}$. (b) In
        ternary with digits $\{0,1\}$, the gap $(1/6, 1/3)$ allows slope $6$
        for every $m$.}
    \label{fig:bit-extraction}
\end{figure}
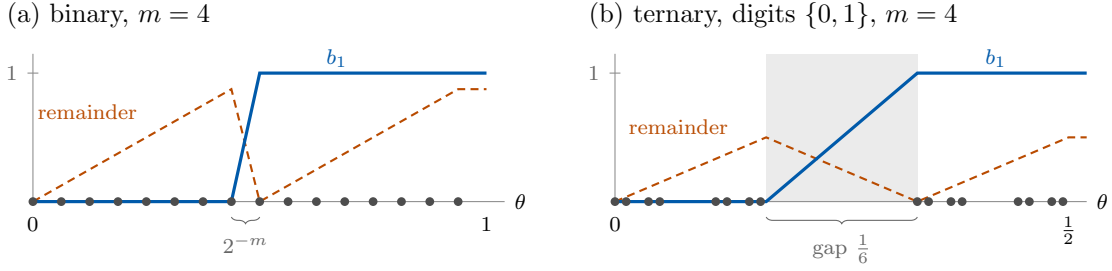

ReLU networks are continuous, so they can only approximate the jump of
$b_1$ at $1/2$ by a ramp. If $m$ digits are stored, the inputs fill the grid
$2^{-m}\Z$, and the two neighboring inputs $1/2-2^{-m}$ and $1/2$ have
different leading digits; see \cref{fig:bit-extraction}(a). A ramp that is exact on this grid therefore needs slope at least $2^m$. With $m\asymp L\log N$ stored
digits, the weights grow exponentially in $L$, as
in the binary decoder of \citet{shen2022optimal}.
\citet{ou2024three} store the same digits in base $3$, as
$\theta=\sum_{k\ge1}b_k3^{-k}$. The digit $2$ is never
used, so the inputs with $b_1=0$ lie in $[0,1/6]$ and those with $b_1=1$ in
$[1/3,1/2]$. The gap between them does not depend on $m$, and the ramp
    $b_1
    =\sigma(6\theta-1)-\sigma(6\theta-2)$
has slope $6$; see \cref{fig:bit-extraction}(b).
A single hidden layer can even read $n \asymp \log N$ digits at once: the inputs then fall into $2^n$ clusters
separated by gaps of order $3^{-n}$, and slopes of order $3^n$ suffice.

The
next proposition extends this mechanism to base $\zeta \ge 3$ with reduced alphabet
$\{0,\dots,\zeta-2\}$. This lets us store the bounded integers of
\cref{sec:higher-order-point-fitting} directly, one digit per value. Its proof
is given in \cref{sec:digit-extraction-proof}.

\begin{proposition}[Digit extraction]
    \label{prop:digit-extraction}
    Let $\zeta, n \in \N_+$ with $\zeta \ge 3$. There is
    \begin{align*}
        \phi
        &\in \Fcal_{1,n+1}\left(2\zeta^n, 1, 2\zeta^{n + 2}\right)
    \end{align*}
    such that, for all $b_1, b_2, \ldots \in \{0, \dots, \zeta-2\}$,
    \begin{align*}
        \phi\left(\sum_{k=1}^\infty b_k \zeta^{-k}\right)
        &= \left(b_1, \dots, b_n, \sum_{k=1}^\infty b_{n+k}\zeta^{-k}\right).
    \end{align*}
\end{proposition}

Choosing $n \asymp \log_\zeta N$, we can use $O(L)$ hidden layers of width $O(N)$ to read $L \log N$
base-$\zeta$ digits. At this size, \Cref{prop:pw-linear} can encode one real-valued number for each of $N^2L$ coefficient blocks, so a network of width $O(N)$ and
depth $O(L)$ can store and retrieve $N^2L^2\log N$ digits. This matches the number
$K^D \asymp N^2L^2\log N$ of coefficients, provided each coefficient costs
only a bounded number of digits. A coefficient stored on its own to accuracy
$K^{-s}$ costs order $\log K\asymp\log(NL)$ digits, however. Treating the coefficients as
unrelated as in \citet{Lu2021} and \citet{nagler2026optimal} therefore leads to logarithmic overhead.

\subsection{Higher-Order Point Fitting}
\label{sec:higher-order-point-fitting}

This subsection contains the main new ingredient. Smoothness allows us to compress the
sequence of coefficient values so that each value
can be stored with a bounded number of digits. It then remains to efficiently decode the stored digits and reassemble the original values.
We explain the mechanism for a
one-dimensional sequence, starting from the first-order construction of
\citet{shen2022optimal}, and treat multivariate arrays in
\cref{sec:multivariate-coefficients}.

Let
$x_1,\ldots,x_J$ be consecutive grid points separated by $K^{-1}$ and set
$y_i=f(x_i)$. A point-fitting network is queried at the normalized indices
$i/J$ and should satisfy
\begin{align*}
    \phi(i/J)\approx y_i,
    \qquad i=1,\ldots,J.
\end{align*}
For the zero-order Taylor coefficient, the required accuracy is $K^{-s}$.

\medskip
\noindent\textit{First-order fitting.}
For H\"older functions of order at most
one, \citet{shen2022optimal} combine bit extraction with a transparent
predict-and-correct construction. Suppose the ordered targets satisfy
$|y_i-y_{i-1}|\le\lambda$. To approximate $y_i$, we instead fit the quantized and scaled sequence
$a_i=\floor{y_i/\lambda}$ exactly, noting that $|\lambda a_k-y_k|<\lambda$.
We can write
\begin{align*}
    a_k=a_1+\sum_{i=2}^k d_i, \qquad \text{with}\qquad  d_i = a_i - a_{i-1} \in \{-1,0,1\}.
\end{align*}
The values $a_1$ and $d_2,\ldots,d_J$ can be digit-encoded into $N^2L$ real values stored in the network weights with \cref{prop:pw-linear}, and the extraction network from \cref{prop:digit-extraction} can decode each into $L \log N $ of the original values. The final output $\lambda a_k$ approximates $y_k$ to accuracy $\lambda$.

\medskip
\noindent\textit{Higher-order differences.}
This first-order construction is insufficient when $s>1$. The values must
be fitted to accuracy $K^{-s}$, but consecutive values only satisfy
$|y_i-y_{i-1}|\lesssim K^{-1}$.
The key observation is that higher-order smoothness of $f$ makes the sequence even more compressible.
Set $h=\lceil s\rceil$ and define $\Delta^0y_i=y_i$ and
\begin{align}
    \label{eq:differences}
    \Delta^r y_i
    &=\Delta^{r-1}y_{i+1}-\Delta^{r-1}y_i
    =\sum_{k=0}^r(-1)^{r-k}\binom{r}{k}y_{i+k},
    \qquad r\ge1.
\end{align}
For intuition, consider $s=2$ and write $\delta=K^{-1}$. A first-difference
bound $|\Delta y_i|\lesssim\delta$ alone still permits
$|\Delta^2y_i|\asymp\delta$. Samples of a twice-smooth function instead
satisfy $|\Delta^2y_i|\lesssim\delta^2$, as illustrated in
\cref{fig:difference-orders}.
\begin{figure}[tp]
    \centering
    \begin{tikzpicture}[>=Stealth, font=\small]
    \definecolor{figblue}{RGB}{0,90,170}
    \definecolor{figbrown}{RGB}{185,75,0}
    \definecolor{figgold}{RGB}{170,125,0}
    \definecolor{figpurple}{RGB}{110,60,140}
    \colorlet{band}{black!12}
    \def\deltav{0.1}   %
    \def\xs{0.42}      %
    \def\colB{6.6}     %
    \def\rowD{-1.4}    %
    \def\rowDD{-3.0}   %

    \def\seqA{0.2,0.3,0.25,0.33,0.26,0.32,0.23,0.32,0.26,0.36,0.31,0.38}

    \node[anchor=south west, font=\small, text width=5.0cm, align=left] at (0.2, 2.35)
        {(a) only small first differences:\\ $|\Delta y_i| \le \delta$, $|\Delta^2 y_i| \approx 3\delta/2$};
    \node[anchor=south west, font=\small, text width=5.0cm, align=left] at (\colB+0.2, 2.35)
        {(b) small second differences:\\ $|\Delta y_i| \le \delta$, $|\Delta^2 y_i| \le \delta^2$};
    \node[anchor=east, font=\small] at (-0.15, 1.1) {$y_i$};
    \node[anchor=east, font=\small] at (-0.15, \rowD) {$\Delta y_i$};
    \node[anchor=east, font=\small] at (-0.15, \rowDD) {$\Delta^2 y_i$};

    \foreach \col/\xoff/\which in {figbrown/0/0, figblue/\colB/1} {
        \ifnum\which=0
            \let\vals\seqA
        \else
            \def\vals{}
            \foreach \i in {1,...,12} {
                \pgfmathsetmacro{\v}{sin(deg(0.2+\deltav*(\i-1)))-0.15}
                \ifx\vals\empty
                    \xdef\vals{\v}
                \else
                    \xdef\vals{\vals,\v}
                \fi
            }
        \fi
        \begin{scope}[shift={(\xoff, 0)}, x=\xs cm, y=2.2cm]
            \draw[->] (0.5, 0) -- (12.8, 0) node[anchor=west, font=\scriptsize] {$i$};
            \foreach \i in {1,...,12} { \draw (\i, 0.012) -- (\i, -0.012); }
            \node[font=\scriptsize, anchor=north] at (1, -0.03) {};
            \node[font=\scriptsize, anchor=north] at (12, -0.03) {};
            \def\prev{}
            \foreach \v [count=\i] in \vals {
                \fill[\col] (\i, \v) circle (1.6pt);
                \ifx\prev\empty\else
                    \pgfmathtruncatemacro{\ip}{\i-1}
                    \draw[\col!50, thin] (\ip, \prev) -- (\i, \v);
                \fi
                \xdef\prev{\v}
            }
        \end{scope}
        \foreach \row/\yoff/\ord in {2/\rowD/1, 3/\rowDD/2} {
            \begin{scope}[shift={(\xoff, \yoff)}, x=\xs cm, y=4cm]
                \pgfmathsetmacro{\bnd}{\deltav^\ord}
                \fill[band] (0.5, -\bnd) rectangle (12.8, \bnd);
                \draw[black!40] (0.5, 0) -- (12.8, 0);
                \ifnum\ord=1
                    \node[anchor=west, font=\scriptsize, black!60] at (12.8, \bnd) {$\delta$};
                    \node[anchor=west, font=\scriptsize, black!60] at (12.8, -\bnd) {$-\delta$};
                \else
                    \node[anchor=west, font=\scriptsize, black!60] at (12.8, 0) {$\pm\delta^2$};
                \fi
                \def\prevA{}\def\prevB{}
                \foreach \v [count=\i] in \vals {
                    \ifnum\ord=1
                        \ifx\prevA\empty\else
                            \pgfmathsetmacro{\d}{\v-\prevA}
                            \pgfmathtruncatemacro{\ip}{\i-1}
                            \draw[\col, thick] (\ip, 0) -- (\ip, \d);
                            \fill[\col] (\ip, \d) circle (1.4pt);
                        \fi
                    \else
                        \ifx\prevB\empty\else
                            \pgfmathsetmacro{\d}{\v-2*\prevA+\prevB}
                            \pgfmathtruncatemacro{\ip}{\i-2}
                            \draw[\col, thick] (\ip, 0) -- (\ip, \d);
                            \fill[\col] (\ip, \d) circle (1.4pt);
                        \fi
                    \fi
                    \xdef\prevB{\prevA}
                    \xdef\prevA{\v}
                }
            \end{scope}
        }
    }
\end{tikzpicture}
    \caption{Two sequences satisfying the same first-difference bound
        $|\Delta y_i|\le\delta$. In (a), oscillations at the finest admissible
        scale leave second differences of order $\delta$. In (b), discrete
        smoothness reduces the second differences to order $\delta^2$.}
    \label{fig:difference-orders}
\end{figure}
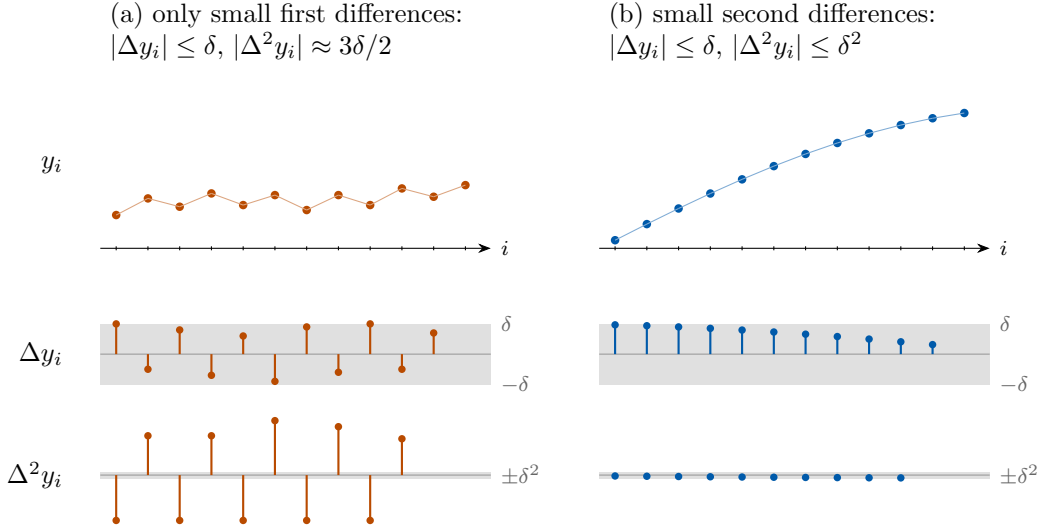
The discrete Taylor--Newton identity \citep[Theorem~1.8.5]{agarwal2000difference}
\begin{align}
    \label{eq:discrete-taylor}
    y_i
    &=\sum_{r=0}^{h-1}\binom{i-1}{r}\Delta^r y_1
    +\sum_{\ell=1}^{i-h}\binom{i-\ell-1}{h-1}\Delta^h y_\ell
\end{align}
expresses the sequence through a fixed number of initial differences of orders below
$h$ and its $h$th differences.
Further, H\"older smoothness and the grid spacing imply
\begin{align*}
    |\Delta^r y_i|
    &\lesssim K^{-\min\{r, s\}},
    \qquad 0\le r \le h.
\end{align*}
In particular, $|\Delta^h y_i|\lesssim K^{-s}$ reaches the required accuracy.
As in the first-order case, we set $\lambda\asymp K^{-s}$ and fit the
quantized sequence $a_i=\floor{y_i/\lambda}$. Its $h$th differences
$d_i=\Delta^h a_i$ are integers with $|d_i|\le 2^{h+2}=:H$, so  $d_i+H \in \{0,\dots,2H\}$. \cref{prop:digit-extraction} with
base $\zeta=2H+2$ decodes one $d_i$ per digit and the Taylor--Newton identity \eqref{eq:discrete-taylor}
recovers $a_k$ from the initial differences $\Delta^r a_1$,
$r<h$, and from $d_1,\dots,d_{k-h}$. Then $\lambda a_k$ approximates $y_k$
to the required accuracy $\lambda \asymp K^{-s}$.

The following result shows that a network of width $N$ and depth
$L$ can achieve this error at $J\lesssim N^2L^2\log N$ samples.

\begin{proposition}[Smooth point fitting]\label{prop:point-fitting-diff}
    Let $\lambda\in(0,1)$ and $h,N,L,q\in\N_+$, set $H=2^{h+2}$ and $\zeta=2H+2$, and suppose that
    $q\le\floor{\log_\zeta N}$ and $N\ge\zeta$. Suppose further that $J\le3N^2L^2q$ and
    that $y_1,\dots,y_J\in\R$ satisfy
    \begin{align*}
        \max_{0\le r\le h-1}\max_{1\le i\le J-r}|\Delta^r y_i|\le1,
        \qquad
        \max_{1\le i\le J-h}|\Delta^h y_i|\le\lambda.
    \end{align*}
    Then there is $\phi\in\wt\Fcal_{1,1}(N,L,N)$, with constants depending only on $h$, such that
    \begin{align*}
        \max_{1\le i\le J}|\phi(i/J)-y_i|\le\lambda.
    \end{align*}
\end{proposition}

For $h=1$, this recovers the first-order mechanism underlying the
constructions of \citet{shen2022optimal} and \citet{ou2024three}. The extension
to $h>1$ is the main new ingredient of our proof: it uses the available
higher-order smoothness to reach the finer accuracy required for lower-order
Taylor coefficients.
Point fitting with arbitrary labels relates directly to the networks' memorization capability \citep{park2021provable,vardi2022on}. Networks of width $N$ and depth $L$ can memorize $J\lesssim N^2L^2$ labels $y_i$ at separated locations $x_i$, as long as both separation distance and label resolution are polynomial in $N, L$. Constructions attaining this order are given by \citet{nagler2026optimal} and \citet{yang2026memorization}. \Cref{prop:point-fitting-diff} fits $J\asymp N^2L^2\log N$ values instead: the smoothness of the labels gains exactly the factor $\log N$ needed for \cref{thm:main-approx}.

\medskip
\noindent\textit{Evaluating the Taylor identity.}
The main obstacle in the
proof of \cref{prop:point-fitting-diff} is to evaluate the discrete Taylor
identity by a network with sufficiently bounded parameters.
The sequence $y_i$ is split into $N^2L$ blocks, each encoded into the network weights using \cref{prop:pw-linear}. Over a block of length
$Lq$, the binomial coefficients in the Taylor--Newton identity reach
order $(Lq)^{h-1}$, so evaluating the identity in one step would
require weights of the same magnitude.

We instead cut each block into $L$
segments of $q$ positions, with $\zeta^q\le N$, and evaluate the identity recursively, one
segment at a time.
Segment $t$ covers the positions $tq+1,\dots,tq+q$ of a block. Its state
\begin{align*}
    \bu^{(t)}
    &=\bigl(\Delta^0a_{tq+1},\dots,\Delta^{h-1}a_{tq+1}\bigr)
\end{align*}
contains the differences of the quantized sequence at the segment's first position.
Applying the Taylor--Newton identity to the shifted sequence
$(a_{tq+j})_{j\ge1}$ expresses the values in the segment through this state $\bu^{(t)}$
and the digits $d_{tq+1},\dots,d_{tq+q}$:
\begin{align*}
    a_{tq+i}
    &=\sum_{r=0}^{h-1}\binom{i-1}{r}u^{(t)}_r
    +\sum_{j=1}^{i-h}\binom{i-j-1}{h-1}d_{tq+j},
    \qquad i=1,\dots,q.
\end{align*}
Furthermore, applying the Taylor--Newton identity again with order $h-r$ to the sequence $(\Delta^ra_{tq+j})_{j\ge1}$
at position $q+1$ gives the state update
\begin{align*}
    u^{(t+1)}_r
    &=\sum_{p=0}^{h-r-1}\binom{q}{p}u^{(t)}_{r+p}
    +\sum_{j=1}^{q-h+r+1}\binom{q-j}{h-r-1}d_{tq+j},
    \qquad r=0,\dots,h-1.
\end{align*}
Starting from the state $\bu^{(0)}$ stored for the block, the network
processes one segment per step, decodes the $q$ digits of the segment
with \cref{prop:digit-extraction}, evaluates both formulas, and passes
$\bu^{(t+1)}$ and the remaining digits on to the next segment. The value
$a_k$ is read from the segment that contains $k$.

All binomial coefficients in this recursion are at most of order $q^{h-1}\lesssim(\log N)^{h-1}$, so the network parameters remain bounded by a factor polylogarithmic in $N$ and not depending on $L$. The full proof is given in \cref{sec:point-fitting-proof}.

\subsection{Smooth Flattening of Coefficient Tables}
\label{sec:multivariate-coefficients}

We now return to the multivariate Taylor approximation. For each multi-index
$\balpha$, the coefficient function
$\partial^{\balpha}f/\balpha!$ has smoothness
$s-\|\balpha\|_1$. Its values on the grid must be retrieved with the target
accuracy $\eps_{\balpha} = K^{-(s-\|\balpha\|_1)}$ determined in \cref{sec:taylor-network-review}. \Cref{prop:point-fitting-diff} is therefore applied with order
    $h_{\balpha}
    =\lceil s\rceil-\|\balpha\|_1.$
If $y_i=\partial^{\balpha}f(\bx_i)/\balpha!$ are the coefficient values at
consecutive grid points $\bx_i$ along a single coordinate direction, then
\begin{align*}
    |\Delta^r y_i|
    &\lesssim K^{-\min\{r,s-\|\balpha\|_1\}},
    \qquad 0\le r\le h_{\balpha}.
\end{align*}
Thus, lower-order Taylor coefficients have more remaining smoothness and
admit higher-order difference bounds.

When $D=1$, the ordered coefficient values directly form a sequence of the
type covered by \cref{prop:point-fitting-diff}. When $D>1$, however, flattening
the coefficient grid creates irregularities when switching between rows, destroying the finite-difference bounds.
This cannot be fixed by changing the order of the grid traversal. To see this, consider the linear function $f(\bx)=x_1$, order the grid
points arbitrarily, and let $y_i = f(\bx_i)$.
The values $y_i$, and hence their second differences, lie in $K^{-1}\Z$.
Smoothness of order larger than one requires $|\Delta^2y_i|=o(K^{-1})$, which
forces $\Delta^2y_i=0$. Such a sequence
is constant or takes every value at most once, whereas along the grid $x_1$
takes each value $K^{D-1}$ times.
We instead insert gaps into a lexicographically ordered sequence and fill them with Hermite interpolation bridges that match
derivatives up to the required order; see
\cref{fig:flattening}. The resulting flattened sequence retains
the smoothness needed in \cref{prop:point-fitting-diff}. For first-order point fitting, \citet[Section~3.3]{shen2022optimal} used a similar layout with linear interpolation. Hermite bridges provide the corresponding higher-order construction.

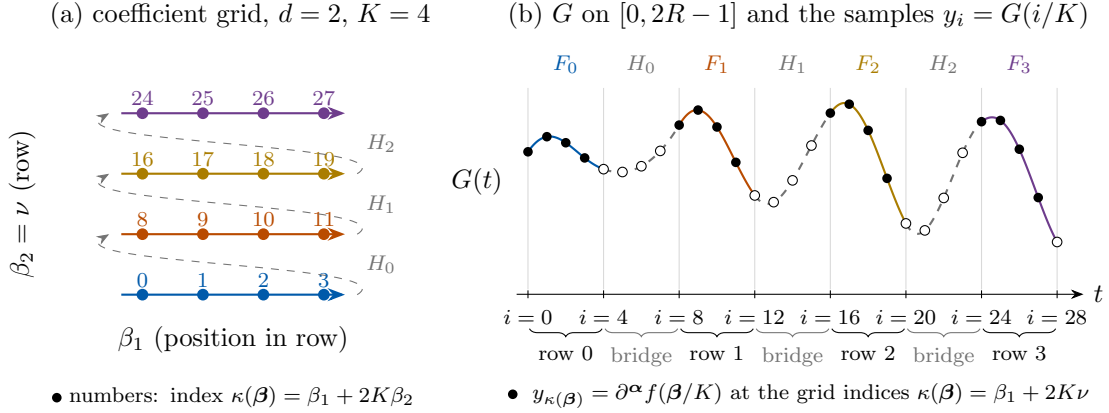
\begin{figure}[tp]
    \centering
    \begin{tikzpicture}[>=Stealth, font=\small]
    \pgfmathdeclarefunction{ff}{2}{%
        \begingroup
        \pgfmathparse{0.5*sin(deg(2.2*#1+2.5*#2))+0.3*cos(deg(3.5*#1-1.5*#2))}%
        \pgfmathsmuggle\pgfmathresult
        \endgroup}
    \pgfmathdeclarefunction{fd}{2}{%
        \begingroup
        \pgfmathparse{1.1*cos(deg(2.2*#1+2.5*#2))-1.05*sin(deg(3.5*#1-1.5*#2))}%
        \pgfmathsmuggle\pgfmathresult
        \endgroup}
    \def\K{4}
    \definecolor{figblue}{RGB}{0,90,170}
    \definecolor{figbrown}{RGB}{185,75,0}
    \definecolor{figgold}{RGB}{170,125,0}
    \definecolor{figpurple}{RGB}{110,60,140}
    \colorlet{bridge}{black!55}

    \begin{scope}[shift={(0,-1.2)}, x=0.8cm, y=0.8cm]  %
        \node[anchor=south west, font=\small] at (-1.7, 4.23)  %
            {(a) coefficient grid, $d = 2$, $K = 4$};
        \foreach \nu/\col in {0/figblue,1/figbrown,2/figgold,3/figpurple} {
            \draw[\col, ->, thick] (-0.35, \nu) -- (3.35, \nu);
            \foreach \b in {0,1,2,3} {
                \pgfmathtruncatemacro{\kap}{\b+2*\K*\nu}
                \fill[\col] (\b, \nu) circle (2.2pt);
                \node[font=\scriptsize, anchor=south, inner sep=1pt, \col]
                    at (\b, \nu+0.08) {$\kap$};
            }
        }
        \foreach \nu in {0,1,2} {
            \pgfmathtruncatemacro{\nup}{\nu+1}
            \draw[bridge, dashed, ->] (3.55, \nu) .. controls (4.4, \nu+0.5) and (-1.4, \nu+0.5) .. (-0.55, \nup);
            \node[bridge, font=\scriptsize, anchor=west, inner sep=1pt] at (3.65, \nu+0.5) {$H_{\nu}$};
        }
        \node[anchor=north, font=\small] at (1.5, -0.35) {$\beta_1$ (position in row)};
        \node[font=\small, rotate=90, anchor=south] at (-1.6, 1.5) {$\beta_2 = \nu$ (row)};
        \node[anchor=north, font=\scriptsize, text width=4.8cm, align=center] at (1.5, -1.35)
            {\tikz\fill (0,0) circle (1.8pt);  numbers: index $\kappa(\bbeta) = \beta_1 + 2K\beta_2$};
    \end{scope}

    \begin{scope}[shift={(5.1,0.3)}, x=1.0cm, y=1.3cm]
        \node[anchor=south west, font=\small] at (-0.4, 1.45)
            {(b) $G$ on $[0, 2R-1]$ and the samples $y_i = G(i/K)$};
        \draw[->] (-0.15, -1.15) -- (7.35, -1.15) node[anchor=west, font=\small] {$t$};
        \foreach \t in {0,...,7} {
            \pgfmathtruncatemacro{\ii}{\K*\t}
            \draw (\t, -1.11) -- (\t, -1.19);
            \draw[black!15, thin] (\t, -1.1) -- (\t, 0.95);
            \node[font=\scriptsize, anchor=north] at (\t, -1.21) {$i = \ii$};
        }
        \node[anchor=east, font=\small] at (-0.2, 0.0) {$G(t)$};
        \foreach \nu/\col in {0/figblue,1/figbrown,2/figgold,3/figpurple} {
            \pgfmathsetmacro{\vv}{\nu/\K}
            \draw[\col, thick] plot[domain=0:1, samples=41, smooth]
                ({2*\nu+\x}, {ff(\x,\vv)});
            \node[\col, font=\scriptsize, anchor=south] at ({2*\nu+0.5}, 1.0) {$F_{\nu}$};
        }
        \foreach \nu/\col in {0/figblue,1/figbrown,2/figgold} {
            \pgfmathtruncatemacro{\nup}{\nu+1}
            \pgfmathsetmacro{\vv}{\nu/\K}
            \pgfmathsetmacro{\vw}{\nup/\K}
            \pgfmathsetmacro{\pz}{ff(1,\vv)}
            \pgfmathsetmacro{\mz}{fd(1,\vv)}
            \pgfmathsetmacro{\po}{ff(0,\vw)}
            \pgfmathsetmacro{\mo}{fd(0,\vw)}
            \draw[bridge, thick, dashed] plot[domain=0:1, samples=41, smooth]
                ({2*\nu+1+\x},
                 {(2*\x^3-3*\x^2+1)*\pz+(\x^3-2*\x^2+\x)*\mz+(-2*\x^3+3*\x^2)*\po+(\x^3-\x^2)*\mo});
            \node[bridge, font=\scriptsize, anchor=south] at ({2*\nu+1.5}, 1.0) {$H_{\nu}$};
        }
        \pgfmathtruncatemacro{\Mmax}{2*\K*\K-\K}
        \foreach \i in {0,...,\Mmax} {
            \pgfmathsetmacro{\t}{\i/\K}
            \pgfmathtruncatemacro{\nu}{floor(\t/2)}
            \pgfmathsetmacro{\u}{\t-2*\nu}
            \pgfmathsetmacro{\vv}{\nu/\K}
            \pgfmathparse{\u<0.99}
            \ifnum\pgfmathresult=1
                \pgfmathsetmacro{\gy}{ff(\u,\vv)}
                \fill (\t, \gy) circle (1.8pt);
            \else
                \pgfmathparse{\u<1.01}
                \ifnum\pgfmathresult=1
                    \pgfmathsetmacro{\gy}{ff(1,\vv)}
                \else
                    \pgfmathtruncatemacro{\nup}{\nu+1}
                    \pgfmathsetmacro{\vw}{\nup/\K}
                    \pgfmathsetmacro{\pz}{ff(1,\vv)}
                    \pgfmathsetmacro{\mz}{fd(1,\vv)}
                    \pgfmathsetmacro{\po}{ff(0,\vw)}
                    \pgfmathsetmacro{\mo}{fd(0,\vw)}
                    \pgfmathsetmacro{\w}{\u-1}
                    \pgfmathsetmacro{\gy}{(2*\w^3-3*\w^2+1)*\pz+(\w^3-2*\w^2+\w)*\mz+(-2*\w^3+3*\w^2)*\po+(\w^3-\w^2)*\mo}
                \fi
                \draw[fill=white] (\t, \gy) circle (1.8pt);
            \fi
        }
        \foreach \nu in {0,...,3} {
            \draw[decorate, decoration={brace, mirror, amplitude=3pt}]
                (2*\nu+0.02, -1.5) -- (2*\nu+0.98, -1.5);
            \node[font=\scriptsize, anchor=north] at (2*\nu+0.5, -1.58) {row $\nu$};
        }
        \foreach \nu in {0,1,2} {
            \draw[decorate, decoration={brace, mirror, amplitude=3pt}, bridge]
                (2*\nu+1.02, -1.5) -- (2*\nu+1.98, -1.5);
            \node[font=\scriptsize, anchor=north, bridge] at (2*\nu+1.5, -1.58) {bridge};
        }
        \node[anchor=north west, font=\scriptsize, align=left] at (-0.4, -1.95)
            {\tikz\fill (0,0) circle (1.8pt); \ $y_{\kappa(\bbeta)} = \partial^{\balpha}f(\bbeta/K)$
             at the grid indices $\kappa(\bbeta) = \beta_1 + 2K\nu$\\[1pt]
             };
    \end{scope}
\end{tikzpicture}
    \caption{Flattening a two-dimensional coefficient grid. Consecutive rows
        are separated by gaps filled with Hermite bridges that match the first
        $h_{\balpha}-1$ derivatives. The resulting sequence preserves the required
        finite-difference bounds.}
    \label{fig:flattening}
\end{figure}

The following proposition makes this construction precise. Its proof is
given in \cref{sec:smooth-flattening-proof}.

\begin{proposition}[Smooth flattening of Taylor arrays]
    \label{prop:smooth-flattening}
    Let $K,D\in\N_+$, $s>0$, and
    $\balpha\in\N_0^D$ with $r=\|\balpha\|_1<s$. Define
        $h=\ceil{s}-r, M=2K^D-K+1.$
    For every $f\in\Hcal_1^s([0,1]^D)$, there is a sequence
    $y_0,\ldots,y_{M-1}$ such that, for every
    $\bbeta\in\{0,\ldots,K-1\}^D$,
    \begin{align}
        \label{eq:smooth-flattening-interpolation}
        y_{\kappa(\bbeta)}
        &=\partial^{\balpha}f(\bbeta/K),
        &
        \kappa(\bbeta)
        &=\beta_1+2K\sum_{j=2}^D\beta_jK^{j-2},
    \end{align}
    and, for $\ell=0,\ldots,h$,
    \begin{align}
        \label{eq:smooth-flattening-differences}
        \max_{0\le i\le M-1-\ell}|\Delta^\ell y_i|
        \lesssim K^{-\min\{\ell,s-r\}}.
    \end{align}
    The implicit constant depends only on $s$.
\end{proposition}

Together, smooth flattening and smooth point fitting give a network module
that retrieves all Taylor coefficients at their respective target
accuracies.

\begin{proposition}[Simultaneous Taylor-coefficient retrieval]
    \label{prop:taylor-coefficient-retrieval}
    Let $s>0$, $D\in\N_+$, $m=\ceil{s}-1$, $N,K\ge2$, and
    $L\in\N_+$. Set
    $\zeta=2^{\ceil{s}+3}+2$ and $q=\floor{\log_\zeta N}$, and suppose that
    \begin{align*}
        N\ge\zeta,
        \qquad
        M:=2K^D-K+1\le3N^2L^2q.
    \end{align*}
    Then, for every $f\in\Hcal_1^s([0,1]^D)$, there is a network
    \begin{align*}
        \phi_{\rm coef}\in\wt\Fcal_{D,\binom{D+m}{m}}(N,L,N)
    \end{align*}
    such that, for all $\balpha\in\N_0^D$ with
    $\|\balpha\|_1\le m$ and all
    $\bbeta\in\{0,\ldots,K-1\}^D$,
    \begin{align}
        \label{eq:simultaneous-coefficient-retrieval}
        \left|
        [\phi_{\rm coef}(\bbeta/K)]_{\balpha}
        -\frac{\partial^{\balpha}f(\bbeta/K)}{\balpha!}
        \right|
        \lesssim K^{-(s-\|\balpha\|_1)}.
    \end{align}
    The constant implied by $\lesssim$ depends only on $s$, and
    the constant hidden in $\wt\Fcal$ is at most polynomial in $D$.
\end{proposition}

The proof, given in \cref{sec:coefficient-retrieval-proof}, applies
\cref{prop:point-fitting-diff} to the sequences constructed in
\cref{prop:smooth-flattening} and then parallelizes over the finitely many
multi-indices with $\|\balpha\|_1\le m$.

\subsection{Monomials, Localization, and Smoothing}
\label{sec:standard-components}

The remaining modules of \cref{fig:main-network} are standard. We review them here for completeness and to track parameter magnitudes.

\subsubsection{Squares, Multiplication, and Monomials}
\label{sec:monomial-networks}

To approximate general monomials, we begin with the square
function. The basic mechanism is the sawtooth construction of
\citet{telgarsky2016benefits}.
Let
\begin{align*}
    \tau(x)
     =2\sigma(x)-4\sigma(x-1/2)+2\sigma(x-1)
\end{align*}
and write $\tau_j=\tau\circ\cdots\circ\tau$ for its $j$-fold composition.
On $[0,1]$, the map $\tau$ is the tent map with peak one at $1/2$, and
$\tau_j$ has $2^j$ linear pieces. Thus, depth creates exponentially many
pieces from a fixed elementary module.
This fact leads to an efficient approximation of the square function. The function
\begin{align}
    \label{eq:square-interpolant}
    S_r(x)=x-\sum_{j=1}^r4^{-j}\tau_j(x),
    \qquad x\in[0,1]
\end{align}
is exactly the linear interpolant of $x^2$ on the dyadic grid $\{j2^{-r}:0\le j\le2^r\}$. Equation \eqref{eq:square-interpolant} is the square-function construction of
\citet{yarotsky2017error}, expressed directly in terms of the
sawtooth.

To move from the square function to products, the polarization identity
\begin{align*}
    uv
    &=2\left(\frac{u+v}{2}\right)^2-\frac{u^2}{2}-\frac{v^2}{2},
\end{align*}
allows for approximate multiplication using three copies of the square
network. Chaining these multiplication modules approximates
general monomials.
The following proposition makes this formal and is proven in \cref{sec:monomial-proof}. It is almost identical to
\citet[Proposition~4.1]{Lu2021}, but explicitly tracks parameter magnitudes, which we need for sharp entropy bounds.
\begin{proposition}[Monomial networks]
    \label{lem:poly-approx}
    Let $D,k\in\N_+$. For every $N\ge2$, $L\in\N_+$, and
    $\balpha\in\N_0^D$ with
    $\|\balpha\|_1\le k$, there is
    \begin{align*}
        \phi_{{\rm mon},\balpha}
         \in\Fcal_{D,1}\bigl((D+6)N,4kL,4N\bigr)
    \end{align*}
    satisfying
    \begin{align*}
        \sup_{\bx \in [0,1]^D} |\phi_{{\rm mon},\balpha}(\bx)| \le 1 \qquad \text{and} \qquad \sup_{\bx\in[0,1]^D}
        \bigl|\phi_{{\rm mon},\balpha}(\bx)-\bx^{\balpha}\bigr|
         \le 3kN^{-L}.
    \end{align*}
\end{proposition}

\subsubsection{Localization}
\label{sec:localization}

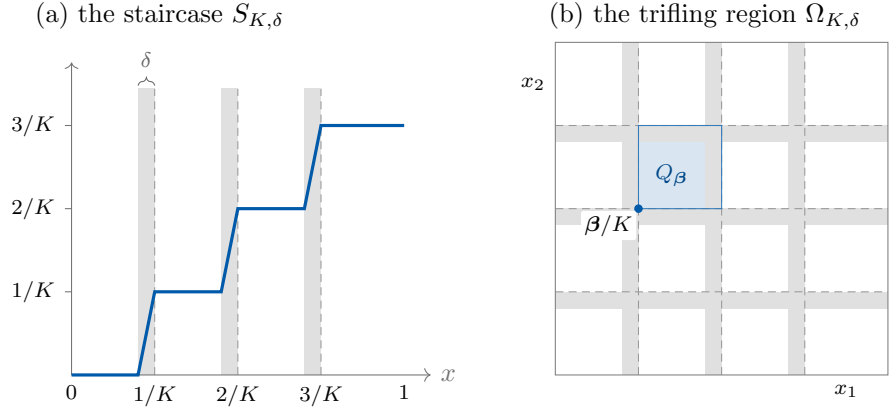
\begin{figure}[tp]
    \centering
    \begin{tikzpicture}[font=\small]
    \definecolor{figblue}{RGB}{0,90,170}
    \definecolor{figbrown}{RGB}{185,75,0}
    \definecolor{figgold}{RGB}{170,125,0}
    \definecolor{figpurple}{RGB}{110,60,140}
    \tikzset{lab/.style={font=\scriptsize}}
    \def\dl{.05}

    \begin{scope}[x=4.4cm, y=4.4cm]
        \node[anchor=west] at (-.14,1.08) {(a) the staircase $S_{K,\delta}$};
        \foreach \k in {1,2,3} {
            \fill[black!12] ({\k/4-\dl},0) rectangle ({\k/4},.86);
            \draw[black!40, densely dashed] ({\k/4},0) -- ({\k/4},.86);
        }
        \draw[decorate, decoration={brace, amplitude=3pt}, black!60]
            ({.25-\dl},.87) -- (.25,.87) node[midway, above=3pt, lab] {$\delta$};
        \draw[->, black!55] (0,0) -- (1.08,0) node[anchor=west] {$x$};
        \draw[->, black!55] (0,0) -- (0,.94);
        \foreach \k/\lb in {0/0, 1/{1/K}, 2/{2/K}, 3/{3/K}, 4/1} {
            \node[lab, anchor=north] at ({\k/4},0) {$\lb$};
        }
        \foreach \k in {1,2,3} {
            \draw[black!55] (0,{\k/4}) -- (-.02,{\k/4});
            \node[lab, anchor=east] at (-.02,{\k/4}) {$\k/K$};
        }
        \draw[figblue, very thick]
            (0,0) -- ({.25-\dl},0) -- (.25,.25) -- ({.5-\dl},.25)
            -- (.5,.5) -- ({.75-\dl},.5) -- (.75,.75) -- (1,.75);
    \end{scope}

    \begin{scope}[shift={(6.4cm,0)}, x=4.4cm, y=4.4cm]
        \node[anchor=west] at (-.05,1.08) {(b) the trifling region $\Omega_{K,\delta}$};
        \fill[figblue!15] (.25,.5) rectangle ({.5-\dl},{.75-\dl});
        \foreach \k in {1,2,3} {
            \fill[black!12] ({\k/4-\dl},0) rectangle ({\k/4},1);
            \fill[black!12] (0,{\k/4-\dl}) rectangle (1,{\k/4});
        }
        \foreach \k in {1,2,3} {
            \draw[black!40, densely dashed] ({\k/4},0) -- ({\k/4},1);
            \draw[black!40, densely dashed] (0,{\k/4}) -- (1,{\k/4});
        }
        \draw[black!50] (0,0) rectangle (1,1);
        \draw[figblue!80] (.25,.5) rectangle (.5,.75);
        \node[lab, figblue!80!black] at ({.375-\dl/2},{.625-\dl/2}) {$Q_{\bbeta}$};
        \fill[figblue] (.25,.5) circle (1.6pt);
        \node[lab, anchor=north east, inner sep=1.5pt, fill=white, rounded corners=1pt]
            at (.245,.495) {$\bbeta/K$};
        \node[lab, anchor=north] at (.875,0) {$x_1$};
        \node[lab, anchor=east] at (0,.875) {$x_2$};
    \end{scope}
\end{tikzpicture}
    \caption{Localization for $K = 4$. (a) The staircase $S_{K,\delta}$
        equals $k/K$ between the strips (gray) and is continuous on them. (b) The
        strips form $\Omega_{K,\delta}$ for $D = 2$, and
        $\Lambda = \bbeta/K$ on the blue part of the cell $Q_{\bbeta}$.}
    \label{fig:localization-strips}
\end{figure}

The Taylor network in \cref{fig:main-network} requires computing the lower corner $\bbeta/K$ of the cell $Q_{\bbeta}$ containing $\bx$. This map is piecewise constant, with discontinuities at the cell boundaries.
\Cref{prop:pw-linear} realizes it coordinatewise, but only away from thin strips before the cell boundaries, as the staircase $S_{K,\delta}$ in \cref{fig:localization-strips}(a).
The network in \cref{prop:pw-linear} supports $K=N^2L$ linear pieces with $O(N)$ width and $O(L)$ depth. For $D \ge 2$, this suffices for the grid $K \asymp (N^2L^2\log N)^{1/D}$ from \cref{sec:taylor-network-review}. For $D = 1$, it does not, but the repetitive staircase pattern can be implemented efficiently using two nested interpolants.
The next result follows the construction of \citet[Proposition~3.1]{shen2022optimal}, with explicit tracking of parameter magnitudes (see \cref{sec:general-network-tools}).

\begin{proposition}[Localization]
    \label{prop:localization}
    Fix $r\in\N_+$. Let $N\ge2$, $L\in\N_+$, and
    $K=K_1K_2$ with $K_1,K_2\in\{1,\dots,N^2L\}$.
    If $K^{-r}\le\delta\le1/(3K)$, there is
    $S_{K,\delta}\in\wt\Fcal_{1,1}(N,L,N)$ such that
    \begin{align*}
        S_{K,\delta}(x)=\frac{k}{K},
        \qquad
        x\in\Bigl[\frac{k}{K},\frac{k+1}{K}
            -\delta\ind_{k\le K-2}\Bigr],
        \quad k=0,\dots,K-1.
    \end{align*}
    The constants hidden by $\wt\Fcal$ depend only on $r$.
\end{proposition}

The proof is given in \cref{sec:localization-proof}.
The multivariate localization map is $\Lambda(\bx) = (S_{K,\delta}(x_1), \dots, S_{K,\delta}(x_D))$. The resulting module is exact on most of $[0, 1]^D$, but fails on the \emph{trifling region}
\begin{align*}
    \Omega_{K,\delta}
    &= \bigcup_{j=1}^D\Bigl\{\bx \in [0,1]^D\colon
            x_j \in \bigcup_{k=1}^{K-1}\Bigl(\frac kK-\delta, \frac kK\Bigr)\Bigr\},
\end{align*}
shown in gray in \cref{fig:localization-strips}(b).

\subsubsection{Median Smoothing}
\label{sec:median-smoothing}

The localization map feeds into the monomial network and coefficient fitting modules of the overall construction. The coefficient-retrieval module is highly irregular away from grid corners and can amplify localization errors.
For $L_p$-approximation bounds, this is easily fixed by clipping the output of the final network, and choosing $\delta$ small enough that the trifling region has negligible measure.
Sup-norm approximation requires an additional smoothing step.

\citet[Lemma~3.4 and Theorem~2.1]{Lu2021} and later improvements \citep{shen2022optimal, ou2024three, nagler2026optimal} all use a \emph{median smoothing} technique, which replaces the output of a network by the median of several copies.
The basic principle is simple: if two of three values are close to a target, then the median is also close to the target.
For $D > 1$, the above-mentioned works apply this idea iteratively, smoothing one coordinate at a time. This multiplies the width by the exponential factor $3^D$. Although this factor does not affect the asymptotic rate for fixed $D$, it makes the construction costly in high dimensions.
This can be circumvented by taking the median over $O(D)$ sub-networks.
\citet{siegel2023optimal} constructs the networks over grids with different
prime bases and \citet{hong2024bridging} use grids shifted along the diagonal. We instead evaluate a single network at points shifted along the diagonal.

\begin{proposition}[Median smoothing]
    \label{prop:median}
    Let $\eps > 0$, $K \in \N_+$, and $\delta \in (0, 1/((2D+1)K)]$. Assume
    $f \in C([0,1]^D)$ and suppose there is a network $\phi \in \Fcal_{D,1}(N, L, B)$, with  $N \ge 2D+2$ and $B \ge D+1$, satisfying
    \begin{align*}
         \sup_{\bx \in [0,1]^D \setminus \Omega_{K,\delta}}|\phi(\bx) - f(\bx)| \le \eps.
    \end{align*}
    Then there is
    \begin{align*}
        \wt \phi
        &\in \Fcal_{D,1}\bigl((2D+1)N,
            L + 2\ceil{\log_2(2D+1)} + 8, B\bigr)
    \end{align*}
    such that
    \begin{align*}
       \sup_{\bx \in [0,1]^D} |\wt \phi(\bx) - f(\bx)| \le \eps + \omega_f(D\delta).
    \end{align*}
\end{proposition}

\begin{figure}[tp]
    \centering
    \begin{tikzpicture}[font=\small]
    \definecolor{figblue}{RGB}{0,90,170}
    \definecolor{figbrown}{RGB}{185,75,0}
    \definecolor{figgold}{RGB}{170,125,0}
    \definecolor{figpurple}{RGB}{110,60,140}
    \tikzset{
        good/.style={circle, fill=figblue, inner sep=0pt, minimum size=4.4pt},
        bad/.style={circle, draw=figbrown, fill=white, thick, inner sep=0pt,
            minimum size=4.4pt},
        lab/.style={font=\scriptsize},
    }

    \begin{scope}[x=.7cm, y=.7cm]
        \node[anchor=west] at (1.6,8.1) {(a) the points $\bz_s$};
        \fill[black!12] (5,0) rectangle (6,7);
        \fill[black!12] (2,5) rectangle (9,6);
        \draw[black!40, densely dashed] (6,0) -- (6,7);
        \draw[black!40, densely dashed] (2,6) -- (9,6);
        \draw[black!50] (2,0) rectangle (9,7);
        \node[lab, black!60, anchor=south] at (5.5,7) {$\Omega_{K,\delta}$};
        \node[lab, black!60, anchor=west] at (9,5.5) {$\Omega_{K,\delta}$};
        \draw[decorate, decoration={brace, mirror, amplitude=3pt}, black!60]
            (5,-.1) -- (6,-.1) node[midway, below=3pt, lab] {$\delta$};
        \foreach \s in {-2,...,2} {
            \draw[black!55] ({5.5+\s},0) -- ({5.5+\s},.25);
            \draw[black!55] (2,{3.5+\s}) -- (2.25,{3.5+\s});
        }
        \draw[figbrown, thick] (5.5,0) -- (5.5,.25);
        \draw[figbrown, thick] (2,5.5) -- (2.25,5.5);
        \draw[black!35] (2,0) -- (9,7);
        \foreach \s/\pt/\st/\lb in {
            -2/{3.5,1.5}/good/{\bz_{-2}},
            -1/{4.5,2.5}/good/{\bz_{-1}},
             0/{5.5,3.5}/bad/{\bx},
             1/{6.5,4.5}/good/{\bz_1},
             2/{7.5,5.5}/bad/{\bz_2}} {
            \node[\st] at (\pt) {};
            \node[lab, anchor=north west, inner sep=1.5pt] at (\pt) {$\lb$};
        }
    \end{scope}

    \begin{scope}[shift={(9.6cm,0)}, x=.7cm, y=.7cm]
        \node[anchor=west] at (-2.6,8.1) {(b) the values $\phi(\bz_s)$};
        \fill[figblue!12] (-.45,1) rectangle (.45,4);
        \draw[->, black!55] (0,0) -- (0,7.2);
        \draw[black!60] (-.45,2.5) -- (.45,2.5);
        \node[lab, anchor=east] at (-.8,2.5) {$f(\bx)$};
        \draw[decorate, decoration={brace, amplitude=3pt}, black!60]
            (-.6,2.55) -- (-.6,4)
            node[midway, left=3pt, lab] {$\eps + \omega_f(2\delta)$};
        \foreach \v/\st/\lb in {
            1.5/good/{\phi(\bz_{-1})},
            2.1/good/{\phi(\bz_{-2})},
            5.3/bad/{\phi(\bx)},
            6.4/bad/{\phi(\bz_2)}} {
            \node[\st] at (0,\v) {};
            \node[lab, anchor=west] at (.3,\v) {$\lb$};
        }
        \node[good] at (0,3.5) {};
        \draw[black, thick] (0,3.5) circle (4.2pt);
        \node[lab, anchor=west] at (.3,3.5) {$\phi(\bz_1) = \wt\phi(\bx)$};
    \end{scope}
\end{tikzpicture}
    \caption{Median smoothing for $D = 2$; filled markers belong to points
        outside $\Omega_{K,\delta}$. (a) The ticks mark the coordinates of the
        five diagonal points, and each strip contains at most one of them.
        (b) The three filled values lie within $\eps + \omega_f(2\delta)$ of
        $f(\bx)$, and so does the median $\wt\phi(\bx)$ of all five values.}
    \label{fig:median-smoothing}
\end{figure}
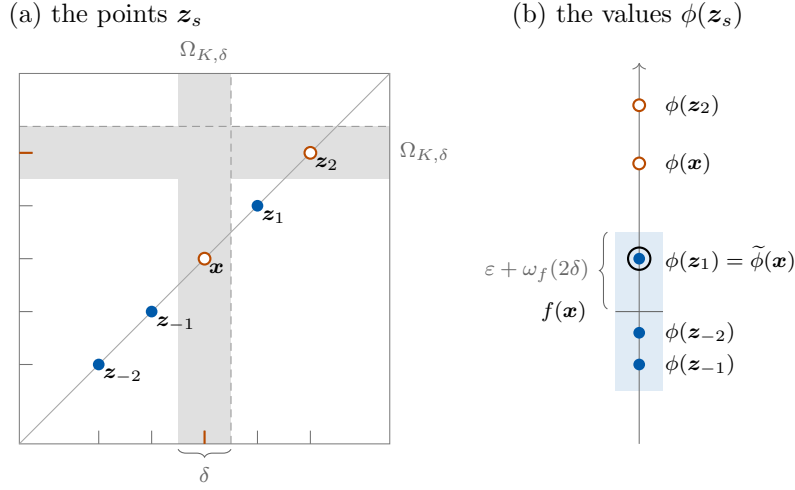

The idea is illustrated in \cref{fig:median-smoothing}.
We take the median of $\phi$ at $2D+1$ points shifted
along the diagonal,
    \begin{align}
        \label{eq:diagonal-shifts}
        \bz_i
        &= \clip_{[0,1]^D}(\bx + i\delta\mathbf 1_D),
        \qquad i = -D, \dots, D,
    \end{align}
where $\mathbf 1_D = (1, \dots, 1)^\top \in \R^D$.
In each coordinate, the unclipped values $x_j + i\delta$ are
$\delta$ apart and span $2D\delta \le 1/K - \delta$, while the strips are open intervals of length $\delta$ that are $1/K-\delta$ apart, so at most one of
the values lies in a strip.
As a consequence, at most $D$ of the $2D+1$ points lie in $\Omega_{K,\delta}$, and the median lies between two clean values. These are within $\eps$ of $f$ at points at distance at most $D\delta$ from $\bx$, so the median is within $\eps + \omega_f(D\delta)$ of $f(\bx)$.

The median of $2D + 1$ numbers can be computed exactly by ReLU networks. Sorting networks compute the median by sorting all entries, with width $O(D)$ and depth $O(\log^2 D)$ \citep{batcher1968sorting}, as used by \citet{siegel2023optimal}. \citet{hong2024bridging} find it through $D$ sequential correction steps. If an approximation suffices, \citet{dutta2026median} show that constant depth and linear width are enough. Our construction is more direct.
For $\by \in \R^{2D+1}$ with entries $y_{(1)} \ge \dots \ge y_{(2D+1)}$ in
decreasing order, the median $y_{(D+1)}$ is the difference of the sums of the
$D+1$ and the $D$ largest entries, and either sum can be computed without
sorting:
\begin{align}
    \label{eq:top-k-sum}
    \sum_{\ell=1}^{k} y_{(\ell)}
    &= \min_{1 \le i \le 2D+1}
        \Bigl(ky_i + \sum_{j=1}^{2D+1} \sigma(y_j - y_i)\Bigr), \quad  k = 1, \dots, 2D+1.
\end{align}
Indeed, for every $t \in \R$ and every set $T$ of $k$ indices,
$kt + \sum_j \sigma(y_j - t) \ge \sum_{j \in T} y_j$, with equality if $T$
indexes the $k$ largest entries and $t = y_{(k)}$. The terms $\sigma(y_j - y_i)$ form
one hidden layer, and a max network computes the minimum.
The full proof is given in \cref{sec:median-proof}.

\subsection{Completing the Proof}

We can now outline how these propositions yield
\cref{thm:main-approx}. With $\zeta$ as in \cref{prop:taylor-coefficient-retrieval}, choose
\begin{align*}
    q&=\floor{\log_\zeta N},
    &
    K&\asymp (N^2L^2q)^{1/D}.
\end{align*}
The grid then contains $K^D\asymp N^2L^2\log N$ cells, which is precisely the
number of coefficient vectors accommodated by
\cref{prop:taylor-coefficient-retrieval}. Combining its output with the
monomial networks of \cref{lem:poly-approx}, the error away from cell
boundaries is bounded by
\begin{align*}
    \underbrace{K^{-s}}_{\text{Taylor remainder}}
    +\sum_{\|\balpha\|_1\le m}
      \;  \underbrace{K^{-\|\balpha\|_1}}_{\text{monomial size}} \; \;
      \underbrace{K^{-(s-\|\balpha\|_1)}}_{\text{retrieval error}} \; \;
    +\underbrace{N^{-cL}}_{\text{polynomial evaluation}}
    \lesssim K^{-s},
\end{align*}
where the fixed constant $c=c(s,D)$ is chosen sufficiently large.

The localization network of \cref{prop:localization} returns the grid point $\bbeta/K$ but makes errors near cell boundaries. Median smoothing (\cref{prop:median}) removes these exceptional strips and yields a uniform error
bound on the whole cube. The resulting network belongs to
$\wt\Fcal_{D,1}(N,L,N)$ and satisfies the required bound
\begin{align*}
    \|f-\phi\|_\infty
    &\lesssim K^{-s}
    \lesssim (N^2L^2\log N)^{-s/D}.
\end{align*}
Finally, we use magnitude-to-width transfer $\wt\Fcal_{D,1}(N, L, N) \subseteq \wt \Fcal_{D,1}(N, L, 1)$ to bound the parameters by 1 while preserving the width and depth orders (\cref{lem:magnitude-width}).
This completes the proof of \cref{thm:main-approx}.

\section{Discussion}
\label{sec:discussion}

This article establishes $\epsilon_{N,L} = (N^2 L^2 \log N)^{-s/D}$ as the exact approximation rate of bounded-weight ReLU networks for H\"older balls $\Hcal_1^s([0,1]^D)$, for every $s>0$ and $D\in\N_+$. The explicit construction allows for flexible allocation of capacity between the width and depth of the network. The key is that Taylor coefficients of smooth functions provide additional compressibility due to bounds on higher-order differences. This allows them to be stored at $O(1)$ cost per coefficient and decoded efficiently using a recursive application of the discrete Taylor--Newton formula. As a consequence, empirical risk minimizers over suitably sized network classes attain the minimax rate $n^{-2s/(2s + D)}$ for nonparametric regression under the square loss without redundant polylogarithmic factors.

This work completes a long line of research seeking the optimal approximation rates of ReLU networks for H\"older functions. Several interesting questions remain open, however.
First, the lower bound in \cref{thm:lower-bound} concerns networks whose parameters grow at most polynomially in the width.
The upper bound on the VC dimension of ReLU networks with potentially unbounded weights \citep{bartlett2019nearly} combined with a bump function argument \citep{shen2022optimal,siegel2023optimal}
gives
\begin{align*}
    \sup_{f \in \Hcal_1^s([0, 1]^D)}
    \inf_{\phi \in \Fcal_{D,1}(N,L,\infty)}\|f - \phi\|_\infty
    &\gtrsim \min\{N^2L^2\log(NL),\,N^3L^2\}^{-s/D}.
\end{align*}
This matches $\epsilon_{N,L}$ whenever $N$ is bounded or $L\le \mathrm{poly}(N)$. In this regime, the rate of \cref{thm:main-approx} is optimal even among networks with arbitrary parameters, and \cref{thm:main-approx} attains it with bounded parameters. Whether large parameters can improve the rate when $\log L\gg\log N\gg 1$ remains open. This is exactly the regime where the best-known upper and lower bounds on the VC dimension of ReLU networks diverge \citep{bartlett2019nearly}.
The lower bound is of order $N^2L^2\log N$, and tightening the upper bound to match it would show that $\epsilon_{N,L}$ is optimal also for unbounded weights.

Second, we have considered H\"older-smooth targets without further structure. To explain the success of deep learning in practice, it is important to understand how additional structure in the target function can be exploited. \citet{schmidt2020nonparametric}, \citet{kohler2021rate}, and \citet{danhofer2025position} argue that networks are particularly well-suited for compositional functions, which can be expressed as a composition of low-dimensional functions.
Composing one network from \cref{thm:main-approx} for each low-dimensional component and controlling the accumulated error with the error propagation principle \citep[Lemma~4.3]{nagler2026optimal} should remove the redundant logarithmic factors in this setting as well. Another route is to consider functions defined on low-dimensional sets, such as manifolds, where optimal approximation rates depend on the intrinsic rather than the ambient dimension. \citet{nagler2026optimal} show that logarithmic factors are unavoidable for approximation uniformly over all sets of Minkowski dimension $d < D$, but it remains an open question for which classes of low-dimensional sets the higher-order point-fitting construction can be adapted.

More broadly, our construction shows that structure in the target function can pay off multiple times in approximation efficiency. Smoothness helps both to reduce the Taylor remainder and to compress the local Taylor coefficients. The underlying principle, that the information stored in a network inherits the regularity of the target and can be compressed accordingly, may prove useful beyond the setting considered here.

\section*{AI Use Statement}

The main results and proof strategies were developed, essentially in the form presented here, in summer
2025 using pen, paper, and occasional discussions with the ChatGPT web interface.
The final write-up started around May 2026 and was assisted by various models through Codex and Claude Code. In particular, these models were used for checking, simplifying, and restructuring the initial proofs and auxiliary results, generating expository figures, and polishing the text. The author takes full responsibility for the final content of the paper.

\ifarxiv
\section*{Acknowledgments}

The author thanks Sophie Langer for helpful comments and edits on an earlier version of this manuscript.
\fi

\newpage

\appendix

\section{Network Calculus}
\label{sec:general-network-tools}

This appendix collects elementary network constructions used
throughout the proofs. Many of them are well-established, but we include them for completeness.
The first lemma allows us to combine networks.

\begin{lemma}[Network calculus]
    \label{lem:network-calculus}
    Let $B_1, B_2 \ge 1$, $B=\max\{B_1,B_2\}$, and
    \begin{align*}
        \phi_1\in\Fcal_{d_1,d_1'}(N_1,L_1,B_1), \qquad \phi_2\in\Fcal_{d_2,d_2'}(N_2,L_2,B_2).
    \end{align*}
    \begin{enumerate}[(i)]
        \item \emph{Parallelization.} If $L_1=L_2=L$ and
              $P_i\colon\R^d\to\R^{d_i}$ are coordinate projections, then
              \begin{align*}
                  \bx\mapsto\bigl(\phi_1(P_1\bx),\phi_2(P_2\bx)\bigr)
                  \in\Fcal_{d,d_1'+d_2'}(N_1+N_2,L,B).
              \end{align*}
        \item \emph{Composition.} If $d_2=d_1'$, then
              \begin{align*}
                  \phi_2\circ\phi_1
                  \in\Fcal_{d_1,d_2'}\bigl(\max\{N_1,N_2,2d_1'\},L_1+L_2+1,B\bigr).
              \end{align*}
        \item \emph{Padding.} For every $k\in\N_+$,
              \begin{align*}
                  \phi_1\in\Fcal_{d_1,d_1'}\bigl(\max\{N_1,2d_1'\},L_1+k,B_1\bigr).
              \end{align*}
        \item \emph{Affine combination.} If $L_1=L_2=L$, $d_1=d_2$, and
              $d_1'=d_2'$, then, for $a_1,a_2,c\in\R$,
              \begin{align*}
                  a_1\phi_1+a_2\phi_2+c
                  \in\Fcal_{d_1,d_1'}\bigl(\max\{N_1+N_2,4d_1'\},L+1,
                  \max\{B,|a_1|,|a_2|,|c|\}\bigr).
              \end{align*}
    \end{enumerate}
\end{lemma}
\begin{proof}
    We prove the four constructions in order.
    \begin{itemize}
        \item[(i)] Write $\phi_i$ in the form \eqref{eq:FNN} with weight
        matrices $W_l^{(i)}$ and bias vectors $\bv_l^{(i)}$. The network with
        \begin{align*}
            W_0
            &=\begin{pmatrix}
                W_0^{(1)}P_1\\
                W_0^{(2)}P_2
            \end{pmatrix},
            &
            W_l
            &=\begin{pmatrix}
                W_l^{(1)} & 0\\
                0 & W_l^{(2)}
            \end{pmatrix},
            &
            \bv_l
            &=\begin{pmatrix}
                \bv_l^{(1)}\\
                \bv_l^{(2)}
            \end{pmatrix},
        \end{align*}
        for $l=1,\dots,L$ and $l=0,\dots,L$, respectively, computes
        $(\phi_1(P_1\bx),\phi_2(P_2\bx))$. Its hidden layers have width at
        most $N_1+N_2$. All
        parameters are bounded by $B$.
        \item[(ii)] Write $\phi_1=A_1\circ h_1$, where $h_1$ consists of the
        hidden layers of $\phi_1$ and $A_1(\bz)=W\bz+\bv$ is its output layer.
        Write $\phi_2=g_2\circ A_2$, where $A_2(\bx)=V\bx+\bb_0$ is its first
        affine map and $g_2$ is the rest. Without loss of generality, we take $g_2=\mathrm{id}$ if $L_2=0$.
        Composing $A_2\circ A_1$ directly would produce the weight matrix $VW$, which amplifies parameter magnitudes.
        Instead, the identity
        $u=\sigma(u)-\sigma(-u)$ gives
        \begin{align*}
            A_2(A_1\bz)
            &=(V,-V)\begin{pmatrix}
                \sigma(W\bz+\bv)\\
                \sigma(-W\bz-\bv)
            \end{pmatrix}+\bb_0.
        \end{align*}
        Hence $\phi_2\circ\phi_1$ is the network $h_1$, followed by one hidden
        layer of width $2d_1'$ with weights $\pm W$ and biases $\pm\bv$,
        followed by $g_2$ with first weight matrix $(V,-V)$. All its
        parameters are bounded by $B$.
        \item[(iii)] We apply (ii) $k$ times with $\phi_2$ the identity map,
        a network of depth $0$ with parameter bound $1$.
        \item[(iv)] We apply (ii) to the network $(\phi_1,\phi_2)$ from (i)
        and the affine map $(\by_1,\by_2)\mapsto a_1\by_1+a_2\by_2+c$. \qedhere
    \end{itemize}
\end{proof}

For a fixed finite collection of networks in $\wt\Fcal(N,L,B)$,
fixed dimensions, and fixed coefficients, \cref{lem:network-calculus}
therefore shows that parallelization, composition, padding, and affine
combination remain in $\wt\Fcal(N,L,B)$. The hidden constants may depend on
the number of networks, the dimensions, and the coefficients, but not on
$N,L,B$. These constants are at most linear in the number of networks
and in the dimensions.

\begin{lemma}[Magnitude-to-width transfer]\label{lem:magnitude-width}
    Let $d, d', W \in \N_+$, $L \in \N$, and $0<B'\le B$. Then
    \begin{align*}
        \Fcal_{d,d'}(W, L, B)
        &\subseteq \Fcal_{d,d'}\left(\max\left\{W, 2d'\ceil{\frac{B}{(B')^2}}\right\}, 2L+1, B'\right).
    \end{align*}
\end{lemma}

\begin{proof}
    Write $f = A_L \circ \sigma \circ A_{L-1} \circ \cdots \circ \sigma \circ A_0$
    with affine maps $A_i(\bz) = W_i\bz+\bv_i$, and set $r=B/B'\ge1$.
    Let $\widetilde f$ be the network with the same architecture, but using affine maps
    \begin{align*}
        \widetilde A_i(\bz) &= r^{-1}W_i\bz+r^{-(i+1)}\bv_i,
        \qquad i = 0,\dots,L,
    \end{align*}
    whose parameters are bounded by $B'$. Denote the hidden-layer outputs
    of $f$ and $\widetilde f$ by $\bu_i$ and $\widetilde\bu_i$, with
    $\bu_0=\widetilde\bu_0=\bx$. By homogeneity of $\sigma$,
    $\bu_i=r^i\widetilde\bu_i$ for $i=0,\dots,L$, so
    $f=r^{L+1}\widetilde f$.

    It remains to multiply the output of $\widetilde f$ by $r^{L+1}$ with
    parameters bounded by $B'$. Consider first $d' = 1$, and let
    $m=\ceil{r/B'}$, $c=B'\bigl(r/(mB')\bigr)^{L+1}$, and
    $\mathbf 1_m=(1,\dots,1)^\top\in\R^m$. We
    replace the output layer of $\widetilde f$ by $L+1$ hidden layers of
    width $2m$ and a new output layer $g$, defined by
    \begin{align*}
        \bz_1 &= \sigma\left(\begin{pmatrix}\mathbf 1_m\\ -\mathbf 1_m\end{pmatrix}
            \widetilde A_L(\widetilde\bu_L)\right),
        & J &= \begin{pmatrix}
            \mathbf 1_m\mathbf 1_m^\top & 0\\
            0 & \mathbf 1_m\mathbf 1_m^\top
        \end{pmatrix},\\
        \bz_{k+1} &= \sigma(B'J\bz_k), \quad k = 1,\dots,L,
        & g &= c\begin{pmatrix}\mathbf 1_m^\top & -\mathbf 1_m^\top\end{pmatrix}\bz_{L+1}.
    \end{align*}
    The first new hidden layer copies $\widetilde A_L$ with either sign,
    and the remaining weights are $0$, $B'$, or $\pm c$. Since $mB'\ge r$,
    we have $c\le B'$, so all parameters are bounded by $B'$. With
    $t=\widetilde f(\bx)$, induction on $k$ gives
    \begin{align*}
        \bz_k &= \begin{pmatrix}\sigma((mB')^{k-1}t)\,\mathbf 1_m\\
            \sigma(-(mB')^{k-1}t)\,\mathbf 1_m\end{pmatrix},
        \qquad k = 1,\dots,L+1,
    \end{align*}
    because $\mathbf 1_m^\top\mathbf 1_m=m$ and $\sigma$ is positively homogeneous.
    Since $\sigma(a)-\sigma(-a)=a$, this gives
    \begin{align*}
        g=cm(mB')^Lt=r^{L+1}t=f(\bx).
    \end{align*}
    The resulting network has depth $2L+1$ and width
    $\max\{W,2m\}$. For $d'>1$, we apply the amplification layers to all
    output coordinates in parallel, which gives width $\max\{W,2d'm\}$.
\end{proof}

\citet[Proposition~H.4]{ou2024three} give a magnitude-to-depth version for
scalar input and output: the width stays fixed, and the number of
additional layers grows with $\log B/\log W$. \cref{lem:magnitude-width}
instead transfers the magnitude to the width, which is more convenient for our
purposes.

\begin{lemma}[ReLU realization of CPwL functions]
    \label{lem:linear-spline}
    Let $0<t_1<\cdots<t_M<1$, let $d'\in\N_+$, and let
    $h\colon[0,1]\to\R^{d'}$ be continuous and affine on the
    intervals determined by these knots. If $\|h(0)\|_\infty\le A$ and every
    slope of every coordinate of $h$ has magnitude at most $A$,
    for some $A\ge1$, then there is
    \begin{align*}
        \phi\in\Fcal_{1,d'}(M+1,1,2A)
    \end{align*}
    such that $\phi=h$ on $[0,1]$.
\end{lemma}

\begin{proof}
    Write $\ba_j\in\R^{d'}$ for the slope of $h$ on the interval
    immediately to the right of $t_j$, with $t_0=0$ and $\ba_0$ the initial
    slope. Then
    \begin{align*}
        h(x)
        =h(0)+\ba_0\sigma(x)
          +\sum_{j=1}^M(\ba_j-\ba_{j-1})\sigma(x-t_j),
        \qquad x\in[0,1].
    \end{align*}
    This is a one-hidden-layer ReLU network of width $M+1$. Its input weights
    equal one, its biases belong to $[-1,0]$, and its output weights have
    magnitude at most $2A$.
\end{proof}

\begin{lemma}[Maximum and maximum norm networks]\label{lem:maximum-networks}
    For every $d \in \N_+$, there are networks
    \begin{align*}
        \phi_{\max}, \; \phi_{\infty} \,
         & \in \Fcal_{d, 1}\bigl(2d, 2\ceil{\log_2 d} + 2, 1\bigr),
    \end{align*}
    such that, for every $\bx \in \R^d$,
    \begin{align*}
        \phi_{\max}(\bx)
         & = \max_{1 \le j \le d} x_j,
         & \phi_{\infty}(\bx)
         & = \|\bx\|_\infty.
    \end{align*}
\end{lemma}

\begin{proof}
    The identities
    \begin{align*}
        \max\{x, y\}
         & = \sigma(x - y) + \sigma(y) - \sigma(-y),
         & |x|
         & = \sigma(x) + \sigma(-x)
    \end{align*}
    are realized by one-hidden-layer networks of widths three and two,
    respectively. Applying the identity for $\max\{x,y\}$ over a binary tree of depth
    $\lceil \log_2 d \rceil$ gives $\phi_{\max}$. Between successive levels, we
    insert the one-hidden-layer identity network
    $z = \sigma(z) - \sigma(-z)$. The tree has width at most $2d$ and uses at most
    two hidden layers per level.
    Prepending the coordinate-wise absolute-value layer gives $\phi_{\infty}$.
\end{proof}

The maximum network also lets us compute the median of an odd number of inputs
without sorting them.

\begin{lemma}[Median network]\label{lem:median-network}
    For every $m \in \N_+$, there is a network
    \begin{align*}
        \phi_{\rm med}
         & \in \Fcal_{2m+1, 1}\bigl((2m+1)(2m+2), 2\ceil{\log_2(2m+1)} + 5, m+1\bigr)
    \end{align*}
    such that $\phi_{\rm med}(\by)$ is the median of $y_1, \dots, y_{2m+1}$
    for every $\by \in \R^{2m+1}$.
\end{lemma}

\begin{proof}
    Let $y_{(1)} \ge \dots \ge y_{(2m+1)}$ be the entries of $\by$ in
    decreasing order. The median $y_{(m+1)}$ is the difference of the sums of
    the $m+1$ and the $m$ largest entries. For $k = 1, \dots, 2m+1$, these
    sums satisfy, by \eqref{eq:top-k-sum} with $m$ in place of $D$,
    \begin{align*}
        \sum_{l=1}^{k} y_{(l)}
         & = \min_{1 \le i \le 2m+1} c_i^{(k)},
        \qquad
        c_i^{(k)}
        = ky_i + \sum_{j=1}^{2m+1} \sigma(y_j - y_i).
    \end{align*}
     With
    $\bc^{(k)} = (c_1^{(k)}, \dots, c_{2m+1}^{(k)})$, this gives
    \begin{align*}
        y_{(m+1)}
         & = \min_{1 \le i \le 2m+1} c_i^{(m+1)}
        - \min_{1 \le i \le 2m+1} c_i^{(m)}
        = \phi_{\max}\bigl(-\bc^{(m)}\bigr)
        - \phi_{\max}\bigl(-\bc^{(m+1)}\bigr),
    \end{align*}
    where $\phi_{\max} \in \Fcal_{2m+1,1}(2(2m+1), 2\ceil{\log_2(2m+1)} + 2, 1)$
    is the network from \cref{lem:maximum-networks}.

    We build $\phi_{\rm med}$ from two parts and connect them and the output
    by identity layers $\bu \mapsto (\sigma(\bu), \sigma(-\bu))$ as in
    \cref{lem:network-calculus}. First, one hidden layer with the units
    $\sigma(y_j - y_i)$, $i \ne j$, and $\sigma(\pm y_i)$ computes
    $\bc^{(m)}$ and $\bc^{(m+1)}$ through
    \begin{align*}
        c_i^{(k)}
         & = k\bigl(\sigma(y_i) - \sigma(-y_i)\bigr)
        + \sum_{j \ne i} \sigma(y_j - y_i).
    \end{align*}
    It has width $(2m+1)(2m+2)$ and parameters bounded by $m+1$. Second, two
    parallel copies of $\phi_{\max}$ evaluate the two minima, and the output
    weights $(1, -1)$ form their difference.

    The two parts have depths $1$ and $2\ceil{\log_2(2m+1)} + 2$, and two
    identity layers connect them and the output. Since $m \ge 1$, the
    identity layers and the copies of $\phi_{\max}$ have width at most
    $4(2m+1) \le (2m+1)(2m+2)$, and all parameters are bounded by $m+1$.
    This gives the stated class.
\end{proof}

\begin{lemma}[Scalar amplification]\label{lem:amplification}
    Let $N\ge2$ and $A\ge1$. The map $x\mapsto Ax$ can be realized exactly
    by a ReLU network $\phi_{\mathrm{amp}} \in \Fcal_{1, 1}(2, 2+\ceil{\log_N A}, N)$.
\end{lemma}

\begin{proof}
    Write $A=a_1\cdots a_t$ with $t=\ceil{\log_N A}$ and
    $a_1,\dots,a_t\in[1,N]$. At every layer carry the positive and negative
    parts of the current value. Multiplying both parts by $a_j$ and subtracting
    them in the final affine layer realizes $Ax$.
\end{proof}

\begin{lemma}[Clipping network]\label{lem:clipping}
    Let $c \le d$. The map $\clip_{[c,d]}$ can be realized exactly
    by a ReLU network $\phi_{\rm clip} \in \Fcal_{1,1}(2,1,\max\{1,|c|,|d|\})$.
\end{lemma}

\begin{proof}
    Use the identity $\clip_{[c,d]}(t)=c+\sigma(t-c)-\sigma(t-d)$.
\end{proof}

\begin{lemma}[Discrete selection networks]
    \label{lem:indicator}
    \label{lem:ind-product}
    \label{lem:wsum}
    \envbreak
    \begin{enumerate}[(i)]
        \item There is $\phi \in \Fcal_{2,1}(2,3,1)$ such that
              \begin{align*}
                  \phi(i,j)
                   & = \ind_{i=j},
                  \qquad i,j \in \Z.
              \end{align*}
        \item For $C \in [1,\infty)$, there is
              $\phi \in \Fcal_{2,1}(4,2,C)$ such that
              \begin{align*}
                  \phi(y,\delta)
                   & = \delta y,
                  \qquad y \in [-C,C],\quad \delta \in \{0,1\},
              \end{align*}
              and $\phi(y,0)=0$ for every $y\in\R$.
        \item Let $n,m \in \N_+$, $C \in [1,\infty)$,
              $W \in \R^{n\times m}$ be fixed, and set
              $w^*=\max_{i,j}|w_{i,j}|$. There is
              \begin{align*}
                  \phi
                   & \in \Fcal_{m+1,1}
                  \left(4\max\{n,m\},3,\max\{n,mCw^*\}\right)
              \end{align*}
              such that
              \begin{align*}
                  \phi(i,\bx)
                   & = \be_i^\top W\bx
                  = \sum_{j=1}^m w_{i,j}x_j,
                  \qquad i=1,\dots,n,
              \end{align*}
              for every $\bx\in\R^m$ with $|x_j|\le C$
              whenever $w_{i,j}\ne0$.
    \end{enumerate}
\end{lemma}

\begin{proof}
    For part~(i), use
    \begin{align*}
        \phi(i,j)
         & = \sigma\left(1-\sigma(i-j)-\sigma(j-i)\right)
        = \sigma\left(1-|i-j|\right).
    \end{align*}
    For part~(ii), use
    \begin{align*}
        \phi(y,\delta)
         & = \sigma(y)-\sigma(y-C\delta)-\sigma(-y)+\sigma(-y-C\delta).
    \end{align*}
    For $\delta=0$, the four terms cancel for every $y\in\R$.

    For part~(iii), set $C'=\max\{1,mCw^*\}$, so that
    $|(W\bx)_i|\le C'$ for all $\bx$ as in the statement. Denote the networks from
    parts~(i) and~(ii) by
    $\phi_{\ind}$ and $\phi_{\rm gate}$, respectively. Then
    \begin{align*}
        \phi(i,\bx)
         & =\sum_{j=1}^n
        \phi_{\rm gate}\bigl((W\bx)_j,\phi_{\ind}(i,j)\bigr)
        =\sum_{j=1}^n(W\bx)_j\phi_{\ind}(i,j)
        =(W\bx)_i.
    \end{align*}
    The first hidden layer computes $\sigma(\pm(i-j))$ for
    $j=1,\dots,n$ and $\sigma(\pm x_k)$ for $k=1,\dots,m$. The second
    computes $\phi_{\ind}(i,j)$ and passes $\sigma(\pm x_k)$ on. The third
    computes the four units of $\phi_{\rm gate}\bigl((W\bx)_j,\phi_{\ind}(i,j)\bigr)$
    for every $j$, using
    $(W\bx)_j=\sum_k w_{j,k}\bigl(\sigma(x_k)-\sigma(-x_k)\bigr)$, and an
    output layer sums the results. The widths are at most $2n+2m$, $n+2m$, and
    $4n$, and the parameters are bounded by
    $\max\{n,C',w^*\}\le\max\{n,mCw^*\}$.
\end{proof}

\section{Proofs of the Propositions in Section~\ref{sec:proof-sketch-compact}}
\label{sec:proposition-proofs}

We prove the propositions of
\cref{sec:proof-sketch-compact} in the order in which they appear.

\subsection{Proof of Proposition~\ref{prop:pw-linear}}
\label{sec:piecewise-affine-proof}

The proof follows the outline in \cref{sec:piecewise-affine-interpolation}
and uses its slope-jump notation $[f']_t$. The following lemma states the two-layer construction with explicit
parameter bounds.

\begin{lemma}[Two-layer interpolation]
    \label{lem:pw-linear-two-layer}
    Let $N \ge 2$ and $Y > 0$. Suppose $0 \le x_1 < \cdots < x_{N^2} \le 1$
    and $y_1, \ldots, y_{N^2} \in [-Y, Y]$, and let
    $\delta = \min_{1 \le i < N^2}(x_{i+1}-x_i)$. Then the CPwL
    interpolant $f$ through $(x_i, y_i)$, with constant
    extension outside $[x_1, x_{N^2}]$, satisfies
    \begin{align*}
        f &\in \Fcal_{1,1}\bigl(3N, 2, \max\{1, 16NY\delta^{-2}\}\bigr).
    \end{align*}
\end{lemma}

\begin{proof}
    We group the knots into $N$ groups of size $N$.
    Write
    \begin{align*}
        x_{i,j} &= x_{(i-1)N+j}, \quad i,j = 1,\dots,N,
        & x_{i,N+1} &= x_{i,N}+\delta/2, \quad i = 1,\dots,N-1.
    \end{align*}
    We call $x_{i,1}$, $x_{i,N}$, and $x_{i,N+1}$ boundary knots and the
    remaining $x_{i,j}$ interior knots. The set $\Bcal \subset [0,1]$ of
    boundary knots has at most $3N$ elements, all separated by at least
    $\delta/2$. We split the slope jumps of $f$ by sign as
    $a_{i,j} = [f']_{x_{i,j}} = a_{i,j}^+-a_{i,j}^-$, where
    $a_{i,j}^+ = \max\{a_{i,j},0\}$ and $a_{i,j}^- = \max\{-a_{i,j},0\}$.

    For $j = 2, \ldots, N-1$, let $g_j^+$ be the CPwL
    function with knots in $\Bcal$ that satisfies
    \begin{align*}
        g_j^+(x) &= a_{i,j}^+(x-x_{i,j}),
        & x &\in [x_{i,1}, x_{i,N}],\ i = 1,\dots,N,\\
        g_j^+(x_{i,N+1}) &= 0,
        & i &= 1,\dots,N-1,
    \end{align*}
    and is constant outside $[x_1, x_{N^2}]$. Define $g_j^-$ in the same
    way with $a_{i,j}^-$ in place of $a_{i,j}^+$. On $[x_{i,1}, x_{i,N+1}]$,
    $g_j^+$ either vanishes or changes sign only at $x_{i,j}$, from
    negative to positive with slope $a_{i,j}^+$. On $[x_{i,N+1}, x_{i+1,1}]$
    and to the left of $x_1$, it is nonpositive; see
    \cref{fig:piecewise-affine}(a). The
    ReLU therefore creates the slope jumps
    \begin{align*}
        [\sigma(g_{j'}^+)']_{x_{i,j}} &= a_{i,j}^+\ind_{j=j'},
        & [\sigma(g_{j'}^-)']_{x_{i,j}} &= a_{i,j}^-\ind_{j=j'},
        \qquad 2 \le j, j' \le N-1,
    \end{align*}
    and all other knots of $\sigma(g_j^+)$ and $\sigma(g_j^-)$ lie in $\Bcal$. The
    residual $h$ in
    \begin{align}
        \label{eq:piecewise-affine-decomposition}
        f(x)
        &= h(x)+\sum_{j=2}^{N-1}
            \Bigl(\sigma\bigl(g_j^+(x)\bigr)
            -\sigma\bigl(g_j^-(x)\bigr)\Bigr)
    \end{align}
    therefore has slope jump $a_{i,j}-a_{i,j}^++a_{i,j}^- = 0$ at every
    interior knot, so its active knots also lie in $\Bcal$.
    We conclude that each function $\varphi \in \{h, g_j^+, g_j^-\}$ is constant outside
    $[x_1, x_{N^2}]$ and has knots only in $\Bcal$, so it has representation
    \begin{align*}
        \varphi(x)
        &= \varphi(x_1)+\sum_{z \in \Bcal}[\varphi']_z\,\sigma(x-z).
    \end{align*}

    With $h = \sigma(h)-\sigma(-h)$, the identity
    \eqref{eq:piecewise-affine-decomposition} is therefore a network with two hidden layers. Its first hidden layer computes
    the hinges $\sigma(x-z)$, $z \in \Bcal$, with weights one and biases in
    $[-1,0]$. Its second hidden layer computes $\sigma(h)$, $\sigma(-h)$,
    and $\sigma(g_j^+)$ and $\sigma(g_j^-)$, with the slope jumps $[\varphi']_z$ as
    weights and the values $\varphi(x_1)$ as biases. Its output weights are
    $\pm1$. We have $h(x_1) = y_1$ and
    \begin{align*}
        |a_{i,j}|,\, |g_j^+(x_i)|,\, |g_j^-(x_i)| &\le 4Y/\delta,\\
        |(g_j^+)'|,\, |(g_j^-)'| &\le 8Y/\delta^2,\\
        |h'| &\le 2Y/\delta+8(N-2)Y/\delta^2 \le 8NY/\delta^2,
    \end{align*}
    where we used $|f'| \le 2Y/\delta$, that the knots in $\Bcal$ are at
    least $\delta/2$ apart, and that at most $N-2$ of the functions $\sigma(g_j^+)$ and $\sigma(g_j^-)$
    are nonzero at any point. Since $|[\varphi']_z| \le 2\sup|\varphi'|$,
    all parameters are bounded by $\max\{1, 16NY\delta^{-2}\}$.
\end{proof}

\begin{proof}[Proof of \cref{prop:pw-linear}]
    We precompose the network with $x \mapsto x/\max\{1,x_M\}$. This map
    is absorbed into the first affine map, does not
    increase the parameter bound, and allows us to assume $x_M \le 1$ and a
    separation of at least $(NL)^{-2p}$.

    A block of $N^2$ knots covers $N^2-1$ knot intervals, and
    $M-1 \le 2(N^2-1)L$. We can therefore choose knots
    $x_1 = t_0 < t_1 < \cdots < t_K = x_M$ with $K \le 2L$ such that each
    interval $[t_{\ell-1}, t_\ell]$ contains at most $N^2$ knots.
    Define the increments
    \begin{align*}
        f_\ell(x)
        &= f\bigl(\clip_{[t_{\ell-1},t_\ell]}(x)\bigr)-f(t_{\ell-1}),
        \qquad \ell = 1,\dots,K.
    \end{align*}
    The function $f_\ell$ is the interpolant through the points
    $(x_i, y_i-f(t_{\ell-1}))$ with $x_i \in [t_{\ell-1}, t_\ell]$, extended
    constantly outside $[t_{\ell-1}, t_\ell]$. These increments telescope:
    \begin{align}
        \label{eq:pw-telescoping}
        f(x) &= y_1+\sum_{\ell=1}^K f_\ell(x).
    \end{align}
    Padding a shorter block to $N^2$ knots
    by equally subdividing one knot interval leaves $f_\ell$ unchanged and
    the separation at least $\delta = (NL)^{-2p-2}$. Set
    $A = 32(NL)^{5p+5}$. The values of $f_\ell/A$ are bounded by
    $Y = 2(NL)^p/A$, and $16NY\delta^{-2} = 32N(NL)^{5p+4}/A \le 1$. Hence
    \cref{lem:pw-linear-two-layer} gives
    \begin{align*}
        f_\ell/A &\in \Fcal_{1,1}(3N,2,1),
        \qquad \ell = 1,\dots,K.
    \end{align*}

    Since all knots are nonnegative, $f_\ell(x) = f_\ell(\sigma(x))$. By
    \eqref{eq:pw-telescoping}, the running sums
    \begin{align*}
        s_0 &= \frac{y_1}{A},
        & s_\ell(x) &= s_{\ell-1}(x)+\frac{f_\ell(\sigma(x))}{A},
        \qquad \ell = 1,\dots,K,
    \end{align*}
    satisfy $s_K = f/A$. We stack the block networks sequentially. Each hidden layer carries
    $\sigma(x)$ with one unit and $s_{\ell-1}$ through $\sigma(s_{\ell-1})$
    and $\sigma(-s_{\ell-1})$. This gives
    \begin{align*}
        f/A &\in \Fcal_{1,1}(4N,2K,1).
    \end{align*}
    By \cref{lem:amplification}, multiplying by $A$ after one identity
    layer adds at most $3+\ceil{\log_N A}$ layers with parameters bounded
    by $N$. The total depth satisfies
    \begin{align*}
        2K+3+\ceil{\log_N A}
        &\le 4L+9+(5p+5)L
        \le (\ceil{5p}+18)L,
    \end{align*}
    where we used $K \le 2L$, $\log_N 32 \le 5$, and $1+\log_N L \le L$.
    Padding by identity layers gives the stated class.
\end{proof}

\subsection{Proof of Proposition~\ref{prop:digit-extraction}}
\label{sec:digit-extraction-proof}

For $\bb \in \{0, \dots, \zeta-2\}^n$, write
$t_{\bb} = \sum_{k=1}^n b_k\zeta^{-k}$. Every input
$\theta = \sum_{k \ge 1} b_k\zeta^{-k}$ with leading digits
$(b_1, \dots, b_n) = \bb$ satisfies
\begin{align*}
    \sum_{k \ge 1} b_{n+k}\zeta^{-k}
    &= \zeta^n(\theta - t_{\bb}).
\end{align*}
The network must therefore output $\bb$ and the affine function
$\zeta^n(\theta - t_{\bb})$.  We use linear interpolation to extract the leading digits.

Since $b_{n+1}\le\zeta-2$ and all digits are below $\zeta-1$, it holds
\begin{align*}
    \sum_{k > n} b_k\zeta^{-k}
    &\le (\zeta-2)\zeta^{-n-1}+(\zeta-1)\sum_{k > n+1} \zeta^{-k}
    = \zeta^{-n}-\zeta^{-n-1},
\end{align*}
so every input with leading digits $\bb$ lies in the interval
\begin{align*}
    I_{\bb} = [t_{\bb}, t_{\bb} + \zeta^{-n}-\zeta^{-n-1}].
\end{align*}
Let $\bb\ne(\zeta-2,\dots,\zeta-2)$
and $k$ be the largest index with
$b_k<\zeta-2$. The word $\bb'$ following $\bb$ in lexicographic order is defined by $b_\ell' = b_\ell$ for $\ell < k$, $b_k' = b_k + 1$, and $b_{\ell}' = 0$ for $ \ell > k$. Since
\begin{align*}
    t_{\bb'} - t_{\bb}
    &= \zeta^{-k} - (\zeta-2)\sum_{m=k+1}^{n} \zeta^{-m}
    \ge \zeta^{-k} - (\zeta-1)\sum_{m=k+1}^{n} \zeta^{-m}
    = \zeta^{-n},
\end{align*}
consecutive intervals $I_{\bb}$ are separated by gaps of length at least
$\zeta^{-n-1}$.

\begin{figure}[tp]
\centering
\begin{tikzpicture}[font=\small, x=18cm, y=1cm]
    \definecolor{figblue}{RGB}{0,90,170}
    \definecolor{figbrown}{RGB}{185,75,0}
    \definecolor{figgold}{RGB}{170,125,0}
    \definecolor{figpurple}{RGB}{110,60,140}
    \tikzset{lab/.style={font=\scriptsize}}
    \def\ra{3.0} \def\rb{1.55} \def\rc{0.1}
    \foreach \a/\w in {0/00, 0.11111/01, 0.33333/10, 0.44444/11} {
        \fill[figblue!10] (\a, -0.05) rectangle ({\a+0.05556}, 4.15);
        \node[lab, black!60] at ({\a+0.02778}, 4.4) {$\w$};
    }
    \foreach \r/\lb in {\ra/{b_1}, \rb/{b_2}, \rc/{\text{remainder}}} {
        \draw[black!45] (0, \r) -- (0.53, \r);
        \node[lab, anchor=east] at (-0.03, {\r+0.35}) {$\lb$};
    }
    \foreach \r in {\ra, \rb} {
        \draw[black!45] (-0.005, {\r+1}) -- (0.005, {\r+1});
        \node[lab, anchor=east, black!60] at (-0.005, {\r+1}) {$1$};
    }
    \draw[black!45] (-0.005, {\rc+0.5}) -- (0.005, {\rc+0.5});
    \node[lab, anchor=east, black!60] at (-0.005, {\rc+0.5}) {$\tfrac12$};
    \draw[figblue, very thick]
        (0, \ra) -- (0.16667, \ra) -- (0.33333, {\ra+1}) -- (0.53, {\ra+1});
    \draw[figblue, very thick]
        (0, \rb) -- (0.05556, \rb) -- (0.11111, {\rb+1}) -- (0.16667, {\rb+1})
        -- (0.33333, \rb) -- (0.38889, \rb) -- (0.44444, {\rb+1}) -- (0.53, {\rb+1});
    \draw[figblue, very thick]
        (0, \rc) -- (0.05556, {\rc+0.5}) -- (0.11111, \rc) -- (0.16667, {\rc+0.5})
        -- (0.33333, \rc) -- (0.38889, {\rc+0.5}) -- (0.44444, \rc) -- (0.5, {\rc+0.5})
        -- (0.53, {\rc+0.5});
    \foreach \a/\lb in {0/0, 0.11111/\tfrac19, 0.33333/\tfrac13, 0.44444/\tfrac49, 0.5/\tfrac12} {
        \draw[black!45] (\a, \rc) -- (\a, {\rc-0.06});
        \node[lab, anchor=north] at (\a, {\rc-0.06}) {$\lb$};
    }
    \node[lab, anchor=west] at (0.535, \rc) {$\theta$};
\end{tikzpicture}
\caption{The function $h$ in the proof of \cref{prop:digit-extraction}
    for $\zeta = 3$ and $n = 2$. Admissible inputs $\theta$ lie in the four intervals
    $I_{\bb}$ (shaded). On each interval, the first two coordinates of $h$ are constant and the remainder coordinate is affine. Across the gaps, all coordinates are affine.}
\label{fig:digit-extraction}
\end{figure}
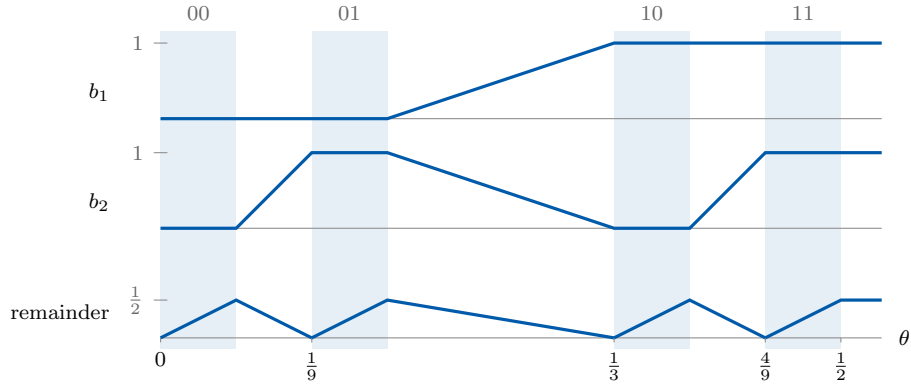
Define the function $h\colon [0,1] \to \R^{n+1}$ by
\begin{align*}
    h(\theta)
    &= \bigl(b_1, \dots, b_n, \zeta^n(\theta - t_{\bb})\bigr),
    \qquad \theta \in I_{\bb},
\end{align*}
using coordinatewise affine interpolation on the gaps between intervals, and constant extension after the last interval; see \cref{fig:digit-extraction}.
When crossing a gap, the first $n$ coordinates of $h$ change by at most
$\zeta-2$ and the last by at most $1$. Since every gap has length
at least $\zeta^{-n-1}$, the coordinate-wise slopes of $h$ are bounded as
\begin{align*}
    \max_{1 \le k \le n+1} |h_k'|
    &\le (\zeta-2)\zeta^{n+1}
    \le \zeta^{n+2}.
\end{align*}

Finally, there are $(\zeta-1)^n$ words $\bb$, hence $(\zeta-1)^n$
intervals $I_{\bb}$, each with two endpoints, where $h$ can change slope. Since moreover $h(0)=0$,
\cref{lem:linear-spline} with $A=\zeta^{n+2}$ gives a network $\phi$ that
equals $h$ on $[0,1]$ and satisfies
\begin{align*}
    \phi
    &\in \Fcal_{1,n+1}\bigl(2(\zeta-1)^n, 1, 2\zeta^{n+2}\bigr)
    \subseteq \Fcal_{1,n+1}\bigl(2\zeta^n, 1, 2\zeta^{n+2}\bigr). \tag*{\qedsymbol}
\end{align*}

\subsection{Proof of Proposition~\ref{prop:point-fitting-diff}}
\label{sec:point-fitting-proof}

As outlined in \cref{sec:higher-order-point-fitting}, rounding the sequence $y_i$ to $\lambda \floor{y_i/\lambda}$ makes every scaled
highest-order difference
\begin{align*}
    d_j=\lambda^{-1}\Delta^h\bar y_j
\end{align*}
an integer in $[-H,H]$. We store the digits $d_j + H \in \{0, \dots, 2H\}$ of
a single number in base $2H+2$ and read them with \cref{prop:digit-extraction}.
The following lemma reconstructs the rounded sequence values by extracting the digits and evaluating the discrete Taylor formula recursively.

\begin{lemma}\label{lem:block-evaluator}
    Let $\lambda\in(0,1)$ and $h,N,L,q\in\N_+$, set $H=2^{h+2}$ and
    $\zeta=2H+2$, and suppose that $q\le\floor{\log_\zeta N}$ and
    $N\ge\zeta$. There is a network
    \begin{align*}
        \phi_{\rm dec}
         & \in\wt\Fcal_{h+2,1}(N,L,N)
    \end{align*}
    such that the following holds. Let $1\le T\le Lq$ and let
    $\bar y_1,\dots,\bar y_T$ be integer multiples of $\lambda$ satisfying the
    following bounds:
    \begin{align*}
        \max_{0\le r\le\min\{h-1,T-1\}}
        \max_{1\le j\le T-r}|\Delta^r\bar y_j|
         & \le H,
         & \max_{1\le j\le T-h}|\Delta^h\bar y_j|
         & \le H\lambda.
    \end{align*}
    Then there are $u_0,\dots,u_{h-1}\in[-H,H]$ and
    $\theta\in[0,1]$ such that
    \begin{align*}
        \phi_{\rm dec}(k,u_0,\dots,u_{h-1},\theta)
         & =\bar y_k,\qquad k=1,\dots,T.
    \end{align*}
    The constant hidden in $\wt\Fcal$ depends only on $h$.
\end{lemma}

\begin{figure}[t]
    \centering
    \begin{tikzpicture}[>=Stealth, font=\small,
    box/.style={draw, rounded corners=2pt, minimum height=0.6cm, minimum width=1.0cm,
                inner sep=3pt, align=center, fill=white},
    upd/.style={box, minimum width=5.4cm, minimum height=1.0cm},
    lab/.style={font=\scriptsize, align=center},
    wire/.style={->},
    bus/.style={preaction={draw=white, line width=3pt}}]

    \begin{scope}
        \node[anchor=south west, font=\small] at (0, 0.7)
            {(a) $\phi_{\rm dec}$: $L$ segment networks in sequence};
        \node[anchor=east] (in) at (4.4, 0) {$\bigl(0,\ k,\ \bu,\ \theta_0\bigr)$};
        \node[box] (b0) at (5.5, 0) {$\phi_{\rm seg}$};
        \node[box] (b1) at (7.6, 0) {$\phi_{\rm seg}$};
        \node (dd) at (9.2, 0) {$\cdots$};
        \node[box] (bL) at (10.8, 0) {$\phi_{\rm seg}$};
        \node[anchor=west] (out) at (11.9, 0) {$[\,\cdot\,]_1 = \bar y_k$};
        \draw[wire] (in) -- (b0);
        \draw[wire] (b0) -- (b1);
        \draw[wire] (b1) -- (dd);
        \draw[wire] (dd) -- (bL);
        \draw[wire] (bL) -- (out);
        \node[lab, above=0.02cm of b0] {$t = 0$};
        \node[lab, above=0.02cm of b1] {$t = 1$};
        \node[lab, above=0.02cm of bL] {$t = L - 1$};
        \node[lab, anchor=north] at (7.4, -0.45)
            {state entering segment $t$:
             $\bigl(\bar y_k\ind\{k \le tq\},\ \sigma(k - tq),\ \bu^{(t)},\ \theta_t\bigr)$};
    \end{scope}

    \begin{scope}[shift={(0, -7.8)}]
        \node[anchor=south west, font=\small] at (0, 6.1)
            {(b) $\phi_{\rm seg}$: segment $t$ holds positions $tq + 1, \dots, tq + q$};
        \coordinate (kin) at (0.9, 5.5);
        \coordinate (zin) at (0.9, 4.1);
        \coordinate (uin) at (0.9, 2.7);
        \coordinate (tin) at (0.9, 0.85);
        \node[anchor=east] at (kin) {$k$};
        \node[anchor=east] at (zin) {$z$};
        \node[anchor=east] at (uin) {$\bu^{(t)}$};
        \node[anchor=east] at (tin) {$\theta_t$};
        \node[box, minimum width=1.6cm] (K) at (6.9, 5.5) {$\sigma(k - q)$};
        \node[upd] (Z) at (6.9, 4.1) {$z + \ind\{1 \le k \le q\}\,\phi_Y(i_q(k), \bu^{(t)}, \bd^{(t)})$};
        \node[upd] (U) at (6.9, 2.7) {$\phi_U(\bu^{(t)}, \bd^{(t)})$};
        \node[upd] (S) at (6.9, 0.85) {$\phi_{\rm stream}(\theta_t)$};
        \draw[wire] (kin) -- (K.west);
        \draw[wire] (zin) -- (Z.west);
        \draw[wire] (uin) -- (U.west);
        \draw[wire] (tin) -- (S.west |- tin);
        \draw[wire] (K.east) -- ++(0.6, 0);
        \draw[wire] (Z.east) -- ++(0.6, 0);
        \draw[wire] (U.east) -- ++(0.6, 0);
        \draw[wire] (S.east |- tin) -- ++(0.6, 0);
        \fill (3.3, 5.5) circle (1.3pt);
        \draw[bus] (3.3, 5.5) -- (3.3, 4.45);
        \draw[wire] (3.3, 4.45) -- (Z.west |- 0, 4.45);
        \node[lab, anchor=east] at (3.2, 4.95) {$k$};
        \fill (1.4, 2.7) circle (1.3pt);
        \draw[wire, bus] (1.4, 2.7) -- (1.4, 3.9) -- (Z.west |- 0, 3.9);
        \node[lab, anchor=west] at (1.45, 3.25) {$\bu^{(t)}$};
        \draw[wire] (S.north -| 9.2, 0) -- (U.south -| 9.2, 0);
        \fill (9.2, 1.6) circle (1.3pt);
        \draw[wire, bus] (9.2, 1.6) -- (3.75, 1.6) -- (3.75, 3.7) -- (Z.west |- 0, 3.7);
        \node[lab, anchor=south west] at (3.95, 1.5) {$\bd^{(t)} = (d_{tq+1}, \dots, d_{tq+q})$};
        \node[lab, anchor=west] at (10.3, 5.5) {positions still ahead};
        \node[lab, anchor=west] at (10.3, 4.1) {$= z + \bar y_k \ind_{1 \le k \le q}$};
        \node[lab, anchor=west] at (10.3, 2.7) {$= \bu^{(t+1)}$};
        \node[lab, anchor=west] at (10.3, 0.85) {$\theta_{t+1}$};
    \end{scope}
\end{tikzpicture}
    \caption{The decoder $\phi_{\rm dec}$ of \cref{lem:block-evaluator}. (a) It consists of
        $L$ copies of one segment network acting on the state
        $(z, k, \bu, \theta)$, with each segment covering $q$ positions.
        (b) Inside a segment, $\phi_{\rm stream}$ from
        \cref{prop:digit-extraction} decodes the $q$ stored $h$-th differences
        $\bd$, the networks $\phi_Y$ and $\phi_U$ evaluate the discrete Taylor
        formulas $Y$ and $U$, and the indicator $\ind_{1 \le k \le q}$ decides whether the segment writes the requested value $\bar y_k$. }
    \label{fig:block-evaluator}
\end{figure}

\begin{proof}
    We evaluate the discrete Taylor formula in segments of $q$
    positions at a time; see \cref{fig:block-evaluator}. We first derive the
    segment recursion, then construct networks for its components, and
    compose them at the end.

    \medskip
    \noindent\textit{Step 1: Segment recursion}

    \medskip
    \noindent
    Extend $(\bar y_j)_{j\le T}$ to a
    sequence $(\bar y_j)_{j\ge1}$ by requiring $\Delta^r\bar y_1=0$ for
    $T\le r<h$ and $\Delta^h\bar y_j=0$ for $j>T-h$. This determines
    $\bar y_j$ for $j>T$ recursively from the following values. For $r<h$ and $j\in\N_+$, set
    \begin{align*}
        u_r
        &=\Delta^r\bar y_1,
        &
        d_j
        &=\lambda^{-1}\Delta^h\bar y_j.
    \end{align*}
    The decoder walks through $L$ segments of $q$ positions each. Segment $t$ holds the
    positions $tq+1,\dots,tq+q$, the differences
    $\bd^{(t)}=(d_{tq+1},\dots,d_{tq+q})$, and the state
    $\bu^{(t)}=(\Delta^r\bar y_{tq+1})_{r<h}$, with
    $\bu^{(0)}=\bu=(u_0,\dots,u_{h-1})$. Applied to the shifted sequence
    $(\bar y_{tq+j})_j$, \eqref{eq:discrete-taylor} gives
    $\bar y_{tq+i}=Y(i,\bu^{(t)},\bd^{(t)})$ and $\bu^{(t+1)}=U(\bu^{(t)},\bd^{(t)})$,
    where
    \begin{align*}
        Y(i,\bu,\bd)
         & =\sum_{r=0}^{h-1}\binom{i-1}{r}u_r
        +\lambda\sum_{j=1}^{i-h}\binom{i-j-1}{h-1}d_j,
         & i                                  & =1,\dots,q,
        \\
        U_r(\bu,\bd)
         & =
        \sum_{p=0}^{h-r-1}\binom{q}{p}u_{r+p}
        +\lambda\sum_{j=1}^{q-h+r+1}
        \binom{q-j}{h-r-1}d_j,
         & r                                  & =0,\dots,h-1.
    \end{align*}
    For $U$, this follows by applying \eqref{eq:discrete-taylor}
    with order $h-r$ to the sequence $(\Delta^r\bar y_{tq+j})_j$ at $j=q+1$.

    \medskip
    \noindent\textit{Step 2: Component networks}

    \medskip
    \noindent
    Four components are needed to implement the segment recursion:
    \begin{itemize}
        \item \emph{Digit stream.} Define
        \begin{align*}
            \theta_t
            &=\sum_{\ell=1}^\infty (d_{tq+\ell}+H)\,\zeta^{-\ell}
            \in[0,1),
            \qquad t\in\N_0.
        \end{align*}
        In particular, $\theta_0$ stores all differences $d_1,d_2,\dots$.
        \Cref{prop:digit-extraction}, with $H$ subtracted from its first $q$
        outputs, gives a network
        \begin{align*}
            \phi_{\rm stream}
             & \in\Fcal_{1,q+1}(2\zeta^q,1,2\zeta^{q+2}),
             & \phi_{\rm stream}(\theta_t)
             & =\bigl(\bd^{(t)},\theta_{t+1}\bigr),
            \qquad t\in\N_0.
        \end{align*}
        We have $\zeta^q\le N$, and $\zeta^2$ depends only on $h$.

        \item \emph{Taylor formulas.} By \cref{lem:wsum} with $C=H$, and noting that $q^hH\lesssim N$, the maps $Y$ and $U$ are realized by networks
        \begin{align*}
            \phi_Y
            &\in\wt\Fcal_{h+q+1,1}(q,1,N),
            &
            \phi_U
            &\in\wt\Fcal_{h+q,h}(q,1,N).
        \end{align*}

        \item \emph{Segment selection.} For $k\in\N_0$, define
        \begin{align*}
            i_q(k) & =1+\sigma(k-1)-\sigma(k-q)
            = \clip_{[1, q]}(k),              \\
            a_q(k) & =\sigma(k)-\sigma(k-1)
            -\sigma(k-q)+\sigma(k-q-1)=\ind_{1\le k\le q}.
        \end{align*}
        Thus
        $\phi_{\rm sel}(k)=(i_q(k),a_q(k))$ is realized by
        $\phi_{\rm sel}\in\wt\Fcal_{1,2}(1,1,q)$.

        \item \emph{Gate.} By \cref{lem:ind-product}, there is
        \begin{align*}
            \phi_{\rm gate}
             & \in\wt\Fcal_{2,1}(1,1,H),
             & \phi_{\rm gate}(y,\delta)
             & =\delta y,\qquad |y|\le H,\ \delta\in\{0,1\},
        \end{align*}
        and $\phi_{\rm gate}(y,0)=0$ for every $y\in\R$.
    \end{itemize}

    \medskip
    \noindent\textit{Step 3: Segment network and decoder}

    \medskip
    \noindent
    Writing $\phi_{\rm stream}(\theta)=(\bd,\theta')$ and
    $\phi_{\rm sel}(k)=(i,a)$, define
    \begin{align*}
        \phi_{\rm seg}(z,k,\bu,\theta)
         & =\bigl(z+ a\phi_Y(i,\bu,\bd),\sigma(k-q),
        \phi_U(\bu,\bd),\theta'\bigr),
    \end{align*}
    where $\phi_{\rm gate}$ implements the product $a\phi_Y(i,\bu,\bd)$. Since
    \begin{align*}
        a_q(\sigma(k-tq))
         & =\ind_{t=\floor{(k-1)/q}}, \qquad t=0,\dots,L-1,
    \end{align*}
    exactly one segment contributes to the first coordinate. Since $\phi_U$
    is exact, $\bu$ equals $\bu^{(t)}$ in every segment. In the segment
    containing $k$, let $i=k-tq$. Then $Y(i,\bu^{(t)},\bd^{(t)})$ only uses
    $u_r=\Delta^r\bar y_{tq+1}$ with $r<i$ and $d_j$ with $tq+j\le k-h$.
    These are differences of $\bar y$ at positions at most $k\le T$, so they
    lie in $[-H,H]$. Hence $\phi_Y$ is exact there, and the contribution is
    $\bar y_k$. In the
    other segments, $\bu$ may leave $[-H,H]^h$, but there $a=0$ and
    $\phi_{\rm gate}(y,0)=0$ for every $y$, so they do not contribute.
    All components have width and weights of order at most $N$
    by Step~2, so \cref{lem:network-calculus} yields
    $\phi_{\rm seg}\in\wt\Fcal_{h+3,h+3}(N,1,N)$.

    We define the final network $\phi_{\rm dec}$ as the $L$-fold composition of $\phi_{\rm seg}$. Set $\theta=\theta_0$. By Steps~1 and~3,
    \begin{align*}
        \phi_{\rm dec}(k,u_0,\dots,u_{h-1},\theta)
         & =\left[
            \phi_{\rm seg}^{\circ L}
            \bigl(0,k,\bu,\theta\bigr)
            \right]_1 = \bar y_k.
    \end{align*}
    This gives
    $\phi_{\rm dec}\in\wt\Fcal_{h+2,1}(N,L,N)$.
\end{proof}

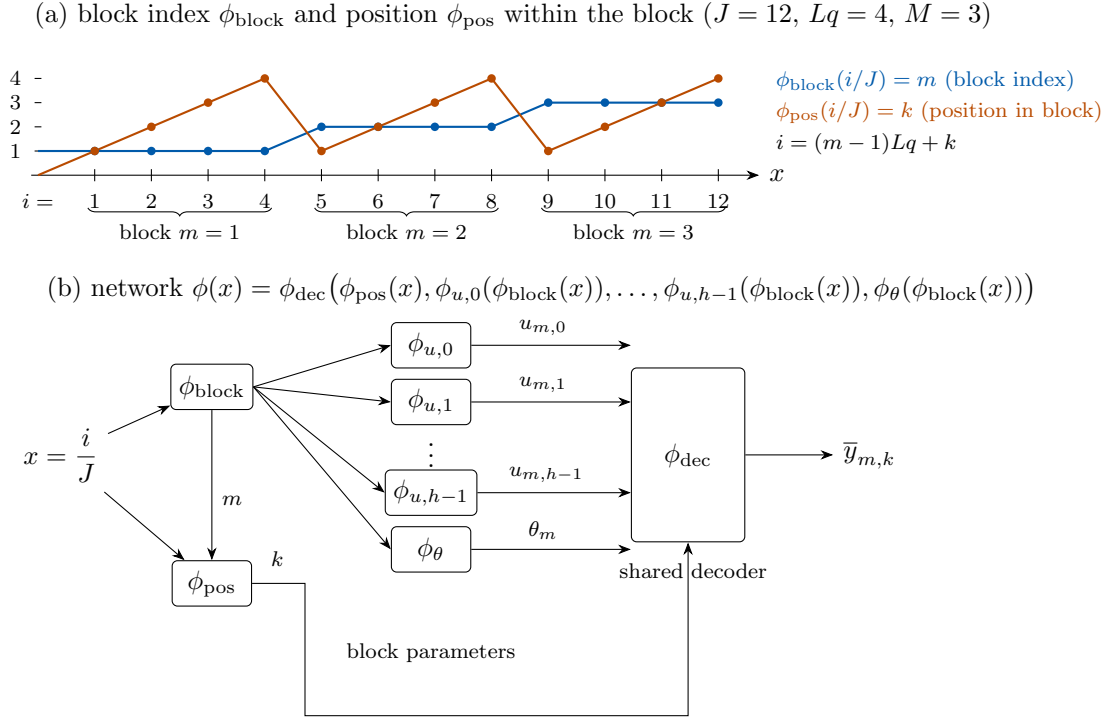
\begin{figure}[t]
    \centering
    \begin{tikzpicture}[>=Stealth, font=\small,
    box/.style={draw, rounded corners=2pt, minimum height=0.6cm, minimum width=1.05cm,
                inner sep=3pt, align=center, fill=white},
    lab/.style={font=\scriptsize, align=center},
    wire/.style={->}]
    \definecolor{figblue}{RGB}{0,90,170}
    \definecolor{figbrown}{RGB}{185,75,0}
    \definecolor{figgold}{RGB}{170,125,0}
    \definecolor{figpurple}{RGB}{110,60,140}

    \begin{scope}[x=9.0cm, y=0.32cm]
        \node[anchor=south west, font=\small] at (-0.02, 5.6)
            {(a) block index $\phi_{\rm block}$ and position $\phi_{\rm pos}$ within the block ($J = 12$, $Lq = 4$, $M = 3$)};
        \draw[->] (-0.01, 0) -- (1.06, 0) node[anchor=west] {$x$};
        \foreach \i in {1,...,12} {
            \draw ({\i/12}, 0.25) -- ({\i/12}, -0.25);
            \node[font=\scriptsize, anchor=north] at ({\i/12}, -0.35) {$\i$};
        }
        \node[font=\scriptsize, anchor=north] at (0.0, -0.35) {$i = $};
        \foreach \v in {1,2,3,4} {
            \draw (-0.005, \v) -- (0.005, \v);
            \node[font=\scriptsize, anchor=east] at (-0.01, \v) {$\v$};
        }
        \draw[figblue, thick] (0, 1) -- (4/12, 1) -- (5/12, 2) -- (8/12, 2) -- (9/12, 3) -- (1, 3);
        \draw[figbrown, thick] (0, 0) -- (4/12, 4) -- (5/12, 1) -- (8/12, 4) -- (9/12, 1) -- (1, 4);
        \foreach \i/\m/\k in {1/1/1,2/1/2,3/1/3,4/1/4,5/2/1,6/2/2,7/2/3,8/2/4,9/3/1,10/3/2,11/3/3,12/3/4} {
            \fill[figblue] ({\i/12}, \m) circle (1.6pt);
            \fill[figbrown] ({\i/12}, \k) circle (1.6pt);
        }
        \node[anchor=west, font=\scriptsize, figblue] at (1.07, 3.9) {$\phi_{\rm block}(i/J) = m$ (block index)};
        \node[anchor=west, font=\scriptsize, figbrown] at (1.07, 2.6) {$\phi_{\rm pos}(i/J) = k$ (position in block)};
        \node[anchor=west, font=\scriptsize] at (1.07, 1.3) {$i = (m - 1)Lq + k$};
        \foreach \m/\a/\b in {1/1/4, 2/5/8, 3/9/12} {
            \draw[decorate, decoration={brace, mirror, amplitude=3pt}]
                ({\a/12-0.01}, -1.35) -- ({\b/12+0.01}, -1.35);
            \node[font=\scriptsize, anchor=north] at ({(\a+\b)/24}, -1.6) {block $m = \m$};
        }
    \end{scope}

    \begin{scope}[shift={(0, -2.3)}]
        \node[anchor=south west, font=\small] at (-0.02, 0.45)
            {(b) network $\phi(x) = \phi_{\rm dec}\bigl(\phi_{\rm pos}(x), \phi_{u,0}(\phi_{\rm block}(x)), \dots, \phi_{u,h-1}(\phi_{\rm block}(x)), \phi_\theta(\phi_{\rm block}(x))\bigr)$};
        \node (x) at (0.3, -1.4) {$x = \dfrac{i}{J}$};
        \node[box] (phim) at (2.3, -0.5) {$\phi_{\rm block}$};
        \node[box] (phik) at (2.3, -3.1) {$\phi_{\rm pos}$};
        \node[box] (u0) at (5.2, 0.05) {$\phi_{u,0}$};
        \node[box] (u1) at (5.2, -0.7) {$\phi_{u,1}$};
        \node (dots) at (5.2, -1.3) {$\vdots$};
        \node[box] (uh) at (5.2, -1.9) {$\phi_{u,h-1}$};
        \node[box] (pt) at (5.2, -2.65) {$\phi_\theta$};
        \node[box, minimum height=2.3cm, minimum width=1.5cm] (dec) at (8.6, -1.4) {$\phi_{\rm dec}$};
        \node (out) at (11.0, -1.4) {$\bar y_{m,k}$};
        \draw[wire] (x) -- (phim);
        \draw[wire] (x) -- (phik);
        \draw[wire] (phim) -- (phik) node[pos=0.6, right, lab] {$m$};
        \foreach \n in {u0,u1,uh,pt} {
            \draw[wire] (phim.east) -- (\n.west);
        }
        \draw[wire] (u0.east) -- (dec.west |- u0.east) node[pos=0.45, above, lab] {$u_{m,0}$};
        \draw[wire] (u1.east) -- (dec.west |- u1.east) node[pos=0.45, above, lab] {$u_{m,1}$};
        \draw[wire] (uh.east) -- (dec.west |- uh.east) node[pos=0.45, above, lab] {$u_{m,h-1}$};
        \draw[wire] (pt.east) -- (dec.west |- pt.east) node[pos=0.45, above, lab] {$\theta_m$};
        \draw[wire] (phik.east) -- ++(0.7, 0) |- (8.6, -4.85) -- (dec.south);
        \node[lab, anchor=south west] at (2.95, -3.0) {$k$};
        \draw[wire] (dec.east) -- (out) node[midway, above, lab] {};
        \node[lab, anchor=north, text width=3.8cm] at (5.2, -3.75) {block parameters};
        \node[lab, anchor=north, text width=3.2cm] at (8.67, -2.7) {shared  decoder};
    \end{scope}
\end{tikzpicture}
    \caption{The point-fitting network of \cref{prop:point-fitting-diff}. (a) The
        indices $i = 1, \dots, J$ are grouped into $M$ blocks of length $Lq$;
        the staircase $\phi_{\rm block}$ returns the block index and the sawtooth $\phi_{\rm pos}$
        the position within the block. (b) The block parameters $u_{m, 0}, \dots, u_{m, h-1}$, $\theta_m$ are retrieved from the block index and, together with $k$, fed into the
        decoder $\phi_{\rm dec}$ of \cref{lem:block-evaluator}.}
    \label{fig:point-fitting}
\end{figure}

    From here, we prove \cref{prop:point-fitting-diff} in two steps. We first split the sequence into blocks
    of length at most $Lq$ and construct networks recovering the block index
    and the position within the block from the input $i/J$. We then fit the
    values of all blocks with a single decoder whose parameters are
    interpolated across blocks. \Cref{fig:point-fitting} shows the layout of the network.

    \medskip
    \noindent\textit{Step 1: Block decomposition}

    \medskip
    \noindent
    Define
    \begin{align*}
        M=\ceil{J/(Lq)}\le3N^2L,
        \qquad
        T_m=\min\{Lq,J-(m-1)Lq\},
        \quad m=1,\dots,M.
    \end{align*}
    Every $i=1,\dots,J$ can be written uniquely as
    \begin{align*}
        i=(m-1)Lq+k,
        \qquad m=1,\dots,M,
        \quad k=1,\dots,T_m,
    \end{align*}
    and we write $y_{m,k}=y_{(m-1)Lq+k}$. Every block is a consecutive segment
    of $(y_i)$, so the assumed difference bounds hold within each block.

    Let $\phi_{\rm block}$ be the continuous staircase that is constant with value $m$
    between the sample points in the $m$-th block and affine between
    $mLq/J$ and $(mLq+1)/J$. It has at most $2M$ knots and satisfies

    \begin{align*}
        \phi_{\rm block}(i/J)=\left\lfloor\frac{i-1}{Lq}\right\rfloor+1,
        \qquad i=1,\dots,J.
    \end{align*}
    The knot separation is at least $1/J$ and the output
    is bounded by $M$. Since $M\le3N^2L$, $J\le3N^2L^2q$,
    and $q\le N$, these quantities satisfy the polynomial bounds of
    \cref{prop:pw-linear}. Applying that proposition with width and depth
    parameters $2N,2L$ allows for the $2M$ knots and gives
    $\phi_{\rm block}\in\wt\Fcal_{1,1}(N,L,N)$. Moreover,
    \begin{align*}
        \log_N J+\log_N(Lq) & \lesssim L,
    \end{align*}
    so \cref{lem:amplification} implements the two scalar multiplications in
    \begin{align*}
        \phi_{\rm pos}(x)=Jx-(\phi_{\rm block}(x)-1)Lq.
    \end{align*}
    Then $\phi_{\rm pos}\in\wt\Fcal_{1,1}(N,L,N)$, $\phi_{\rm block}(i/J)=m$, and $\phi_{\rm pos}(i/J)=k$
    when $i=(m-1)Lq+k$.

    \medskip
    \noindent\textit{Step 2: Fitting the blocks}

    \medskip
    \noindent
    Round each block separately by
        $\bar y_{m,k}=\lambda\floor{y_{m,k}/\lambda},$
    so that $|y_{m,k} - \bar y_{m,k}| \le \lambda$. By \eqref{eq:differences},
    rounding changes an $r$th difference by at most $2^r\lambda$.
    The assumed difference bounds and $\lambda < 1$ therefore give
    \begin{align*}
        |\Delta^r\bar y_{m,k}|
        &\le 1 + 2^r\lambda \le 2^{r+1}, \qquad r = 0, \dots, h-1, \\
        |\Delta^h\bar y_{m,k}|
        &\le (1 + 2^h)\lambda \le 2^{h+1}\lambda.
    \end{align*}
    Since $T_m\le Lq$ and $2^{h+1}\le H$, \cref{lem:block-evaluator} applies
    to each block. It gives a common decoder $\phi_{\rm dec}$ and values
    $u_{m,0},\dots,u_{m,h-1}$ in $[-H,H]$ and $\theta_m$ in $[0,1]$
    such that
    \begin{align}
        \label{eq:block-decoder-identity}
        \phi_{\rm dec}(k,u_{m,0},\dots,u_{m,h-1},\theta_m)
         & =\bar y_{m,k},\qquad k=1,\dots,T_m.
    \end{align}
    By \cref{prop:pw-linear}, there are networks
    \begin{align*}
        \phi_{u,0},\dots,\phi_{u,h-1},\phi_\theta
        \in\wt\Fcal_{1,1}(N,L,N)
    \end{align*}
    interpolating these $h+1$ quantities at
    $m=1,\ldots,M$.
    The desired network is
    \begin{align*}
        \phi(x)=\phi_{\rm dec}\bigl(\phi_{\rm pos}(x),
        \phi_{u,0}(\phi_{\rm block}(x)),\dots,\phi_{u,h-1}(\phi_{\rm block}(x)),
        \phi_\theta(\phi_{\rm block}(x))\bigr).
    \end{align*}
    For $i=(m-1)Lq+k$, Step~1 gives $\phi_{\rm block}(i/J)=m$ and $\phi_{\rm pos}(i/J)=k$, so
    \eqref{eq:block-decoder-identity} and the rounding bound yield
    \begin{align*}
        \max_{1\le i\le J}\left|\phi(i/J)-y_i\right|
         & =\max_m\max_k\left|\bar y_{m,k}-y_{m,k}\right|
        \le \lambda.
    \end{align*}
    By \cref{lem:network-calculus}, the composition belongs to
    $\wt\Fcal_{1,1}(N,L,N)$, which proves the claim. \qedsymbol

\subsection{Proof of Proposition~\ref{prop:smooth-flattening}}
\label{sec:smooth-flattening-proof}

Write $d=D$, $t=s-r$, and $\theta=t-h+1\in(0,1]$.

Fix a multi-index $\balpha$ and a grid resolution $K$. The coefficient values
$\partial^{\balpha}f(\bbeta/K)$ form an array indexed by
$\bbeta\in\{0,\dots,K-1\}^d$. For each fixed tuple
$(\beta_2,\dots,\beta_d)$, collect the $K$ coefficients obtained by varying
$\beta_1$. We call this ordered collection a row. Thus, there are $K^{d-1}$
rows, each containing $K$ coefficient values. We label the rows by the
base-$K$ index
\begin{align*}
\nu(\bbeta)=\sum_{i=2}^d\beta_iK^{i-2}.
\end{align*}
We place the start of row $\nu$ at index $2K\nu$ and define the flattened
index of $\bbeta$ by
\begin{align*}
\kappa(\bbeta)=\beta_1+2K\nu(\bbeta).
\end{align*}
The spacing between consecutive row starts reserves a block of indices for a
smooth transition; see \cref{fig:flattening}(a).

We first consider $d=2$ and $K\ge h+1$, which is the situation
illustrated in \cref{fig:flattening}. The cases $d=1$ and $d>2$ are
discussed at the end.
Note that fixing $\beta_2=\nu$ and varying $\beta_1$ produces the row
\begin{align*}
    \left\{
    \partial^{\balpha}f(\beta_1/K,\nu/K):
    \beta_1=0,\dots,K-1
    \right\},
    \qquad \nu=0,\dots,K-1.
\end{align*}
Let $\be_1=(1,0)^\top\in\R^2$ be the first standard basis vector and
define the corresponding row function by
\begin{align*}
    F_\nu(u)
     & =\partial^{\balpha}f(u,\nu/K),
    \qquad u\in[0,1].
\end{align*}
For $\ell=0,\dots,h-1$, differentiation gives
\begin{align*}
    F_\nu^{(\ell)}(u)
     & =\partial^{\balpha+\ell\be_1}f(u,\nu/K).
\end{align*}
As $f \in \Hcal_1^{s}([0,1]^d)$ and with the identities $t=s-r$ and
$\theta=t-h+1$, it follows that $\|F_\nu^{(\ell)}\|_\infty\le1$ and
\begin{align*}
    |F_\nu^{(h-1)}(u)-F_\nu^{(h-1)}(v)|
     & \le |u-v|^\theta,
    \qquad u,v\in[0,1].
\end{align*}

For $\nu=0,\dots,K-2$, let $H_\nu$ be the polynomial of degree at
most $2h-1$ satisfying
\begin{align*}
    H_\nu^{(\ell)}(0)
     & =F_\nu^{(\ell)}(1),
     & H_\nu^{(\ell)}(1)
     & =F_{\nu+1}^{(\ell)}(0),
    \qquad \ell=0,\dots,h-1.
\end{align*}
The Hermite interpolation map on $[0,1]$ is a linear map from the
$2h$ endpoint derivatives to the space of polynomials of degree at most
$2h-1$. Since this space is finite-dimensional and the endpoint
derivatives are bounded by one, $ \max_{0\le\ell\le h}\|H_\nu^{(\ell)}\|_\infty
    \lesssim 1$ uniformly in $\nu$.

Place the row functions on the even unit intervals and the bridges on the
intervening odd unit intervals by defining
\begin{align*}
    G(2\nu+u)
     & =F_\nu(u),
     &            & \nu=0,\dots,K-1, \\
    G(2\nu+1+u)
    & =H_\nu(u),
    &            & \nu=0,\dots,K-2,
\end{align*}
for $u\in[0,1]$. The endpoint conditions on $H_\nu$ make $G$ well
defined. In fact,
\begin{align*}
    G&\in C^{h-1}([0,2K-1]),
    &
    \max_{0\le\ell\le h-1}\|G^{(\ell)}\|_\infty
    &\lesssim 1.
\end{align*}
On every row interval, $G^{(h-1)}$ is $\theta$-H\"older. On every
bridge interval, the bound on $H_\nu^{(h)}$ makes it Lipschitz and hence
$\theta$-H\"older. Because $h/K<1$, two points at distance at most
$h/K$ cross at most one boundary between pieces. Inserting that boundary
point and using the matching derivatives therefore gives
\begin{align*}
    \sup_{\substack{u,v\in[0,2K-1] \\|u-v|\le h/K}}
    |G^{(h-1)}(u)-G^{(h-1)}(v)|
     & \lesssim K^{-\theta}.
\end{align*}

In the case $d=2$, we have $M-1=K(2K-1)$. Set
\begin{align*}
    y_i
     & =G(i/K),
    \qquad i=0,\dots,M-1.
\end{align*}
For every $\bbeta\in\{0,\dots,K-1\}^2$,
\begin{align*}
    \frac{\kappa(\bbeta)}{K}
    =2\beta_2+\frac{\beta_1}{K}.
\end{align*}
Consequently,
\begin{align*}
    y_{\kappa(\bbeta)}
    =F_{\beta_2}(\beta_1/K)
    =\partial^{\balpha}f(\bbeta/K),
\end{align*}
which proves \eqref{eq:smooth-flattening-interpolation}. The spacing
$2K$ between consecutive row starts leaves precisely the indices needed
to sample the intervening bridge.

To prove \eqref{eq:smooth-flattening-differences}, set $\delta=1/K$.
The case $\ell=0$ follows from $\|G\|_\infty\lesssim1$. For
$\ell=1,\dots,h-1$, iterating the fundamental theorem of calculus gives
the integral representation
\begin{align*}
    \Delta^\ell y_i
    =\int_0^\delta\!\cdots\!\int_0^\delta
    G^{(\ell)}(i\delta+t_1+\cdots+t_\ell)
    \,dt_1\cdots dt_\ell.
\end{align*}
Consequently,
\begin{align*}
    |\Delta^\ell y_i|
     & \le K^{-\ell}\|G^{(\ell)}\|_\infty
    \lesssim K^{-\ell}.
\end{align*}
Applying the same representation $h-1$ times and leaving the final
difference unexpanded yields
\begin{align*}
    \Delta^h y_i
    ={}&\int_0^\delta\!\cdots\!\int_0^\delta
    \Bigl[
    G^{(h-1)}(i\delta+\delta+t_1+\cdots+t_{h-1}) \\
    &\hspace{42mm}
    -G^{(h-1)}(i\delta+t_1+\cdots+t_{h-1})
    \Bigr]\,dt_1\cdots dt_{h-1},
\end{align*}
where the integral is absent when $h=1$. Hence
\begin{align*}
    |\Delta^h y_i|
     & \le K^{-(h-1)}
    \sup_{\substack{u,v\in[0,2K-1] \\|u-v|\le h/K}}
    |G^{(h-1)}(u)-G^{(h-1)}(v)|
    \lesssim K^{-(h-1+\theta)}
    =K^{-t}.
\end{align*}
Since $t\in(h-1,h]$, these are precisely the bounds in
\eqref{eq:smooth-flattening-differences}.

This proves the claim for $d=2$ and $K\ge h+1$. For $d>2$, enumerate the
tuples $(\beta_2,\dots,\beta_d)$ by
\begin{align*}
    \nu(\bbeta)=\sum_{i=2}^d\beta_iK^{i-2},
    \qquad \nu(\bbeta)\in\{0,\dots,R-1\},
    \qquad R=K^{d-1}.
\end{align*}
If $\tilde\bbeta(\nu)=(0,\beta_2(\nu),\dots,\beta_d(\nu))$ and
$\be_1=(1,0,\dots,0)^\top\in\R^d$ is the first standard basis vector,
replace the two-dimensional row functions above by
\begin{align*}
    F_\nu(u)
    =\partial^{\balpha}f\left(
    \frac{\tilde\bbeta(\nu)}{K}+u\be_1
    \right),
    \qquad \nu=0,\dots,R-1.
\end{align*}
The construction depends only on the number of rows and their uniform
smoothness bounds, so the preceding argument applies verbatim with $K$
rows replaced by $R$ rows. Moreover,
\begin{align*}
    M-1=K(2R-1)=2K^d-K,
\end{align*}
and $\kappa(\bbeta)/K=2\nu(\bbeta)+\beta_1/K$, proving both parts in
general dimension. For $d=1$, there is only one row: setting
$y_i=\partial^{\balpha}f(i/K)$ for $i=0,\dots,K$ gives the result directly
by the same finite-difference estimates, and $M=K+1$ as required.

It remains to consider $K<h+1$. In this case, set
$y_{\kappa(\bbeta)}=\partial^{\balpha}f(\bbeta/K)$ and set all remaining
entries to zero. Then \eqref{eq:smooth-flattening-interpolation} holds and,
because $|y_i|\le1$, \eqref{eq:differences} gives
$|\Delta^\ell y_i|\le2^\ell$. Since $K\le h$, this implies
\eqref{eq:smooth-flattening-differences} after adjusting the implicit
constant, which may depend on $s$. \qed

\subsection{Proof of Proposition~\ref{prop:taylor-coefficient-retrieval}}
\label{sec:coefficient-retrieval-proof}

Fix $\balpha\in\N_0^D$ with $\|\balpha\|_1\le m$ and define
\begin{align*}
    t_{\balpha}=s-\|\balpha\|_1,
    \qquad
    h_{\balpha}=\ceil{t_{\balpha}}
    =\ceil{s}-\|\balpha\|_1.
\end{align*}
By \cref{prop:smooth-flattening}, the array
$(\partial^{\balpha}f(\bbeta/K)/\balpha!)_{\bbeta}$ can be embedded into a
sequence $y_0,\ldots,y_{M-1}$ satisfying
\begin{align*}
    \max_i|\Delta^\ell y_i|
    &\lesssim K^{-\min\{\ell,t_{\balpha}\}},
    \qquad \ell=0,\ldots,h_{\balpha}.
\end{align*}
After
division by a constant depending only on $s$, the
differences below order $h_{\balpha}$ are bounded by one and those of
order $h_{\balpha}$ are bounded by $K^{-t_{\balpha}}$.
The encoding base $\zeta_{\balpha}=2^{h_{\balpha}+3}+2$ satisfies
$\zeta_{\balpha}\le\zeta\le N$, so
$q\le\floor{\log_{\zeta_{\balpha}}N}$.
Applying \cref{prop:point-fitting-diff} with
\begin{align*}
    J=M,
    \qquad h=h_{\balpha},
    \qquad \lambda=K^{-t_{\balpha}}
\end{align*}
and rescaling the output gives a
network $\phi_{{\rm fit},\balpha}\in\wt\Fcal_{1,1}(N,L,N)$ satisfying
\begin{align}
    \label{eq:scalar-coefficient-retrieval}
    \left|\phi_{{\rm fit},\balpha}
    \left(\frac{\kappa(\bbeta)+1}{M}\right)
    -\frac{\partial^{\balpha}f(\bbeta/K)}{\balpha!}\right|
    \lesssim K^{-t_{\balpha}}.
\end{align}
For grid points $\bx=\bbeta/K$, the normalized flattened index is the
affine function
\begin{align*}
    \rho_K(\bx)
    =\frac{Kx_1+2\sum_{j=2}^D K^j x_j+1}{M}
    =\frac{\kappa(\bbeta)+1}{M}.
\end{align*}
Its coefficients are bounded by two uniformly in $K$. Consequently,
$\phi_{{\rm coef},\balpha}=\phi_{{\rm fit},\balpha}\circ\rho_K$ satisfies $\phi_{{\rm coef},\balpha}(\bbeta/K)=\phi_{{\rm fit},\balpha}((\kappa(\bbeta)+1)/M)$, and parallelizing over all
$\|\balpha\|_1\le m$ gives
\begin{align*}
    \phi_{\rm coef}=(\phi_{{\rm coef},\balpha})_{\|\balpha\|_1\le m} \in\wt\Fcal_{D,\binom{D+m}{m}}(N,L,N).
\end{align*}
The bound
\eqref{eq:simultaneous-coefficient-retrieval} follows from
\eqref{eq:scalar-coefficient-retrieval}. \qed

\subsection{Proof of Proposition~\ref{lem:poly-approx}}
\label{sec:monomial-proof}

\begin{lemma}[Approximation of $x^2$]
    \label{lem:square-approx}
    For every $N\ge2$ and $L\in\N_+$, setting
    $q=\lfloor\log_2N\rfloor$, there is
        $\phi_{\rm sq}
        \in\Fcal_{1,1}(2N,2L,4N)$
    satisfying
    \begin{align*}
        \sup_{x\in[0,1]}
        \bigl|\phi_{\rm sq}(x)\bigr| \le 1 \qquad \text{and} \qquad
        \sup_{x\in[0,1]}
        \bigl|\phi_{\rm sq}(x)-x^2\bigr|
        \le N^{-L}.
    \end{align*}
\end{lemma}

\begin{proof}[Proof of \cref{lem:square-approx}]
    Recall that the function $S_r$ in
    \eqref{eq:square-interpolant} is a linear interpolant of $x^2$ on $2^r+1$ equally spaced points in $[0,1]$. Linear interpolation of a quadratic on intervals of
    length $2^{-r}$ gives
    \begin{align}
        \label{eq:square-interpolation-error}
        0\le S_r(x)-x^2\le 4^{-(r+1)},
        \qquad x\in[0,1].
    \end{align}
    Let $q=\lfloor\log_2N\rfloor$. The map
    $x\mapsto(S_q(x),\tau_q(x))$ is CPwL on a dyadic grid with $2^q\le N$ pieces. The slopes of
    $S_q$ have magnitude at most two, and the slopes of $\tau_q$ have magnitude
    $2^q$. \Cref{lem:linear-spline} with $d'=2$ and $A=2^q$ therefore realizes
    this map by a network in $\Fcal_{1,2}(2^q,1,2^{q+1})\subseteq\Fcal_{1,2}(N,1,2N)$.

    To construct a network of appropriate layout, we group the $qL$ refinements into $L$ blocks. Write
    \begin{align}
        \label{eq:square-grouped-sum}
        S_{qL}(x)
        &=x-\sum_{\ell=0}^{L-1}\sum_{i=1}^q
            4^{-(\ell q+i)}\tau_{\ell q+i}(x)
        =x-\sum_{\ell=0}^{L-1}4^{-\ell q}
            \sum_{i=1}^q4^{-i}
            \tau_i\bigl(\tau_{\ell q}(x)\bigr),
    \end{align}
    where the second identity uses
    $\tau_{m+n}=\tau_m\circ\tau_n$. To realize this  starting from
    $(a_0,y_0)=(x,x)$ and iterating over
    $\ell=0,\ldots,L-1$, perform the update
    \begin{align*}
        a_{\ell+1}
         &=a_\ell-4^{-\ell q}\{y_\ell-S_q(y_\ell)\},
        &
        y_{\ell+1}
         &=\tau_q(y_\ell).
    \end{align*}
    Since $y-S_q(y)=\sum_{i=1}^q4^{-i}\tau_i(y)$, \eqref{eq:square-grouped-sum} shows inductively that
    $y_\ell=\tau_{\ell q}(x)$ and $a_\ell=S_{\ell q}(x)$. Both variables
    lie in $[0,1]$. Each block can thus use one hidden layer for the common
    hinge basis and one to pass the updated state through ReLU. Carrying
    $a_\ell$ requires one additional unit, so the width is at most
    $2^q+1\le2N$; all parameters are bounded by $4N$. The $L$ blocks hence
    realize
    $\phi_{\rm sq}=S_{qL}\in\Fcal_{1,1}(2N,2L,4N)$. Since
    $2\lfloor\log_2N\rfloor\ge\log_2N$ for $N\ge2$,
    \eqref{eq:square-interpolation-error} implies
    \begin{align}
        \label{eq:square-network-error}
        \sup_{x\in[0,1]}
        |\phi_{\rm sq}(x)-x^2|
        \le4^{-qL}\le N^{-L},
    \end{align}
    completing the proof.
\end{proof}

\begin{proof}[Proof of \cref{lem:poly-approx}]
    Let $\phi_{\rm sq}$ be the network from
    \cref{lem:square-approx}.
    For $u,v\in[0,1]$, set
    \begin{align*}
        \phi_{\rm mult}(u,v)
        =\clip_{[0,1]}\left(
            2\phi_{\rm sq}\left(\frac{u+v}{2}\right)
            -\frac12\phi_{\rm sq}(u)
            -\frac12\phi_{\rm sq}(v)
        \right).
    \end{align*}
    Since
    $uv\in[0,1]$, clipping does not increase the error, and \eqref{eq:square-network-error} gives
    \begin{align}
        \label{eq:multiplication-network-error}
        |\phi_{\rm mult}(u,v)-uv|
         \le3N^{-L}.
    \end{align}
    Three parallel square networks followed by the clipping network from \cref{lem:clipping} give
    \begin{align*}
        \phi_{\rm mult}\in\Fcal_{2,1}(6N,2L+1,4N).
    \end{align*}

    Finally, let $r=\|\balpha\|_1$ and list coordinates
    $z_1,\ldots,z_r$ so that coordinate $x_j$ occurs $\alpha_j$ times. For $r\ge2$, define successively
    \begin{align*}
        p_1=z_1,
        \qquad
        p_j=\phi_{\rm mult}(p_{j-1},z_j),
        \quad j=2,\ldots,r.
    \end{align*}
    Every $p_j$ takes values in $[0,1]$. Using
    \eqref{eq:multiplication-network-error} and $z_j\le1$ inductively,
    \begin{align*}
        \left|p_j-\prod_{i=1}^jz_i\right|
        &\le3N^{-L}
          +\left|p_{j-1}-\prod_{i=1}^{j-1}z_i\right|
        \le3(j-1)N^{-L}.
    \end{align*}
    We carry the $D$ input coordinates through identity units
    while composing the multiplication modules. This gives width at most
    $6N+D\le(D+6)N$. Inserting one identity layer between consecutive
    modules keeps the parameter bound at $4N$, while the depth is at most
        $(r-1)(2L+2)\le4kL,$
    so
    $p_r\in\Fcal_{D,1}((D+6)N,4kL,4N)$. Taking
    $\phi_{{\rm mon},\balpha}=p_r$ proves the claim for $r\ge2$. For $r\in\{0,1\}$, the constant one and a coordinate projection are exact and can be padded to the same class.
\end{proof}

\subsection{Proof of Proposition~\ref{prop:localization}}
\label{sec:localization-proof}

The construction is a bounded-weight variant of
\citet[Proposition~3.1]{shen2022optimal}. We first realize staircases on
coarser grids directly and then nest two of them.

For a grid size $K_0\le N^2L$ and $\delta_0\in(0,1/(3K_0)]$, let
$S_{K_0,\delta_0}$ be the continuous function with value
$k/K_0$ on intervals $[k/K_0,(k+1)/K_0-\delta_0\ind_{k\le K_0-2}]$ for
$k=0,\dots,K_0-1$ that is affine between the intervals and constant outside $[0,1]$. This function has
at most $2K_0$ knots, separated by at least $\delta_0$, and takes values
in $[0,1]$. If $\delta_0^{-1}$ is bounded by a fixed polynomial in $N$
and $L$, then \cref{prop:pw-linear}, applied with parameters $2N, 2L$,
gives
\begin{align}
    \label{eq:direct-staircase}
    S_{K_0,\delta_0}
    &\in\wt\Fcal_{1,1}(N,L,N).
\end{align}

Now let $K=K_1K_2$. The assumptions on $\delta$
give $K_1\delta \le 1/(3K_2)$ and $\delta^{-1} \le K^r \le(N^2L)^{2r}$, so
\eqref{eq:direct-staircase} applies to $S_{K_1,\delta}$ and
$S_{K_2,K_1\delta}$. Define
\begin{align*}
    S_{K,\delta}(x)
    &= S_{K_1,\delta}(x)
    +\frac{1}{K_1}S_{K_2,K_1\delta}
    \Bigl(K_1\bigl(x-S_{K_1,\delta}(x)\bigr)\Bigr).
\end{align*}
If $K_2=1$, the second staircase vanishes and
$S_{K,\delta}=S_{K_1,\delta}$. Otherwise, let $x\in[k/K,(k+1)/K-\delta\ind_{k\le K-2}]$ and write
$k = k_1K_2 + k_2$ with $0 \le k_2 < K_2$. Then $x$ lies on the segment
of $S_{K_1,\delta}$ with value $k_1/K_1$, and $K_1(x - k_1/K_1)$ lies on
the segment of $S_{K_2,K_1\delta}$ with value $k_2/K_2$. Hence,
\begin{align*}
    S_{K,\delta}(x) = k_1/K_1 + k_2/K = k/K.
\end{align*}
By \cref{lem:amplification},
the affine maps and scalar multiplications in this expression can be
implemented with depth $O(L)$, width $O(N)$, and weight bound $O(N)$.
Combining these networks with the two staircase networks using
\cref{lem:network-calculus} shows that $S_{K,\delta} \in \wt\Fcal_{1,1}(N,L,N)$. \qed

\subsection{Proof of Proposition~\ref{prop:median}}
\label{sec:median-proof}

Let $\wt\phi(\bx)$ be the median of $\phi(\bz_{-D}),\dots,\phi(\bz_D)$, with $\bz_i$ from \eqref{eq:diagonal-shifts}. The bound $|\wt\phi(\bx)-f(\bx)|\le\eps+\omega_f(D\delta)$ was shown in \cref{sec:median-smoothing}.

It remains to realize $\wt\phi$ by a network. We build it from three
parts and connect consecutive parts by
identity layers.
First, on $[0,1]^D$,
\begin{align*}
    z_{i,j}
    &= \sigma(x_j + i\delta) - \sigma(x_j + i\delta - 1),
\end{align*}
where the second term is omitted for $i \le 0$. With at most
two units for each of the $D(2D+1)$ pairs $(i, j)$, this is a network
with one hidden layer of width at most $2D(2D+1)$ and parameters in
$[-1, 1]$. Second, $2D+1$ copies of $\phi$, applied in parallel to
$\bz_{-D}, \dots, \bz_D$, compute
$y_1 = \phi(\bz_{-D}), \dots, y_{2D+1} = \phi(\bz_D)$. Finally, the network
\begin{align*}
    \phi_{\rm med}
    &\in \Fcal_{2D+1,1}\bigl((2D+1)(2D+2), 2\ceil{\log_2(2D+1)} + 5, D+1\bigr)
\end{align*}
from \cref{lem:median-network} with $m = D$ computes their median
$\wt\phi(\bx)$.

The three parts have depths $1$, $L$, and $2\ceil{\log_2(2D+1)} + 5$, and
two identity layers connect them. Since $N \ge 2D+2$ and $B \ge D+1$, all
widths are at most $(2D+1)N$ and all parameters are bounded by $B$. This
gives the stated class. \qed

\section{Proofs of the Main Results}
\label{sec:main-proofs}

\subsection{Proof of Theorem~\ref{thm:main-approx}}
\label{sec:proof-main-approx}

Fix $f\in\Hcal_1^s([0,1]^D)$ and set $\zeta=2^{\ceil{s}+3}+2$. Define
\begin{align*}
    m&=\ceil{s}-1,
    &q&=\floor{\log_\zeta N},
    &K&=\floor{(N^2L^2q)^{1/D}},
    &M&=2K^D-K+1.
\end{align*}
We may without loss of generality assume that $N\ge\zeta$ after inflating the width by a constant factor depending only on $s$. If $K<2D+1$, then
$N^2L^2q<(2D+1)^D$ and $\log N\le2q\log\zeta$, so the zero network satisfies
\begin{align*}
    \|f - 0\|_\infty
    &\le1
    \le(2\log\zeta)^{s/D}(2D+1)^s\,\epsilon_{N,L}.
\end{align*}
We may therefore also assume $K\ge 2D+1$. Then $q\ge1$ and $K\ge2$, so
\begin{align*}
    M\le2N^2L^2q+1\le3N^2L^2q.
\end{align*}
We now construct the final network in \cref{fig:main-network}.

\medskip
\noindent\textit{Step 1: Parameters and coefficient retrieval}

\medskip
\noindent
\cref{prop:taylor-coefficient-retrieval} gives
\begin{align*}
    \phi_{\rm coef}
    &=\bigl(\phi_{{\rm coef},\balpha}\bigr)_{\|\balpha\|_1\le m}
      \in\wt\Fcal_{D,\binom{D+m}{m}}(N,L,N),
\end{align*}
where, with
$c_{\balpha,\bbeta}=\partial^{\balpha}f(\bbeta/K)/\balpha!$,
\begin{align}
    \label{eq:proof-coefficient-error}
    \left|\phi_{{\rm coef},\balpha}(\bbeta/K)-c_{\balpha,\bbeta}\right|
    \lesssim K^{-(s-\|\balpha\|_1)}.
\end{align}
Since $|c_{\balpha,\bbeta}|\le1$, composing each
$\phi_{{\rm coef},\balpha}$ with $\clip_{[-1,1]}$ (\cref{lem:clipping}) does not increase this
error and adds one hidden layer. We may therefore assume
$|\phi_{{\rm coef},\balpha}|\le1$. By \cref{prop:taylor-coefficient-retrieval}, the constant in
\eqref{eq:proof-coefficient-error} depends only on $s$, and the width of
$\phi_{\rm coef}$ grows by a factor $O(D^m)$.

\medskip
\noindent\textit{Step 2: Localization}

\medskip
\noindent
Let
\begin{align*}
    \delta:=K^{-\max\{2,\ceil{s}\}} \le1/((2D+1)K).
\end{align*}
We apply \cref{prop:localization} with $r=\max\{2,\ceil{s}\}$,  and $K_1=N^2L, K_2= Lq$ if $D=1$ or with $K_1=K, K_2=1$ if $D\ge2$.
This gives a network
$S_{K,\delta}\in\wt\Fcal_{1,1}(N,L,N)$ with $S_{K,\delta}(x) = k/K$ for
$x\in[k/K,(k+1)/K-\delta\ind_{k\le K-2}]$. Define
\begin{align*}
    \Lambda(\bx)
    &=\bigl(S_{K,\delta}(x_1),\ldots,
            S_{K,\delta}(x_D)\bigr),
    \qquad \bx\in[0,1]^D.
\end{align*}
Its $D$ parallel staircases increase the width by the factor $D$.
For $\bx\notin\Omega_{K,\delta}$, let
$\beta_j(\bx)=\min\{\floor{Kx_j},K-1\}$ and
$\bu(\bx)=\bx-\Lambda(\bx)$. Then
\begin{align*}
    \Lambda(\bx)
    &=\frac{\bbeta(\bx)}{K},
    \qquad
    \bu(\bx)\in[0,K^{-1}]^D.
\end{align*}

\medskip
\noindent\textit{Step 3: Product and monomial networks}

\medskip
\noindent
Observe that since $q\le N$ and $L\le N^L$,
\begin{align*}
    K^s
    &\le(N^2L^2q)^{s/D}
    \le N^{5sL/D}.
\end{align*}
Thus, for
$\|\balpha\|_1\le m$,
\cref{lem:poly-approx} with degree $\max\{1, m\}$ and depth
 $\ceil{5s/D} L$ gives a network $\phi_{{\rm mon},\balpha}$ of width $(D+6)N$ satisfying
\begin{align*}
    0\le \phi_{{\rm mon},\balpha}\le1,
    \qquad
    \sup_{\bu\in[0,1]^D}
    |\phi_{{\rm mon},\balpha}(\bu)-\bu^{\balpha}|
    \lesssim N^{-\ceil{5s/D}L} \le K^{-s}.
\end{align*}
Let $\phi_{\rm mult}$ be the multiplication network from the proof
of \cref{lem:poly-approx} with depth parameter $\ceil{5s/D}L$. The identity
$ab=4((a+2)/4)b-2b$ shows that
\begin{align*}
    \phi_{\rm prod}(a,b)
    &=4\phi_{\rm mult}\left(\frac{a+2}{4},b\right)-2b
\end{align*}
belongs to $\wt\Fcal_{2,1}(N,L,N)$ and satisfies
\begin{align*}
    \sup_{(a,b)\in[-2,2]\times[0,1]}
    |\phi_{\rm prod}(a,b)-ab|
    \lesssim N^{-\ceil{5s/D}L} \le K^{-s}.
\end{align*}

\medskip
\noindent\textit{Step 4: Local Taylor approximation}

\medskip
\noindent
Now define
\begin{align}
    \label{eq:proof-local-taylor-network}
    \phi(\bx)
    =\sum_{\|\balpha\|_1\le m}
    \phi_{\rm prod}\left(
        \phi_{{\rm coef},\balpha}(\Lambda(\bx)),
        \phi_{{\rm mon},\balpha}(\bu(\bx))
    \right).
\end{align}
Since there are $\binom{D+m}{m}\le(D+1)^m$ summands,
\cref{lem:network-calculus} gives
\begin{align*}
    \phi&\in\wt\Fcal_{D,1}(N,L,N).
\end{align*}
Fix $\bx\notin\Omega_{K,\delta}$, abbreviate
$\bbeta=\bbeta(\bx)$ and $\bu=\bu(\bx)$, and recall that
$|\phi_{{\rm coef},\balpha}(\bbeta/K)|\le1$. Moreover,
$|\bu^{\balpha}|\le K^{-\|\balpha\|_1}$. For every
$\|\balpha\|_1\le m$,
\begin{align*}
    &\quad \, \left|
    \phi_{\rm prod}\bigl(\phi_{{\rm coef},\balpha}(\bbeta/K),\phi_{{\rm mon},\balpha}(\bu)\bigr)
    -c_{\balpha,\bbeta}\bu^{\balpha}
    \right|\\
    &\le
    \left|\phi_{\rm prod}\bigl(\phi_{{\rm coef},\balpha}(\bbeta/K),\phi_{{\rm mon},\balpha}(\bu)\bigr)
    -\phi_{{\rm coef},\balpha}(\bbeta/K)\phi_{{\rm mon},\balpha}(\bu)\right|\\
    &\qquad+
    |\phi_{{\rm coef},\balpha}(\bbeta/K)|
    |\phi_{{\rm mon},\balpha}(\bu)-\bu^{\balpha}|
    +|\phi_{{\rm coef},\balpha}(\bbeta/K)-c_{\balpha,\bbeta}|
    |\bu^{\balpha}|\\
    &\lesssim
    K^{-s}+K^{-(s-\|\balpha\|_1)}K^{-\|\balpha\|_1} \\
    &=K^{-s}.
\end{align*}
Taylor's theorem gives
\begin{align*}
    \left|f(\bx)-
    \sum_{\|\balpha\|_1\le m}
    c_{\balpha,\bbeta}\bu^{\balpha}\right|
    \le\frac{D^m}{m!}K^{-s} \lesssim K^{-s}.
\end{align*}
Combining these two bounds in \eqref{eq:proof-local-taylor-network},
\begin{align}
    \label{eq:proof-away-boundaries}
    \sup_{\bx\in[0,1]^D\setminus\Omega_{K,\delta}}
    |\phi(\bx)-f(\bx)|
    \lesssim K^{-s}.
\end{align}

\medskip
\noindent\textit{Step 5: Smoothing}

\medskip
\noindent
Since $f$ is $\min\{s, 1\}$-Hölder continuous with constant $D$, its modulus of continuity satisfies
\begin{align*}
    \omega_f(D\delta)
    &\le D (D\delta)^{\min\{s,1\}}\le D^2K^{-s}.
\end{align*}
Applying \cref{prop:median} to
\eqref{eq:proof-away-boundaries} gives
$\wt\phi\in\wt\Fcal_{D,1}(N,L,N)$ such that
\begin{align*}
    \|\wt\phi-f\|_\infty
    &\lesssim K^{-s}
    \le2^s(2\log\zeta)^{s/D}(N^2L^2\log N)^{-s/D}.
\end{align*}
This increases the width by the factor $2D+1$ and adds
$O(\log D)$ layers.
By \cref{lem:magnitude-width} with $B'=1$,
$\wt\Fcal_{D,1}(N,L,N) \subseteq \wt\Fcal_{D,1}(N,L,1)$.
All other constants depend only on $s$ and polynomially on $D$ by
\cref{lem:network-calculus,lem:magnitude-width}. \qed

\subsection{Proof of Theorem~\ref{thm:lower-bound} and Corollary~\ref{cor:network-entropy}}
\label{sec:entropy-proof}

The optimality results rely on the following entropy comparison.
\begin{lemma}[Approximation and entropy]\label{lem:approx-entropy}
    Let $\Hcal$ and $\Gcal$ be subsets of a metric space with metric $d$ and
    let $\epsilon > 0$. If
    $\sup_{f \in \Hcal}\inf_{g \in \Gcal}d(f, g) \le \epsilon$, then
    \begin{align*}
        \mathcal N(\epsilon, \Gcal, d)
         & \ge \mathcal N(3\epsilon, \Hcal, d).
    \end{align*}
\end{lemma}

\begin{proof}
    Let $\Ccal$ be an $\epsilon$-cover of $\Gcal$. Given $f \in \Hcal$, the
    definition of the infimum provides $g \in \Gcal$ with $d(f, g) < 2\epsilon$,
    and $\Ccal$ provides $\widetilde g$ with $d(g, \widetilde g) \le \epsilon$.
    The triangle inequality gives $d(f, \widetilde g) < 3\epsilon$, so $\Ccal$
    is a $3\epsilon$-cover of $\Hcal$.
\end{proof}

The lemma is combined with two established entropy bounds for the H\"older ball and ReLU networks, respectively. First, let $\epsilon\in(0,1/4)$ and $\Gcal$ be a class that approximates $\Hcal_1^s([0,1]^D)$ uniformly to accuracy $\epsilon$. Then \cref{lem:approx-entropy} gives
    \begin{align}
        \label{eq:approx-forces-entropy}
        \log \mathcal N\bigl(\epsilon, \Gcal, \|\cdot\|_\infty\bigr)
         & \ge \log \mathcal N\bigl(3\epsilon, \Hcal_1^s([0,1]^D),
        \|\cdot\|_\infty\bigr)
        \gtrsim \epsilon^{-D/s},
    \end{align}
    by the standard entropy estimate for Hölder balls
    \citep[e.g.,][Theorem~27.14]{polyanskiyWu2025}.
    Second, the covering-number bound of \citet[Theorem~2.1]{ou2024covering} gives, for all $\epsilon\in(0,1)$ and $B\ge1$,
    \begin{align}
        \log \mathcal N\bigl(\epsilon, \Fcal_{D,1}(N,L,B), \|\cdot\|_\infty\bigr)
        &\lesssim N^2L^2\log(NB) + N^2L\log(1/\epsilon). \label{eq:network-entropy-bound}
    \end{align}

\begin{proof}[Proof of \cref{thm:lower-bound}]
    Define
    \begin{align*}
        e_{N,L,B} := \sup_{f \in \Hcal_1^s([0, 1]^D)}
        \inf_{\phi \in \Fcal_{D,1}(N,L,B)}\|f - \phi\|_\infty.
    \end{align*}
    We may assume that $e_{N,L,B} \le 1/4$ because otherwise the claim is immediate for any $c\le1/4$.
    The inequalities \eqref{eq:approx-forces-entropy} and  \eqref{eq:network-entropy-bound} give
    \begin{align*}
       e_{N,L,B}^{-D/s} \lesssim \log \mathcal N(e_{N,L,B}, \Fcal_{D,1}(N,L,B), \|\cdot\|_\infty)  \lesssim N^2L^2\log(NB) + N^2L\log(1/e_{N,L,B}).
    \end{align*}
    If the second term on the right dominates the first, then $\eta=e_{N,L,B}^{-D/s}$ satisfies $\eta \lesssim N^2L\log \eta$, which implies
    \begin{align*}
        e_{N,L,B}^{-D/s}
        \lesssim N^2L\log(N^2L)
        \lesssim N^2L^2 \log N.
    \end{align*}
    Otherwise, the first term dominates and we have
    \begin{align*}
        e_{N,L,B} \gtrsim (N^2L^2\log(NB))^{-s/D}.
    \end{align*}
    In both cases, $1 \le B \le N^b$ implies
    \begin{align*}
        e_{N,L,B} \gtrsim (N^2L^2\log(NB))^{-s/D} \gtrsim (N^2L^2\log N)^{-s/D},
    \end{align*}
    as claimed.
\end{proof}

\begin{proof}[Proof of \cref{cor:network-entropy}]
    Let $A$ and $\Fcal^{N,L}$ be those from \cref{thm:main-approx}. The class $\Fcal^{N,L}$ is one of the classes the
    infimum in \eqref{eq:compact-optimal-entropy} ranges over, so
    \begin{align*}
        \inf_{\Gcal} \log \mathcal N(\epsilon_{N,L}, \Gcal,
        \|\cdot\|_\infty) \le  \log \mathcal N(\epsilon_{N,L}, \Fcal^{N,L},
        \|\cdot\|_\infty),
    \end{align*}
    and \eqref{eq:network-entropy-bound} with $\log L \le L$ and $\log\log N \lesssim \log N$ gives
    \begin{align*}
        \log \mathcal N\bigl(\epsilon_{N,L}, \Fcal^{N,L}, \|\cdot\|_\infty\bigr)
         & \lesssim N^2L^2\log N + N^2L\log\bigl(N^2L^2\log N\bigr)
          \lesssim N^2L^2\log N.
    \end{align*}
    For the lower bound, $C \ge 1$ and \eqref{eq:approx-forces-entropy} give
    \begin{align*}
        \log \mathcal N\bigl(\epsilon_{N,L}, \Gcal, \|\cdot\|_\infty\bigr)
         & \ge \log \mathcal N\bigl(C\epsilon_{N,L}, \Gcal, \|\cdot\|_\infty\bigr)
         \gtrsim \epsilon_{N,L}^{-D/s}
        = N^2L^2\log N.
    \end{align*}
    Both quantities in
    \eqref{eq:compact-optimal-entropy} are therefore of order $N^2L^2\log N$.
\end{proof}

\subsection{Proof of Theorem~\ref{thm:regression}}
\label{sec:regression-proof}

By \cref{thm:main-approx} and \eqref{eq:compact-network-balance}, for every
$f_0\in\Hcal_1^s([0,1]^D)$ there is an $f_n\in\Fcal^{N_n,L_n}$ satisfying
\begin{align*}
\|f_n-f_0\|_{\infty,[0,1]^D}\le Cr_n.
\end{align*}
Since $f_0$ takes values in $[-1,1]$,
\begin{align*}
\inf_{f\in\Gcal_n}\|f-f_0\|_{L_2(P)}\le \|\clip_{[-1,1]}\circ f_n-f_0\|_{\infty,[0,1]^D}
\le \|f_n-f_0\|_{\infty,[0,1]^D}
\le Cr_n.
\end{align*}
Because clipping is nonexpansive in the uniform norm, any cover of
$\Fcal^{N_n,L_n}$ induces a cover of $\Gcal_n$. The covering-number
bound of \citet[Theorem~2.1]{ou2024covering} gives
\begin{align*}
\log\mathcal N\bigl(C^2r_n^2,\Gcal_n,\|\cdot\|_\infty\bigr)
&\lesssim N_n^2L_n
    \log\left(
    \frac{(AN_n+1)^{AL_n}}{C^2r_n^2}
    \right)
\lesssim N_n^2L_n^2\log N_n
\lesssim nr_n^2,
\end{align*}
where $A$ is the constant from \cref{thm:main-approx} and
depends only on $s$ and $D$. The second inequality follows from
\eqref{eq:compact-network-balance}, which also implies
$\log(1/r_n)\lesssim L_n\log N_n$.

Suppose that $n$ is large enough that $Cr_n<1/2$. Since the optimization error $r_n^2$ in \eqref{eq:compact-approx-erm} is at most $C^2r_n^2$, \citet[Theorem~4.1]{ou2024covering} applied with
$\|f_0\|_\infty\le1$ and
$\sup_{f\in\Gcal_n}\|f\|_\infty \le 1$ implies
\begin{align*}
\E_{f_0}\bigl[\|\widehat f_n-f_0\|_{L_2(P)}^2\bigr]
&\lesssim C^2r_n^2+
    \frac{1+\log\mathcal N(C^2 r_n^2,\Gcal_n,\|\cdot\|_\infty)}{n}
\lesssim r_n^2,
\end{align*}
where the constants depend only on $s$, $D$, and $\gamma$. For the finitely many $n$ with $C r_n\ge1/2$, the claim holds trivially because $\|\widehat f_n-f_0\|_\infty\le2$. \qed

\vskip 0.2in
\bibliography{hcm.bib}
\end{document}